\pdfoutput=1

\def\MODE{1} 
\if\MODE1

\else

\fi

\documentclass[10pt]{article}

\usepackage{palatcm}
\usepackage{subfigure} 
\usepackage{soul}
\usepackage{color, xcolor}
\usepackage[thinlines]{easytable}
\usepackage{relsize}
\usepackage{xfrac}
\usepackage{verbatim}
\usepackage{algorithm}
\usepackage[noend]{algorithmic}
\usepackage{amsmath}
\usepackage{amssymb}
\usepackage{amsthm}
\usepackage{epsfig}
\usepackage{wrapfig}
\usepackage{url}
\usepackage[colorlinks=true,citecolor=blue,linkcolor=blue]{hyperref}
\usepackage{multirow}
\usepackage{fullpage}

\definecolor{ed}{RGB}{225,0,0}

\usepackage{wrapfig}
\usepackage{bbm}
\usepackage{hyperref}       
\usepackage{url}            
\usepackage{booktabs}       
\usepackage{amsfonts}       
\usepackage{microtype}      
\usepackage{microtype}
\usepackage{graphicx}
\usepackage{booktabs} 
\usepackage{amsthm,amsmath,amssymb}
\usepackage{float,url,amsfonts,alltt}
\usepackage{mathtools,rotating}
\usepackage{ifpdf,fancyvrb}
\usepackage{enumitem}

\usepackage{microtype}
\usepackage{graphicx}
\usepackage{booktabs} 
\usepackage{hyperref}

\usepackage{graphicx,xspace,verbatim,comment}
\usepackage{hyperref,array,color,balance,multirow}
\usepackage{balance,float,url,amsfonts,alltt}
\usepackage{mathtools,rotating,amsmath,amssymb}
\usepackage{color,ifpdf,fancyvrb,array}
\usepackage{etoolbox,listings}
\usepackage{bigstrut,morefloats}

\usepackage{pbox}
\usepackage[boxruled,algo2e,linesnumbered]{algorithm2e}
\DeclarePairedDelimiterX{\inp}[2]{\langle}{\rangle}{#1, #2}

\newtheorem{theorem}{Theorem}

\newtheorem{lemma}[theorem]{Lemma}

\newtheorem{definition}{Definition}
\newcommand{\eat}[1]{}

\newcommand{\R}{\mathbb{R}}
\newcommand{\E}{\mathbb{E}}
\newcommand{\var}{\operatorname{Var}}

\DeclarePairedDelimiter{\ceil}{\lceil}{\rceil}

\numberwithin{equation}{section}

\usepackage{amsmath}
\usepackage{accents}
\newlength{\dhatheight}

\usepackage[T1]{fontenc}
\usepackage{upgreek}

\usepackage[greek,english]{babel}

\newcommand{\WH}[1]{}
\newcommand{\ST}[1]{}
\newcommand{\XP}[1]{}

\newtoggle{tr}
\toggletrue{tr}
\iftoggle{tr}{
	\makeatother
}{}

\usepackage{subfigure}

\newcommand{\systemnameAPIShift}{\textsc{MASA}}

\newcommand{\SCM}{\pmb{C}}
\newcommand{\deltaCM}{\Delta\pmb{C}}

\newcommand{\hatdeltaCM}{\Delta\hat{\pmb{C}}}
\newcommand{\SN}{\pmb N}
\newcommand{\Smu}{\hat{\pmb \mu}}
\newcommand{\truemu}{{\pmb \mu}}
\newcommand{\usc}{\pmb \sigma}
\newcommand{\hatusc}{\hat{\pmb \sigma}}
\newcommand{\hash}{\pmb H}

\newcommand{\sample}{\pmb z}

\newcommand{\loss}{\mathcal L}

\newcommand{\Exp}{\mathbb{E}}

\title{Did the Model Change? Efficiently Assessing Machine Learning API Shifts}

\author{
Lingjiao Chen, Tracy Cai, Matei Zaharia, James Zou\\\\
Stanford University} 

\date{}

\begin{document}

\maketitle

\begin{abstract}
Machine learning (ML) prediction APIs are increasingly widely used. An ML API can change over time due to model updates or retraining. This presents a key challenge in the
 usage of the API because it's often not clear to the user if and how the ML model has changed. Model shifts can affect downstream application performance and  also create oversight issues (e.g. if consistency is desired). In this paper, we initiate a systematic investigation of ML API shifts. We first quantify the performance shifts from 2020 to 2021 of popular ML APIs from Google, Microsoft, Amazon, and others on a variety of datasets. We identified significant model shifts in 12 out of 36 cases we investigated. Interestingly, we found several datasets where the API's predictions became significantly worse over time. This motivated us to formulate the API shift assessment problem at a more fine-grained level as estimating how the API model's confusion matrix changes over time when the data distribution is constant. Monitoring confusion matrix shifts using standard random sampling can require a large number of samples, which is expensive as each API call costs a fee. We propose a principled adaptive sampling algorithm, \systemnameAPIShift, to efficiently estimate confusion matrix shifts. MASA can accurately estimate the confusion matrix shifts in commercial ML APIs using up to $90\%$ fewer samples compared to random sampling. This work establishes ML API shifts as an important problem to study and provides a cost-effective approach to monitor such shifts.         

\end{abstract}

\section{Introduction}\label{Sec:FAMEShift:Intro}
Machine learning (ML) prediction APIs have made it dramatically easier to deploy ML applications.
For example, one can use Microsoft text API \cite{MicrosoftAPI} to determine the polarity of a text review written by a customer, or Google speech API \cite{GoogleSpeechAPI} to recognize  users' spoken commands received by a smart home device.
These APIs have been gaining popularity \cite{MLasS_MarketInfo, FrugalML2020}, as they avoid the need to collect data and train one's own models.

Monitoring and assessing the performance of those third-party ML APIs over time, however, are under-explored.
ML API providers continuously collect new data or change their model architectures \cite{IBMAPIUpdatePaper2020} to update their services, which could silently  help or hurt downstream applications' performance.
For example, as shown in Figure \ref{fig:FAMEShift:Example} (a) and (b), we observe a 7\% overall accuracy drop of IBM speech API on the AUDIOMNST dataset in March 2021 compared to its evaluation in March 2020. In our systematic study of 36 API + dataset combinations, there are 12 cases where the API's performance changed by more than $1\%$ on the same dataset from 2020 to 2021 (sometimes for the worse).
Such performance shifts are of serious concern  not only because of potential disruptions to downstream tasks but also because consistency is often required for audits and oversight. Therefore it is important to precisely assess shifts in an API model's predictions over time. In this assessment, it is often much more informative to quantify how the entire confusion matrix of the API has changed rather than just the overall accuracy. In the IBM case in Figure \ref{fig:FAMEShift:Example}, it is interesting that a major culprit of the drop in performance is the 2021 model mistaking ``four'' for ``five''. In other settings, changes in the confusion matrix could still cause issues even if the overall accuracy stays the same. 

\begin{figure*}[t]
	\centering
	\vspace{-0mm}
	\includegraphics[width=1.0\linewidth]{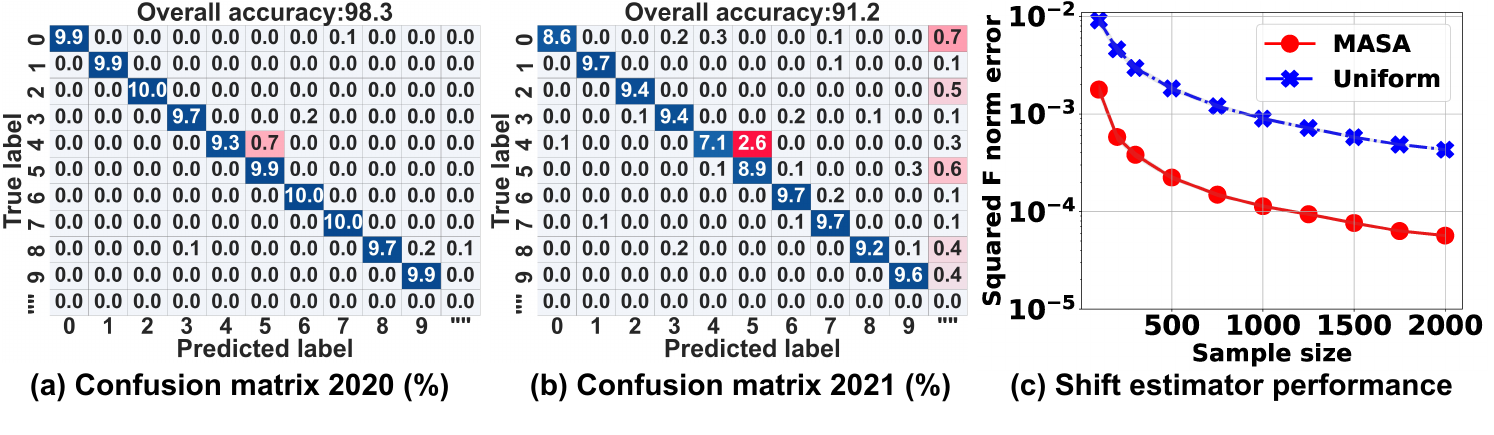}
	\vspace{-0mm}
	\caption{ML API shift for IBM speech recognition API on AMNIST, a spoken digit dataset. (a) and (b) give its (normalized) confusion matrix in April 2020 and 2021, respectively. 
	There is an overall 7\% accuracy drop. One factor is the 2021 model incorrectly predicting more ``four'' as ``five''. 
	(c) Given a sample budget, the proposed \systemnameAPIShift{} can assess the API shift with much smaller error in Frobenius norm compared to standard uniform sampling.}
	\label{fig:FAMEShift:Example}
\end{figure*}

In this paper, we formalize the problem of assessing API shifts as estimating changes in the confusion matrix on the same data set. The straightforward approach is to compare the API's prediction on randomly sampled data. However, this can require a large number of API calls to estimate the confusion matrix, which is too expensive since each API call costs a fee. 
To help address this challenge, we propose \systemnameAPIShift{}, a principled algorithm for \underline{M}L \underline{A}PI \underline{s}hift \underline{a}ssessments. 
\systemnameAPIShift{} efficiently estimates shifts in the API's confusion matrix by clustering the dataset and adaptively sampling data from different clusters to query the API. \systemnameAPIShift{} automates its sampling rate from different data clusters based on the uncertainty in the confusion matrix estimation.   
For example, it may query the ML API  on more samples with the true label ''four`` than  "one", if it is less sure about the estimated performance change on the former.  
Employing an upper-confidence-bound approach to estimate the uncertainties,   \systemnameAPIShift{} enjoys a low computation and space cost as well as a fast estimation error rate guarantee.

\systemnameAPIShift{}'s adaptive sampling substantially improves the quality of  estimation for API shifts. 
In extensive experiments on real world ML APIs, 
 \systemnameAPIShift{}'s assessment error is often an order of magnitude smaller than that of standard uniform sampling with same sample size (e.g., Figure \ref{fig:FAMEShift:Example} (c)). 
To reach the same tolerable estimation error, \systemnameAPIShift{} can reduce the required sample size by more than 50\%, sometimes up to 90\%.


\textbf{Contributions.} In short, our main contributions include:
\begin{enumerate}
    \item We demonstrate that commercial Ml APIs can experience significant performance shifts over time, and formulate ML API shift assessments via confusion matrix difference estimation as an important practical problem. 
    
    \item We propose \systemnameAPIShift{}, an algorithm to assess the ML API performance shifts efficiently. \systemnameAPIShift{} adaptively determines querying the ML API on which data points to minimize the shift estimation error under a sample size constraint.
    We show that \systemnameAPIShift{} enjoys a low computation cost and performance guarantee.
    
     \item We evaluate MASA on real world APIs from Google, Microsoft, Amazon and other providers for tasks including speech recognition, sentiment analysis, and facial emotion recognition. 
     \systemnameAPIShift{} leads to estimation errors an order of magnitude smaller than standard uniform sampling using the same sample size, or more than 90\% fewer samples to reach the same tolerable estimation error. 

\end{enumerate}

\paragraph{Related Work.}

\textbf{Distribution shifts in ML deployments:} Performance shifts in ML systems have been observed in  applications like disease diagnosis \cite{labelshift2018},  facial recognition \cite{MaskFaceDataset2020}, and molecular inference \cite{DistributionShiftBenchmark2021}. 
Most of them are attributed to distribution shifts, i.e., the distribution of the test and training datasets are different.
Distribution shifts are usually modeled as covariate shifts \cite{covariateshift2000,covshift2007,covshiftkernel2007}, referring to the feature distribution change, and label shifts \cite{labelshift2018,priorshift2002,Learningunderlabelshift2019,activelearninglabelshift2021}, referring to the label distribution change. 
API shifts are orthogonal to distribution shifts: instead of attributing the performance shifts to data distribution changes, API shifts concern with ML APIs changes which changes its predictions on the same dataset. The methods for detecting distribution drifts typically rely on changes in data feature statistics and can not detect changes in the API on the same data. To the best of our knowledge, this is the first work to systematically investigate ML API shifts.

\textbf{Deploying and monitoring ML APIs:} 
Several issues in deployed ML APIs have been studied.
For example, \cite{pmlr-v81-buolamwini18a} shows that strong biases toward minority may exist  in commercial APIs and \cite{ACLtest2020} reveals that several bugs in commercial APIs can be detected using checklists. \cite{Modelassertion2020} adopts program assertions to monitor and improve deployed ML models. 
\cite{FrugalML2020} considers the trade-offs between accuracy performance and cost via exploiting multiple APIs.
On the other hand, the proposed \systemnameAPIShift{}  focuses on estimating (silent) API performance changes cheaply and accurately, which has not been studied before.

\textbf{Stratified sampling and multi-arm bandit:} Stratified sampling has proved to be useful in various domains, such as approximate query processing \cite{AQPsampling2007}, population mean estimation \cite{SamplingMAB2012,SamplingMAB2015}, and complex integration problems \cite{samplingintegration2011}. 
A common approach is to model stratified sampling as a multi-arm bandit (MAB) problem: view each data partition as an arm, and set the regret as the variance of the obtained estimator.
While estimating confusion matrix shifts  incurs a unique regret compared to the standard literature on adaptive estimation.
Therefore, a new algorithm based on upper-confidence bound is developed in  \systemnameAPIShift{}.

\section{The API Shift Problem}\label{Sec:APIShift:Preli}
\eat{
\paragraph{Notation.}
Throughout this paper, tensors, matrices and vectors are denoted in bold, while scalars, sets and functions are in  standard script. 
Given a tensor $\mathbf{A} \in \mathbb{R}^{n_1\times n_2\times \cdots \times n_{m}}$, we let $\mathbf{A}_{i_1,i_2,\cdots, i_m}$ denote its entry at location $(i_1,i_2,\cdots, i_m)$. 
$[n]$ denotes $\{1,2,\cdots,n\}$.
}

\eat{
\begin{table}[t]
  \centering
  \small
  \caption{ML API and Datasets.}
    \begin{tabular}{|c|c||c|c|c|c|}
    \hline
    \multirow{2}[4]{*}{Sentiment Analysis} & ML API & Amazon & Google & Baidu &  \bigstrut\\
\cline{2-6}          & Dataset & YELP  & IMDB  & WAIMAI & SHOP \bigstrut\\
    \hline
    \multirow{2}[4]{*}{Facial Emotion} & ML API & Microsoft & Google & Face++ &  \bigstrut\\
\cline{2-6}          & Dataset & FER+  & EXPW  & RAFDB & AFNET \bigstrut\\
    \hline
    \multirow{2}[4]{*}{Speech Recognition} & ML API & IBM   & Google & Microsoft &  \bigstrut\\
\cline{2-6}          & Dataset & DIGIT & AMNIST & CMD   & FLUENT \bigstrut\\
    \hline
    \end{tabular}%
  \label{tab:addlabel}%
\end{table}%
}

\begin{figure}[t]\label{fig:FAMEShift:summary}
	\centering
	\includegraphics[width=1\linewidth]{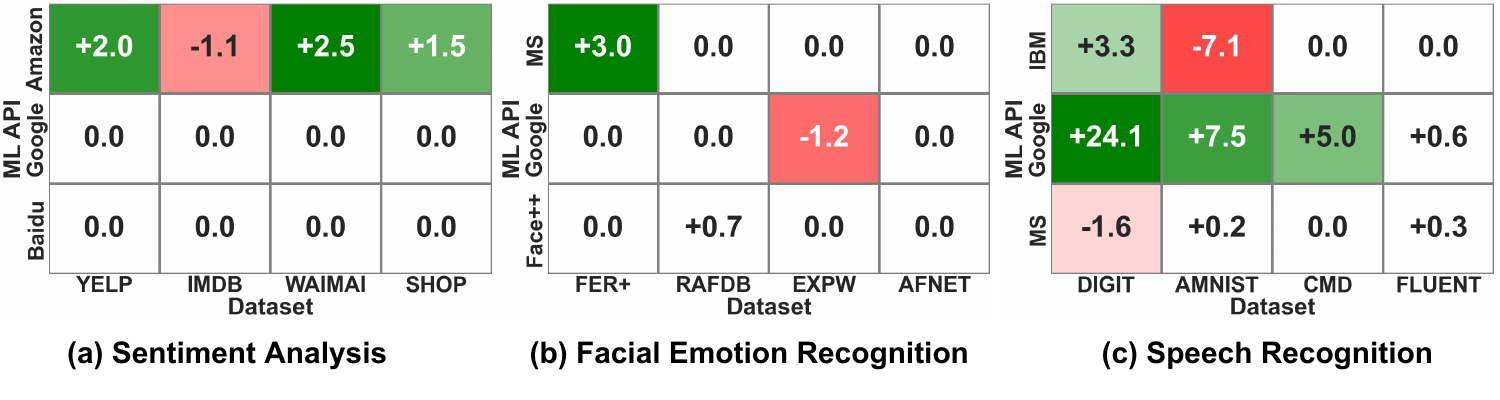}
\caption{Observed overall accuracy changes. Each row corresponds to an ML API, and each column represents a dataset. The entry is the overall accuracy difference between evaluation in spring 2020 and spring 2021. 
In 12 out of 36 cases, the API's overall accuracy changed by more than 1\%; this includes several cases of substantial drops in performance.
}
\end{figure}

\paragraph{Empirical assessment of ML API shifts.}
We start by making an interesting observation:  \textit{Commercial ML APIs' performance can change substantial over time on the same datasets}. We investigated twelve standard datasets across three different tasks, namely,  YELP \cite{Dataset_SEntiment_YELP}, IMDB \cite{Dataset_SEntiment_IMDB_ACL_HLT2011}, WAIMAI \cite{Dataset_SENTIMENT_WAIMAI}, SHOP \cite{Dataset_SENTIMENT_SHOP} for sentiment analysis, FER+ \cite{Dataset_FER2013}, RAFDB \cite{Dataset_FAFDB_li2017reliable}, EXPW \cite{Dataset_EXPW_SOCIALRELATION_ICCV2015}, AFNET \cite{Dataset_AFFECTNET_MollahosseiniHM19} for facial emotion recognition, and DIGIT 
\cite{Dataset_Speech_DIGIT}, AMNIST \cite{Dataset_Speech_AudioMNIST_becker2018interpreting}, CMD \cite{Dataset_Speech_GoogleCommand}, FLUENT \cite{Dataset_Speech_Fluent_LugoschRITB19}, for speech recognition. 
For each dataset, we evaluated three commercial ML APIs' accuracy in April 2020 and April 2021. Figure \ref{fig:FAMEShift:summary} summarizes the overall accuracy changes. 

There are several interesting empirical findings. First, API performance changes are quite common.  
In fact, as shown in Figure \ref{fig:FAMEShift:summary}, API performance changes exceeding 1\% occurred in about 33\% of all (36) considered ML API-dataset combinations. 
Since the data distribution remains fixed, such a change is due to ML APIs' updates. 
Second, the API updates can either help or hurt the accuracy performance depending on the datasets. 
For example, as shown in Figure \ref{fig:FAMEShift:summary} (a), the Amazon sentiment analysis API's accuracy increases on YELP, WAIMAI, and SHOP, but decreases on IMDB. 
In addition, the update of Microsoft facial emotion recognition API only affects performance on the FER+ dataset, as shown in Figure \ref{fig:FAMEShift:summary} (b). 
Another interesting finding is that the magnitude of the performance change can be quite different. 
In fact, most of the accuracy differences are between 1--3\%, but on DIGIT dataset, Google's accuracy change is more than 20\%.

\paragraph{Fine-grained assessment of API shift as changes in the confusion matrix.} 

Based on feedback from practitioners, accuracy change alone is insufficient, and attribution to per class change is often much more informative \cite{Imagenetbenckmark20,RebalancingClassifier2019,ConfusionMatrixclassimbalance2019}. %
Thus, a natural idea is to quantify an ML API's performance by its confusion matrix. We assess the change of the confusion matrix over time as a measure of API shift.

Formally,  consider an ML service for a classification task with $L$ labels.
For a data point $x$ from some domain $\mathcal{X}$, let   $\hat{y}(x) \in [L]$ denote its predicted label on $x$, and $y(x)$ be the true label.
For example, for sentiment analysis, $x$ is a text paragraph, and the task is to predict if the polarity of $x$ is positive or negative. Here $L=2$, and $\hat{y}(x)=1$ implies positive predicted label while ${y}(x)=2$ indicates negative true label. 
The confusion matrix is denoted by $\SCM\in\R^{L\times L}$ where $\SCM_{i,j} \triangleq \Pr[y(x) = i, \hat{y}(x)=j]$.
Given a confusion matrix of the ML API measured previously (say, a few months ago), $\SCM^o$, the ML API shift is defined as $\deltaCM \triangleq \SCM-\SCM^O$.

Using confusion matrix difference to quantify the ML API shift is informative. 
E.g,  the overall accuracy change is simply the trace of $\deltaCM$.
It also explains which label gets harder or easier for the updated API. 
Still consider, e.g., sentiment analysis. 
Given a 2\% overall accuracy change,
$\deltaCM_{1,2}=1\%$ and $\deltaCM_{2,1}=-3\%$ implies that the change is due to predicting less (-3\%) negative texts as positive, by sacrificing the accuracy on positive texts slightly (1\%).
This suggests that the API could have been updated with more negative training texts. 

\section{\systemnameAPIShift{}: ML API Shift Assessment}\label{Sec:FAMEShift:theory}

Now we present \systemnameAPIShift{}, an algorithmic framework efficiently to assess ML API shifts. 
Suppose the old confusion matrix $\SCM^o$ and a large labeled dataset $D$ are available. 
Given a query budget $N$, our goal is to generate $\hatdeltaCM$, an estimation of  the API shifts as accurately as possible by querying the ML API $\hat{y}(\cdot)$ on $N$ samples drawn from $D$. 

\begin{figure}[t]
	\centering
\includegraphics[width=1.0\linewidth]{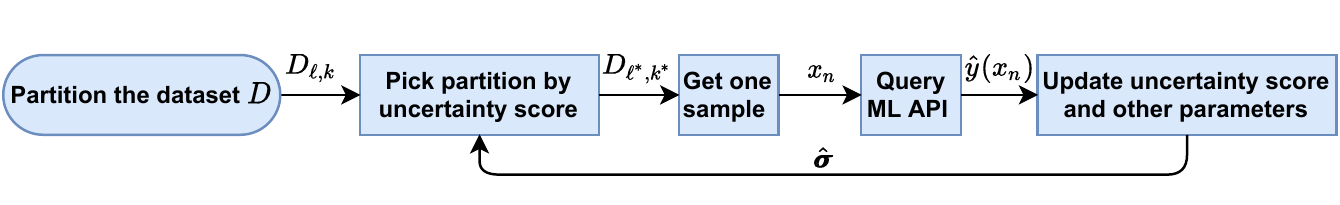}
     \caption{How \systemnameAPIShift{} works. \systemnameAPIShift{} first partitions the dataset. Then it picks which partition to sample based on some uncertainty measurement, queries the ML API on the drawn sample, and uses the API' prediction to update uncertainty and estimated shifts on this partition. This is repeated until the ML API has been queried $N$ times. 
     Finally, the estimated shifts on different partitions are aptly fused to obtain the desired API shifts. 
     }\label{fig:FAMEShift:workflow}
\end{figure}
 
 \systemnameAPIShift{} achieves its goal via an adaptive  sampling approach (Figure \ref{fig:FAMEShift:workflow}). 
It first divides the given dataset $D$ into several partitions (clusters).
Then it adaptively decides which sample to query the ML API in an iterative manner: at each iteration, it selects one data partition based on some uncertainty measure (defined below), and queries the ML API on one sample randomly drawn from this partition.
The API's prediction is obtained to update the uncertainty measure as well as the estimated shift $\hatdeltaCM$.
This process is repeated until the ML API has been queried $N$ times or if a stopping rule is reached. We explain each step in detail next.

\subsection{Data Partitioning}
A key intuition in \systemnameAPIShift{} is that not all samples are  equally informative for estimating API shifts. Consider, for example, a vision  API makes perfect predictions on ``dog'' images, and guesses randomly on ''cat`` pictures. 
The ``dog'' images are less informative, as even a small sample of  ``dog'' queries would tell that there is essentially no confusion for this class.
Intuitively, within a sample budget, an estimator with more samples from ``cat''  pictures should be more accurate overall than that from ``dog''.  
Generally, harder images tend to be more informative.

Thus, it is natural to partition all data points based on factors that may correlate with their informativeness, and sample from those partitions separately. 
In MASA, we use partitions $D_{i,k}$ that each contain the points with true label $i$ and difficulty level $k$. 
The difficulty level is an integer indicating how hard it is to predict the data point's label.
It needs not be perfect, and can be simply the discretized prediction confidence generated by some simple  ML models. A total of $L$ labels and $K$ distinct difficulty labels lead to a total of $LK$ partitions. 
If the uncertainty or variability of the ML API's prediction on each partition is different, then drawing a different number of samples from each partition may improve the shift assessment performance compared to standard uniform sampling. 
We verify this empirically in our evaluation (Section \ref{Sec:FAMEShift:Experiment}).

\subsection{Budget Allocation Problem}
Given the data partition, two questions arise: (i) how many samples should be drawn from each partition, and (ii) how to estimate the ML API shifts given available samples. 
The second question is relatively straightforward. Note that
the API shifts satisfy 
\begin{equation*}
\begin{split}
    \deltaCM_{i,j} = &\Pr[y(x)=i, \hat{y}(x)=j] - \SCM^o_{i,j}= \sum_{i=1}^{L} \sum_{k=1}^{K} \Pr[y(x)=i, \hat{y}(x)=j, x \in D_{i,k}]- \SCM^o_{i,j}\\
    = &\sum_{k=1}^{K} \Pr[\hat{y}(x)=j, x \in D_{i,k}] - \SCM^o_{i,j} = \sum_{k=1}^{K} \Pr[x\in D_{i,k}] \Pr[\hat{y}(x)=j| x \in D_{i,k}] - \SCM^o_{i,j}\\
\end{split}
\end{equation*}
where the first equation is by definition, the second is due to total probability rule, the third uses the fact that $x\in D_{i,k}$ implies $y(x)=i$, and the last equation applies conditional probability. 
Here, $\Pr[x\in D_{i,k}]$ is simply ratio of size of partition $D_{i,k}$ and entire dataset $D$, known a prior. 
To assess $\deltaCM_{i,j}$,  we only need to estimate $\Pr[\hat{y}(x)=j| x \in D_{i,k}]$, the predicted label distribution on partition $D_{i,k}$.
It can be estimated simply via the frequency of predicting label $j$ among all available samples drawn from $D_{i,k}$.

Now we consider the sample allocation problem. 
For ease of notation, we denote $\Pr[x_i\in D_{i,j}]$ by $\pmb p_{i,k}$,  $\Pr[\hat{y}(x)=j| x \in D_{i,k}]$ and its estimation by $\truemu_{i,k,j}$  and $\Smu_{i,k,j}$, respectively.
Then for deterministic sample allocations, 
the squared Frobenius norm error can be written as 
\begin{equation*}
\begin{split}
\Exp\left[\| \deltaCM - \hatdeltaCM \|_F^2\right] = &\sum_{i,j} \Exp\left(\deltaCM_{i,j} - \hatdeltaCM_{i,j} \right)^2 = \sum_{i,j} \Exp\left(\sum_{k}^{} \pmb p_{i,k}[\truemu_{i,k,j}-\Smu_{i,k,j}] \right)^2  \\= &  \sum_{i,j,k} \pmb p_{i,k}^2 \Exp\left([\truemu_{i,k,j}-\Smu_{i,k,j}] \right)^2    
\end{split}
\end{equation*}
Thus we use the loss $\loss(\mathcal{A},N)\triangleq\sum_{i,j,k} \pmb p_{i,k}^2 \Exp\left([\truemu_{i,k,j}-\Smu_{i,k,j}] \right)^2$ to measure the performance of any sample budget allocation algorithm $\mathcal{A}$ using $N$ samples.
For any fixed $N$, our goal is to find a sample budget allocation algorithm $A$ to minimize the loss $\loss(A,N)$. 
Notably, we can generalize it for other scenarios by replacing $(\deltaCM - \hatdeltaCM)$ with $\pmb W \odot (\deltaCM - \hatdeltaCM)$, where $\odot$ is element-wise multiplication and $\pmb W$ is an  $L\times L$ weight matrix.
Different choices of $\pmb W$ can penalize the error of each entry in $\hatdeltaCM$ differently.
For instance, if $\pmb W$ is identical matrix, then the focus is simply the overall accuracy. across all labels.
If $\pmb W_{1,2}=1$ and $\pmb W_{i,j}=0, \forall  (i,j) \not=(1,2)$, then we are only interested in incorrectly predicting label 1 as label 2. 
The methodology and analysis for the standard Frobneuis norm can be easily adopted for the general formulation, and we focus on standard Frobneuis norm for exposition purposes.

\subsection{Uncertainty Score and Optimal Allocation }
The optimal sample allocation is directly connected to how informative each data partition is.
To see this, let us first introduce the notation of \textit{uncertainty score} for each data partition.  

\begin{definition}
$\usc^2_{i,k} \triangleq (1-\sum_{j=1}^{{L}} \Pr^2[\hat{y}(x)=j| x \in D_{i,k} ] )$ denotes the uncertainty score of $D_{i,k}$. \end{definition}

The uncertainty score quantifies how informative each $D_{i,k}$ is by subtracting from 1 the sum of the square of each label's probability mass.
The uncertainty score is related to collision entropy (discussed in Appendix \ref{sec:FAMEShift:techdetails}), and determines the optimal allocation as follows.

\begin{lemma}\label{lemma:FAMEShift:optimalsampling}
Let $A^*$ be the sample allocation algorithm that achieves  the smallest expected squared Frobenius norm error.
Then the number of samples drawn from $D_{i,k}$  by $A^*$ is  %
\begin{equation*}
    \SN_{i,k}^* = \frac{\pmb p_{i,k} \usc_{i,k}}{\sum_{i,k}\pmb p_{i,k} \usc_{i,k}} N
\end{equation*}
\end{lemma}


Lemma \ref{lemma:FAMEShift:optimalsampling} shows that the optimal budget allocation depends on the uncertainty score, but in practice,  we do not know the uncertainty score before drawing samples and querying the ML API.
Thus, a natural question is how to estimate the uncertainty score $\usc_{i,k}^2$.
Suppose $n$ samples, $x_1,x_2,\cdots, x_n$, are drawn from  partition $D_{i,k}$.
Then we can estimate $\usc_{i,k}^2$ by %
\begin{equation}\label{equ:FAMEAPIShift:uncertainscoreestimator}
\hatusc_{i,k}^2 \triangleq  1 - \frac{1}{ n (n-1)} \sum_{s=1}^{n} \sum_{t:t=1, t\not=s}^{n} \mathbbm{1}_{\hat{y}(x_s)=\hat{y}(x_t)}
\end{equation}

Note that naively computing the estimated uncertainty score as above can incur a computational cost quadratic in the number of samples $n$, which is prohibitive for large $n$. Later we will show how this can be updated in an online fashion with total cost linear in $n$.

\subsection{An Uncertainty-aware Adaptive Sampling Algorithm}
\begin{algorithm}[t]
\caption{\systemnameAPIShift{}'s ML API shift assessment algorithm.}
	\label{Alg:FAMEShift:MainAlg}
	\SetKwInOut{Input}{Input}
	\SetKwInOut{Output}{Output}
	\Input{ML API $\hat{y}(\cdot)$, query budget $N$, partitions $D_{i,k}$, $\pmb p\in\R^{L\times K}$, $\SCM^o\in\R^{L\times L}$, and $a> 0$}
	\Output{Estimated ML API Shift $\hatdeltaCM \in \R^{L \times \hat{L}}$}


  Set $ \SN=\pmb 0_{L\times K}, \Smu= \pmb 0_{L\times K\times L}, \hatusc= \pmb 0_{L\times K}, \hash = \pmb 0_{L\times K
  \times L}$  \hfill{$\triangleright$ Initialization}
  
  \For{$n\gets 1$ \KwTo $N$}
  {
 $ (i^*,k^*) \gets \begin{cases}
  (i,k), & \textit{if  $\SN_{i,k}<2$}\\
  \arg \max_{i,k} \frac{\pmb p_{i, k}}{\SN_{i, k}} \left(\hatusc_{i,k} +  \sqrt[4]{\frac{a}{\SN_{i,k}}}\right), & \textit{o/w}
  \end{cases}$ \hfill{$\triangleright$ Determine data partition}

  Sample $x_n$ from $D_{i^*,k^*}$ and query the ML API to obtain $\hat{y}(x_n)$
  
  $ \SN_{i^*,k^*}\gets \SN_{i^*,k^*}+1$ \hfill{$\triangleright $ Update sample size}
  
$\Smu_{i^*,k^*,j} \gets \Smu_{i^*,k^*,j} + \frac{\mathbbm{1}_{\hat{y}(x_n)=j}-\Smu_{i^*,k^*,j}}{\SN_{i^*,k^*}}, \forall j \in [L] $
   \hfill{$\triangleright$ Update predicted label distribution}
   
   $\hatusc_{i^*,k^*}^2 \gets\begin{cases}
   \frac{1}{2} \hash_{i^*,k^*,\hat{y}(x_n)}, & \textit{if $\SN_{i^*,k^*}<2$}\\
   \hatusc_{i^*,k^*}^2 + \frac{1 -  \frac{\hash_{i^*,k^*,\hat{y}(x_n)}}{\SN_{i^*,k^*}-1} - \hatusc^2_{i^*,k^*}}{\SN_{i^*,k^*}} , & \textit{o/w}\end{cases}$  \hfill{$\triangleright$ Update uncertainty score}
  
    $\hash_{i^*,k^*,\hat{y}(x_n)} \gets \hash_{i^*,k^*,\hat{y}(x_n)} +1$ \hfill{$\triangleright$ Update label frequency }
     
  }

  Return $\hatdeltaCM \in \R^{L\times L}$ 
  where $\hatdeltaCM_{i,j} = \sum_{k=1}^{K} \pmb p_{i,k} \Smu_{i,k,j}-\SCM^o_{i,j}, \forall i,j$  \hfill{$\triangleright$ Confusion estimation }
\end{algorithm}

Now we have a chicken-and-egg problem: estimating the uncertainty scores is needed to find the optimal sample allocation, but sampling from all partitions is needed to estimate their uncertainty scores.
To overcome this issue, we adopt an iterative sampling approach, as shown in Algorithm \ref{Alg:FAMEShift:MainAlg}.
At each iteration, it alternates between (i) uncertainty score-based new sample selection (line 3 - 4) and (ii) uncertainty score and predicted label distribution update using the new sample (line 5 - 8).
After querying the ML API $N$ times, the API shifts are obtained by (iii) fusing the estimated predicted label distribution on each  partition (line 10). 
We give the details as follows. 

\paragraph{Uncertainty score and predicted label distribution update.}
After obtaining the predicted label for a sample from partition $D_{i^*,k^*}$, we need to update (i) the number of samples already drawn from this partition, denoted by $\SN_{i^*,k^*}$, (ii) the estimated predicted label distribution, denoted by
 $\Smu_{i^*,k^*,j}, \forall j$, and  (iii)  the estimated uncertainty score,  $\hatusc^2_{i^*,k^*}$.
 For $\SN_{i^*,k^*}$ and $\Smu_{i^*,k^*,j}$ (line 5-6), we use standard incremental update approach  \cite{onlinesample}, which requires constant space and computational cost per iteration. 
For $\hatusc^2_{i^*,k^*}$, naive incremental update for $\hatusc^2_{i^*,k^*}$  using equation  \ref{equ:FAMEAPIShift:uncertainscoreestimator} requires comparing computational cost linear in number of drawn samples at each iteration.
This leads to an overall computational complexity quadratic in the number of samples.
To overcome this, we additionally maintain the number of label $j$ being predicted among all samples drawn from $D_{i^*,k^*}$, denoted by $\hash_{i^*,k^*,j}$ (line 8).
The key insight is that  the number of label $j$ being predicted is sufficient statistics for uncertainty score estimation.
Furthermore,  incrementally updating the number of predicted labels $\hash_{i^*,k^*,j}$ and the uncertainty score given $\hash_{i^*,k^*,j}$ is both fast. 
This enables a fast incremental update of  $\hatusc^2_{i^*,k^*}$ (line 7).

 \paragraph{Uncertainty score-based new sample selection.}
To determine on which partition to select a new sample, we use an upper-confidence-bound approach on the weighted uncertainty score (second case in line 3), after ensuring two samples have been drawn from each partition (first case in line 3).
Two samples are needed for an initial estimation of each partition's uncertainty score.
Here, we use a parameter $a>0$ to balance between exploiting knowledge of uncertainty score ($\hatusc^2_{i,k}$) and exploring more partitions ($\sqrt[4]{\frac{1}{\SN_{i,k}}}$).
To see this, consider the extreme.
If $a$ is infinite, $\sqrt[4]{\frac{1}{\SN_{i,k}}}$ dominates, and the algorithm always selects the partition with least number of observed samples, which forces the same number of samples from each partition finally. 
If $a=0$, $N$ is infinite and thus $\hatusc_{i,k} = \usc_{i,k}$, the algorithm forces $\frac{\pmb p_{i, k}}{\SN_{i, k}} \usc_{i,k}$ to be identical $\forall i,k$.
Thus, sample number of partition $D_{i,k}$ is proportional to $\pmb p_{i,k} \usc_{i,k}$, which is the optimal allocation. 

We quantify the performance of  \systemnameAPIShift{} v.s. the optimal allocation algorithm $A^*$ as follows. %

\begin{theorem}\label{thm:FAMEAPIShift:mainbound}
If $a>2 \log L + \log K + \frac{9}{4}\log N$ and $N>4LK$, 
then we have
\begin{equation*}
    \loss(\systemnameAPIShift{},N) - \loss(\mathcal{A}^*, N) \leq O(N^{-\frac{5}{4}}\log^{\frac{1}{4}} N)
\end{equation*}
\end{theorem}
Roughly speaking, Theorem \ref{thm:FAMEAPIShift:mainbound} shows that the loss gap between the API shift estimated by \systemnameAPIShift{} and the (unreachable) optimal allocation algorithm ceases in the rate of $N^{-5/4}$. 
Note that the loss of the optimal allocation decays  in the rate of $N^{-1}$. Thus, as $N$ gets larger and larger, the relative gap becomes more and more negligible. 
Another advantage of \systemnameAPIShift{} is that its computation cost is only linear in $N$, making it suitable for large $N$.
$N>4 LK$ is needed to obtain a few samples from each data partition for initial uncertainty estimation. 
There are in total $LK$ many partitions. 
Thus, the initial sample size linear in $LK$ is necessary.

\begin{figure}[t]\label{fig:FAMEShift:casestudy}
	\centering
	\begin{subfigure}[CM 2020]{\includegraphics[width=0.24\linewidth]{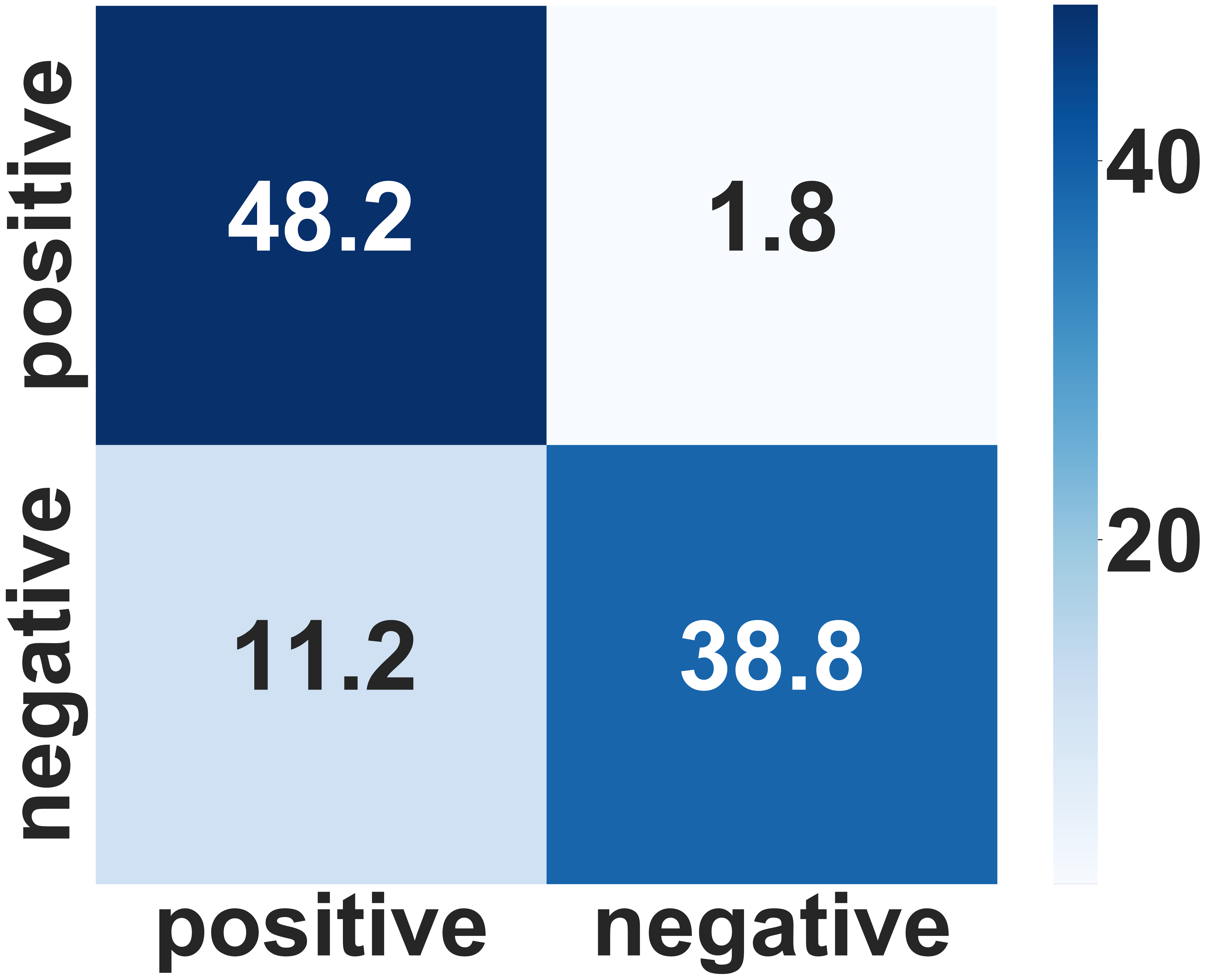}}
    \end{subfigure}
    \begin{subfigure}[CM 2021]{\includegraphics[width=0.24\linewidth]{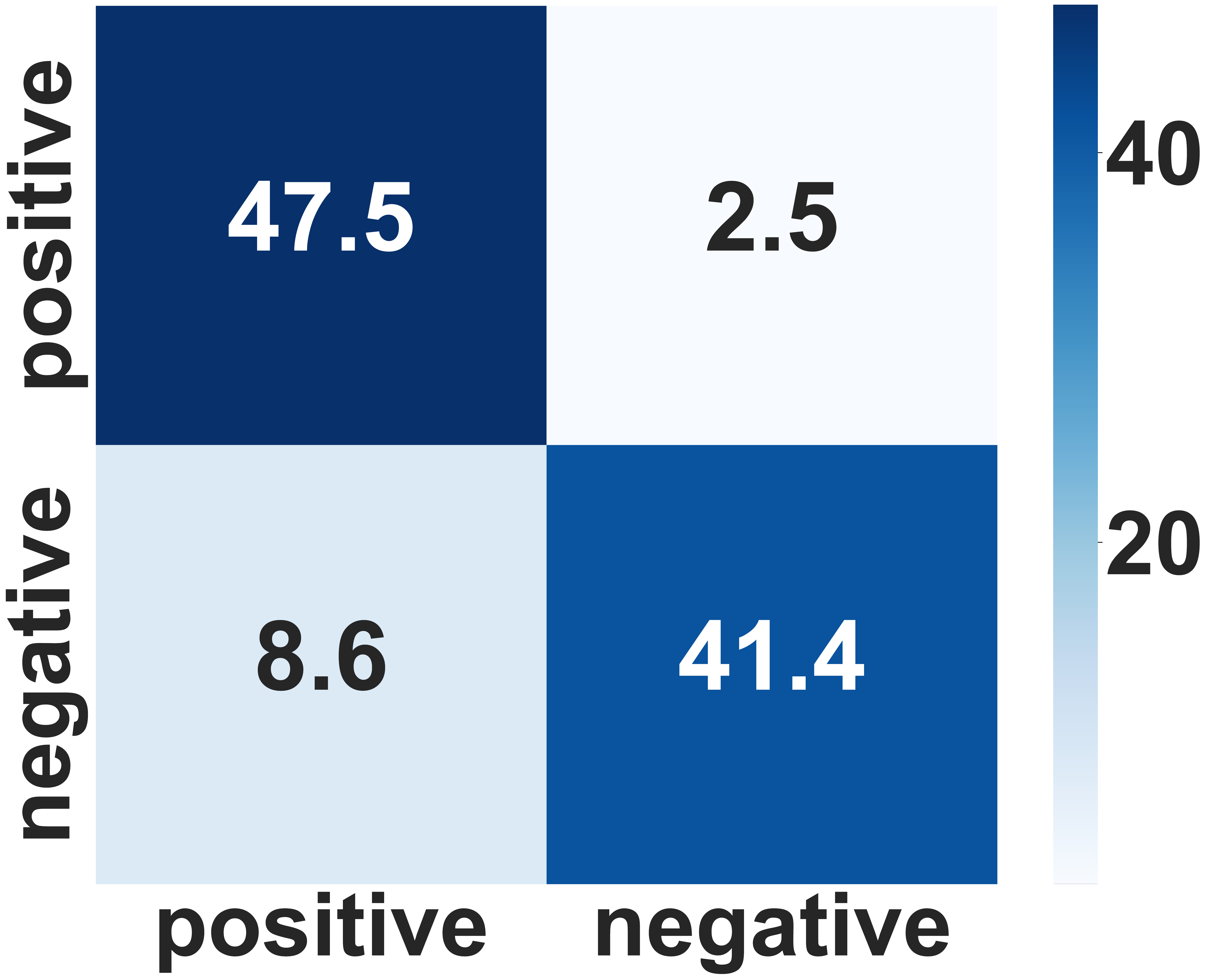}}
    \end{subfigure}
    \begin{subfigure}[True API shift ]{\includegraphics[width=0.24\linewidth]{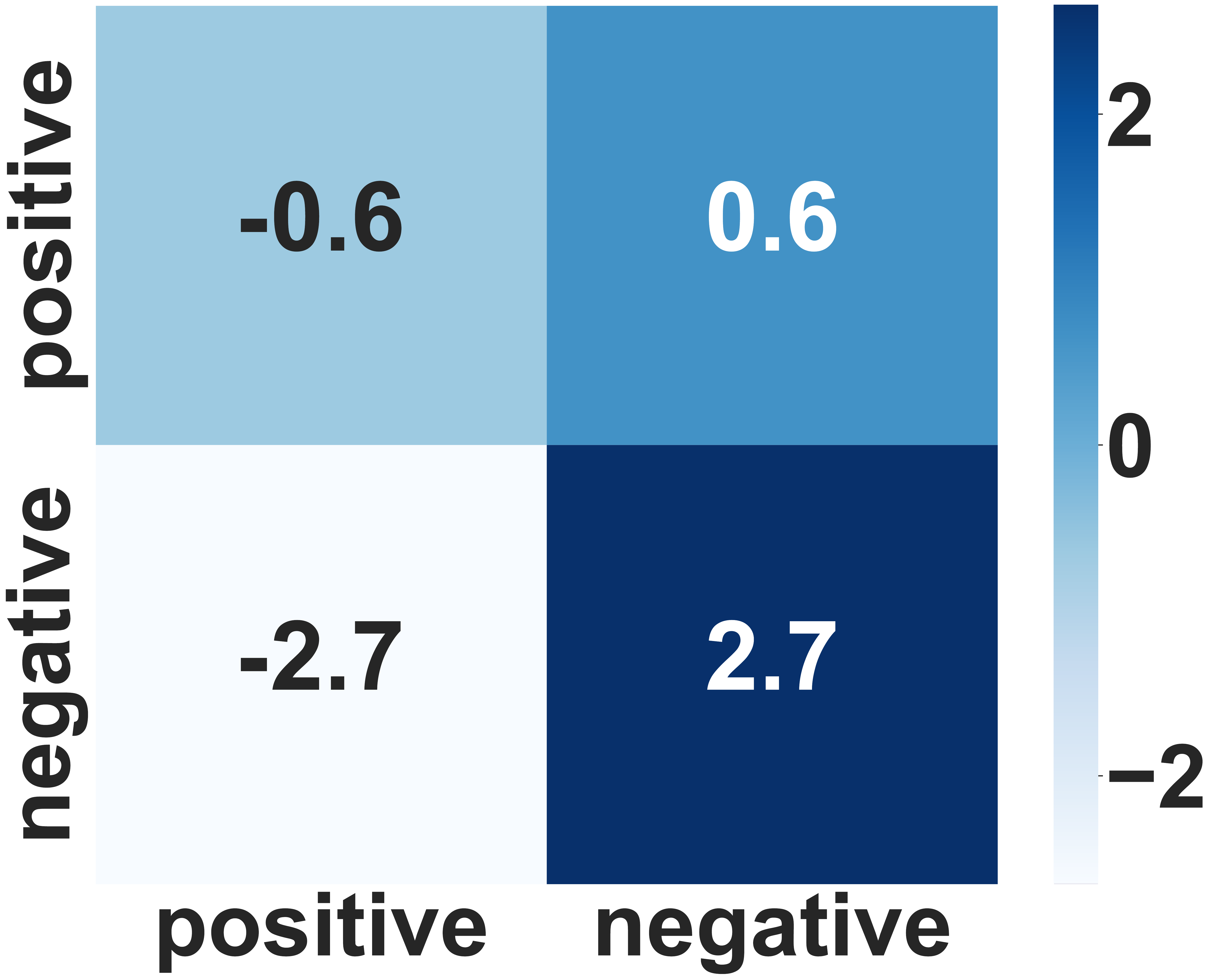}}
    \end{subfigure}
    \begin{subfigure}[Estimated Shift]{\includegraphics[width=0.24\linewidth]{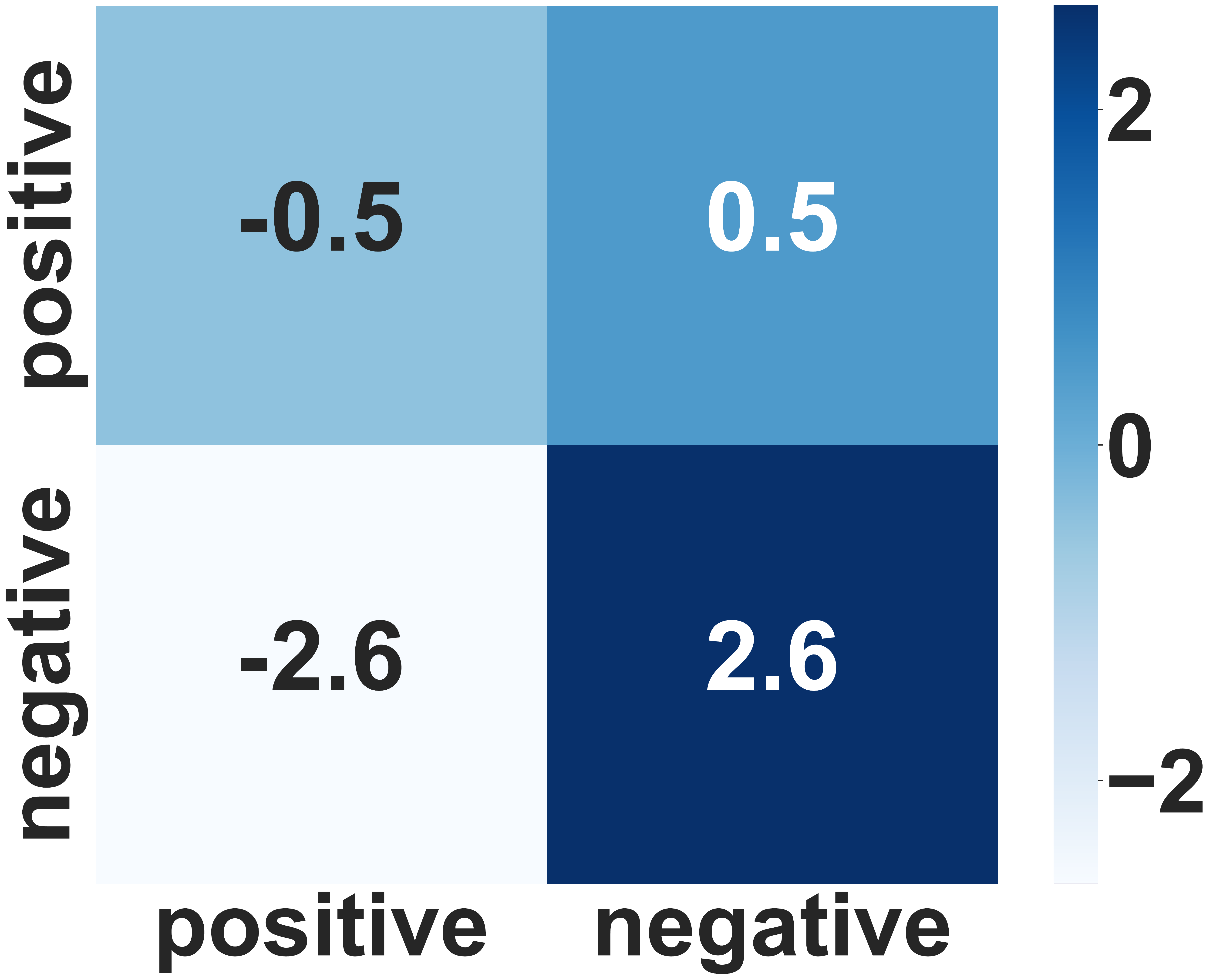}}
    \end{subfigure}
    
    \begin{subfigure}[Partition Size]{\includegraphics[width=0.24\linewidth]{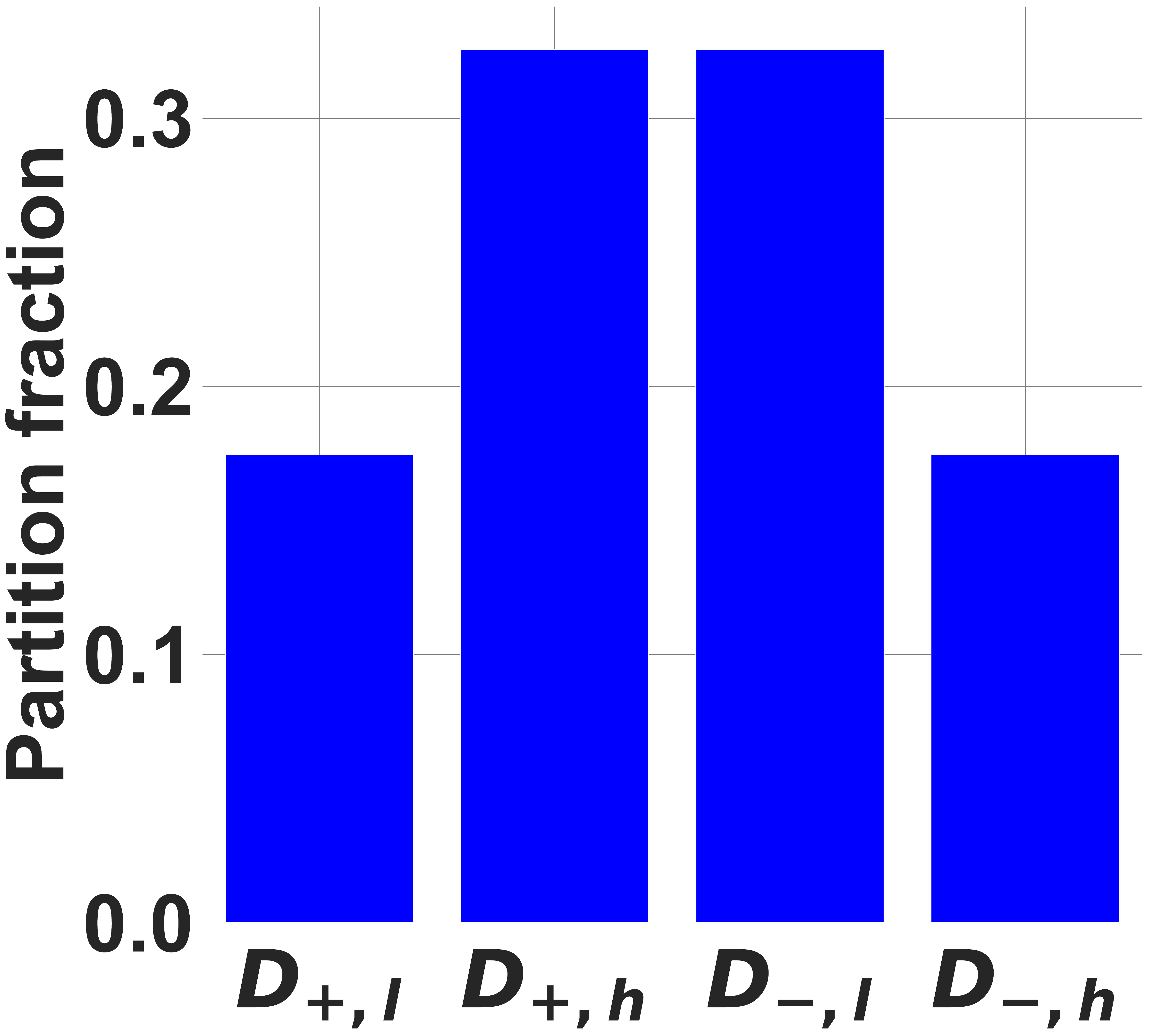}}
    \end{subfigure} 
    \begin{subfigure}[Uncertainty score]{\includegraphics[width=0.24\linewidth]{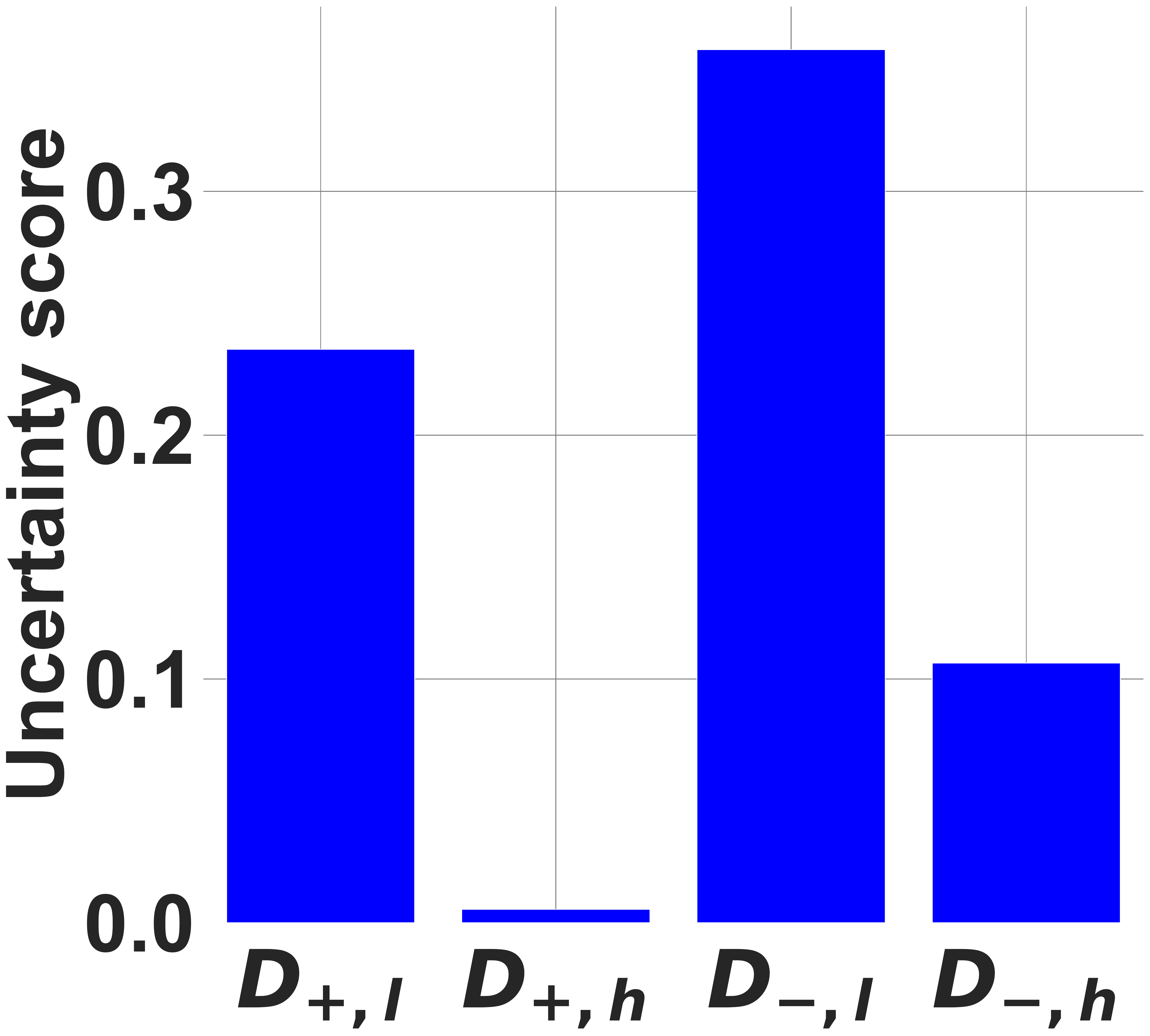}}
    \end{subfigure} 
    \begin{subfigure}[Sample allocation]{\includegraphics[width=0.24\linewidth]{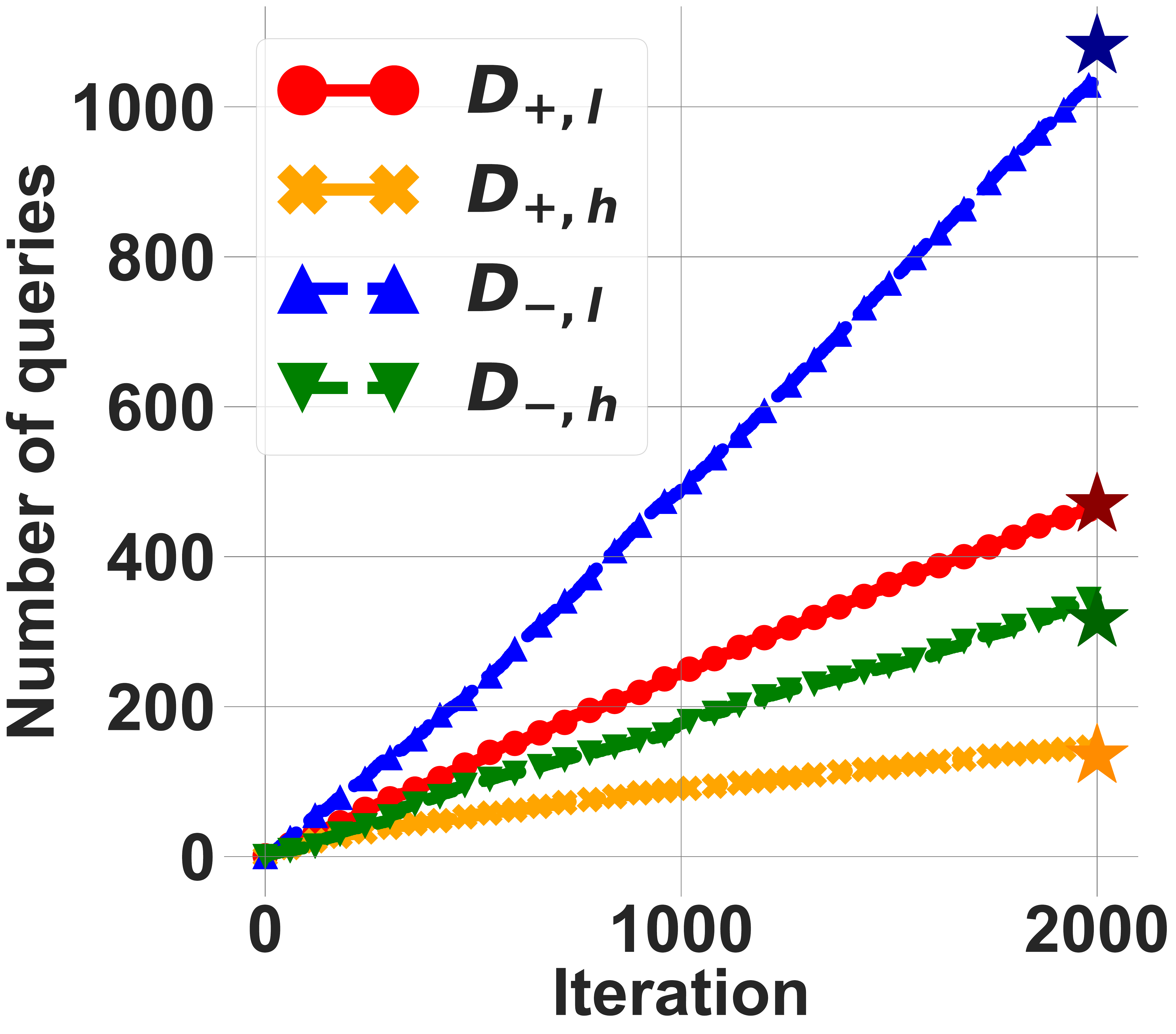}}
    \end{subfigure} 
    \begin{subfigure}[Performance]{\includegraphics[width=0.24\linewidth]{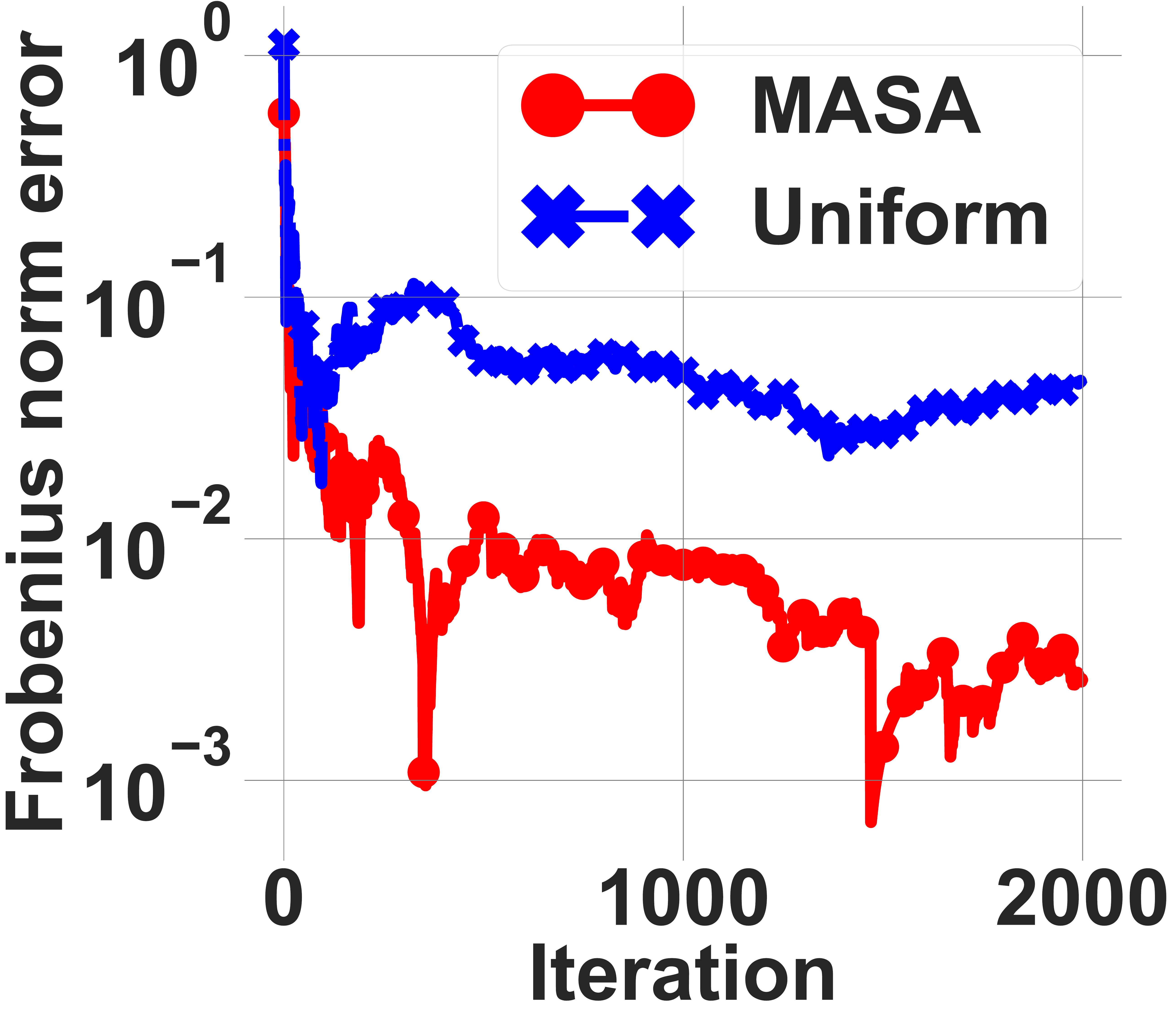}}
    \end{subfigure}    
    \caption{Case study for Amazon API's performance shift on dataset YELP. (a) and (b) give its confusion matrix in spring 2020 and spring 2021, respectively. (c) is their differences, i.e., the API shift. \systemnameAPIShift{}'s estimated shift using 2000 samples is in (d). 
    The dataset is divided into 4 partitions based on (i) positive ($+$) or negative ($-$) true labels, and (ii) low ($l$) or high ($h$)  quality score. (e) and (f) give the size and uncertainty score of each partitions.
    (g) shows \systemnameAPIShift{}'s sampling decision per iteration, where the dark dot points represent the (unreachable) optimal sample allocation.  (h) reveals its performance. 
    }
\end{figure}

\section{Experiments}\label{Sec:FAMEShift:Experiment}

We apply \systemnameAPIShift{}
to estimate the shifts of several real world ML services for various tasks. 
Our goal is three-fold:  (i) understand if and why \systemnameAPIShift{} assess the API shifts efficiently, (ii) examine how much sample cost \systemnameAPIShift{} can reduce compared to standard sampling,   and (iii) exploit the trade-offs between estimation accuracy and query cost achieved by \systemnameAPIShift{}. 
We also study how the hyperparameters affect \systemnameAPIShift{}'s performance, left to Appendix \ref{sec:FAMEShift:experimentdetails}. 

\begin{table}[t]
  \centering
  
  \caption{Required sample size to reach 1\%  Frobnius norm error. Here we compare \systemnameAPIShift{} with uniform (U) sampling and stratified (S) sampling. U and S required similar sample sizes and are reported in the same column.  The sample size is obtained when a 1\% Frobenius norm error is achieved with probability 95\%.}
    \begin{tabular}{|c||c|c|c|c||c|c|c|}
    \hline
    \multirow{2}[4]{*}{API;Dataset} & \multicolumn{2}{c|}{Sample size} & \multirow{2}[4]{*}{Save} & \multirow{2}[4]{*}{API;Dataset} & \multicolumn{2}{c|}{Sample size} & \multirow{2}[4]{*}{Save} \bigstrut\\
\cline{2-3}\cline{6-7}          & MASA  & U/S   &       &       & MASA  & U/S   &  \bigstrut\\
    \hline
    \hline
    Amazon;YELP & 4.5K  & 19.7K & 77\%  & IBM;DIGIT & 3.6K  & 17.0K & 79\% \bigstrut\\
    \hline
    Amazon;IMDB & 10.3K & 20.8K & 51\%  & IBM;AMNIST & 2.4K  & 18.5K & 87\% \bigstrut\\
    \hline
    Amazon;WAIMAI & 7.8K  & 18.0K & 57\%  & Google;DIGIT & 4.2K  & 17.0K & 75\% \bigstrut\\
    \hline
    Amazon;SHOP & 4.8K  & 20.8K & 77\%  & Google;AMNIST & 1.1K  & 18.5K & 94\% \bigstrut\\
    \hline
    MS; FER+ & 2.6K  & 19.9K & 87\%  & Google;CMD & 1.6K  & 15.2K & 89\% \bigstrut\\
    \hline
    Google; EXPW & 4.2K  & 17.9K & 77\%  & MS;DIGIT & 3.3K  & 17.0K & 81\% \bigstrut\\
    \hline
    \end{tabular}%
  \label{tab:FAMEShift:costsaving}%
\end{table}%
\begin{figure}[t]
	\centering
    \begin{subfigure}[Amazon YELP]{\includegraphics[width=0.24\linewidth]{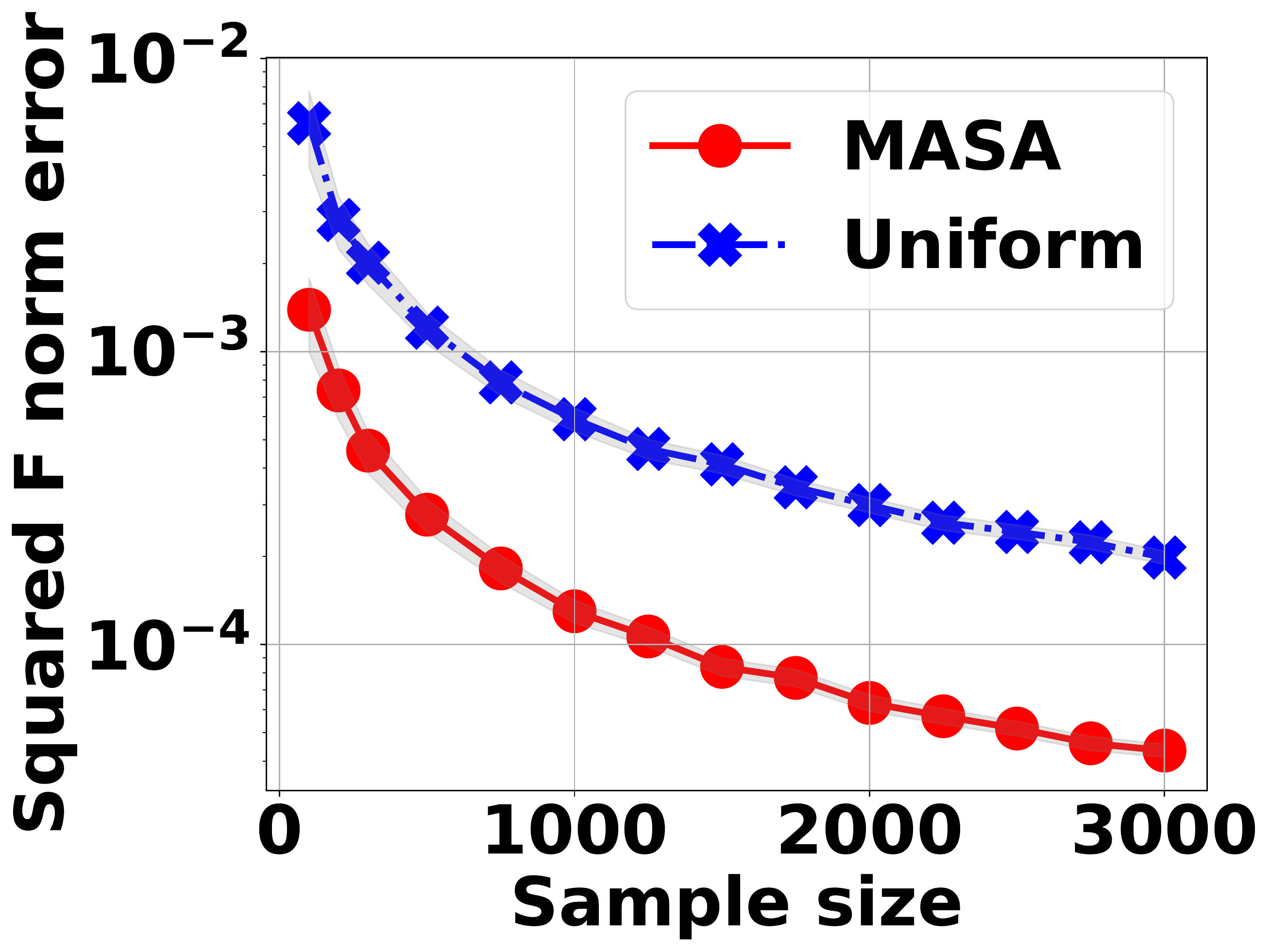}}
    \end{subfigure}
    \begin{subfigure}[Amazon IMDB]{\includegraphics[width=0.23\linewidth]{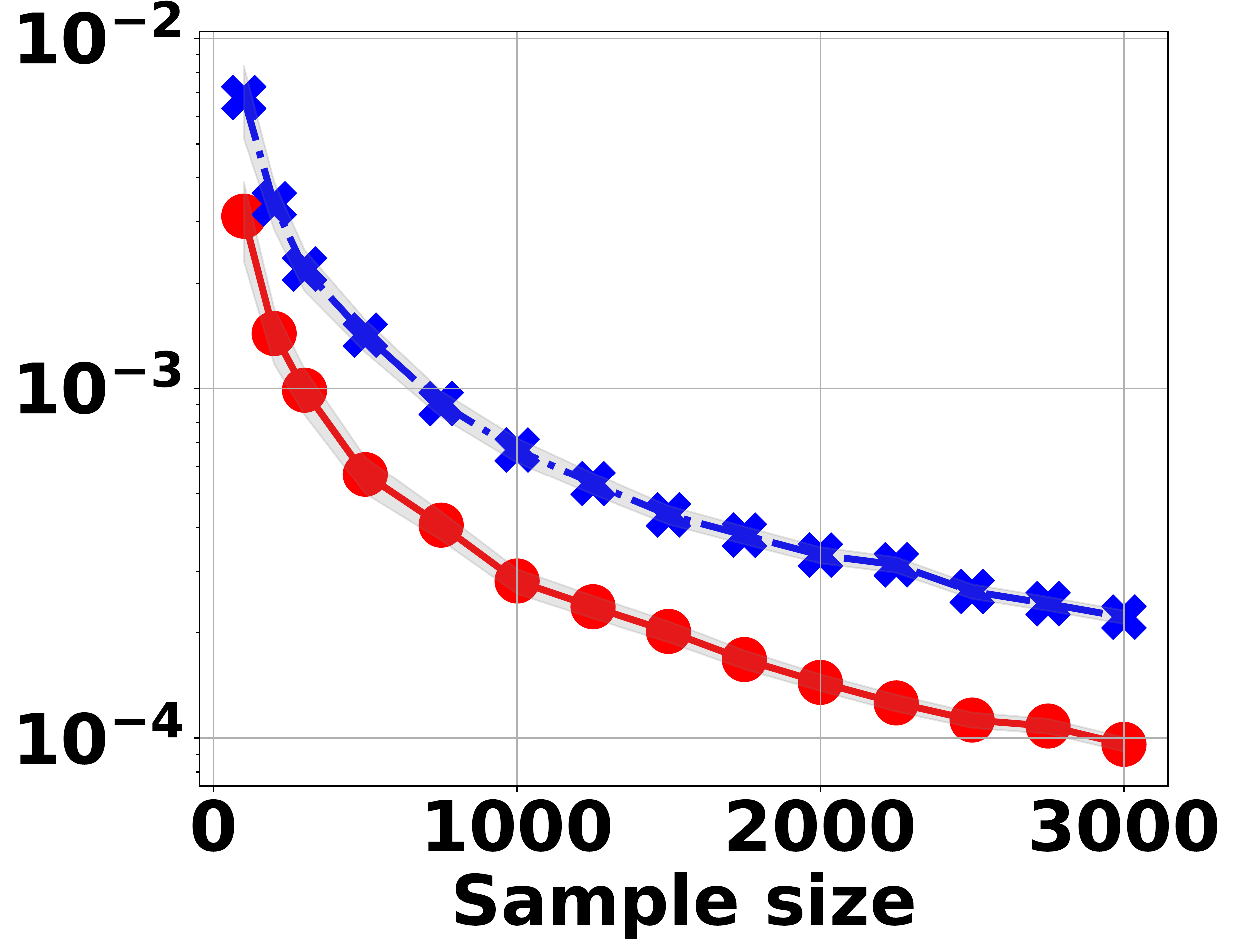}}
    \end{subfigure}
    \begin{subfigure}[Amazon WAIMAI]{\includegraphics[width=0.23\linewidth]{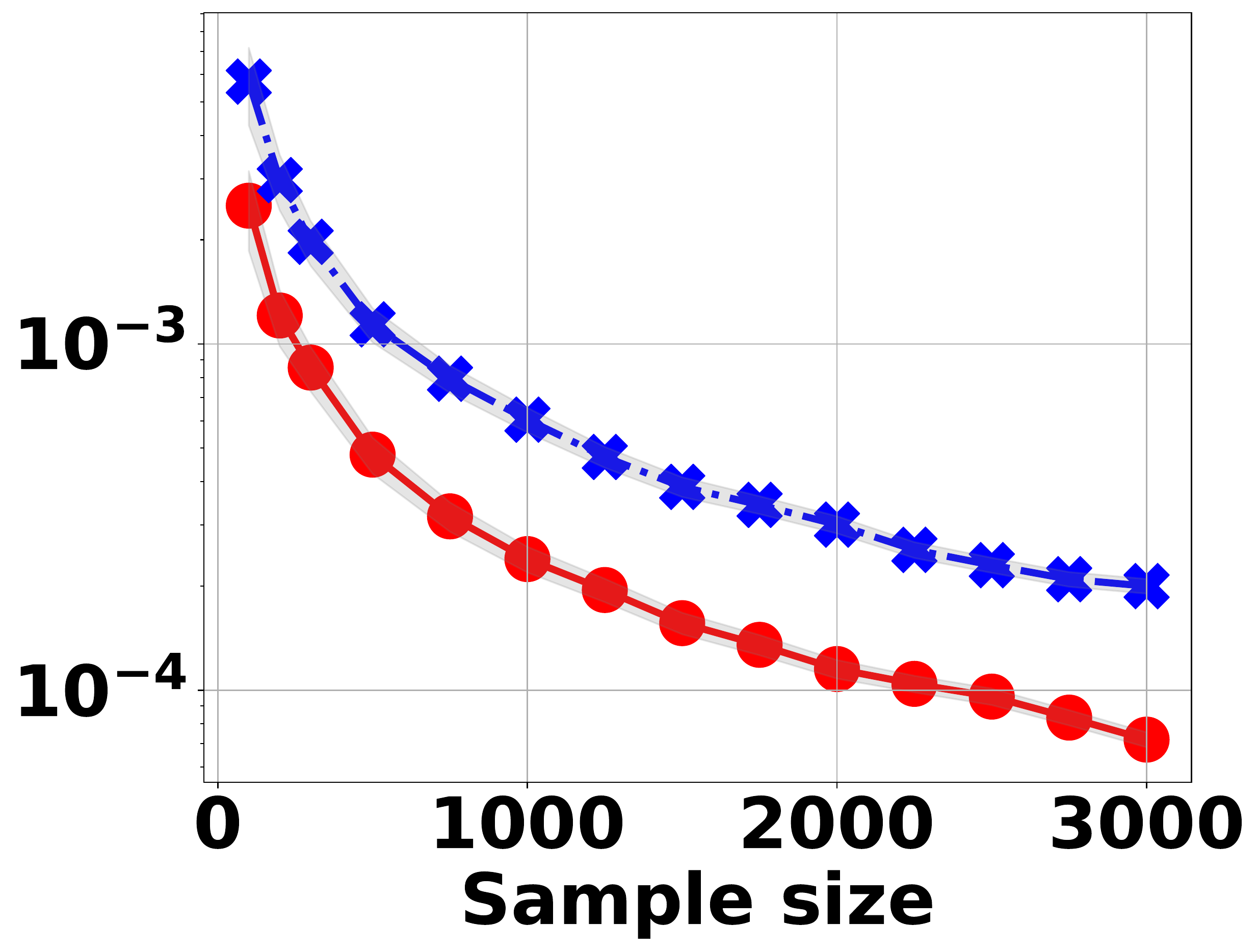}}
    \end{subfigure}
    \begin{subfigure}[Amazon SHOP]{\includegraphics[width=0.23\linewidth]{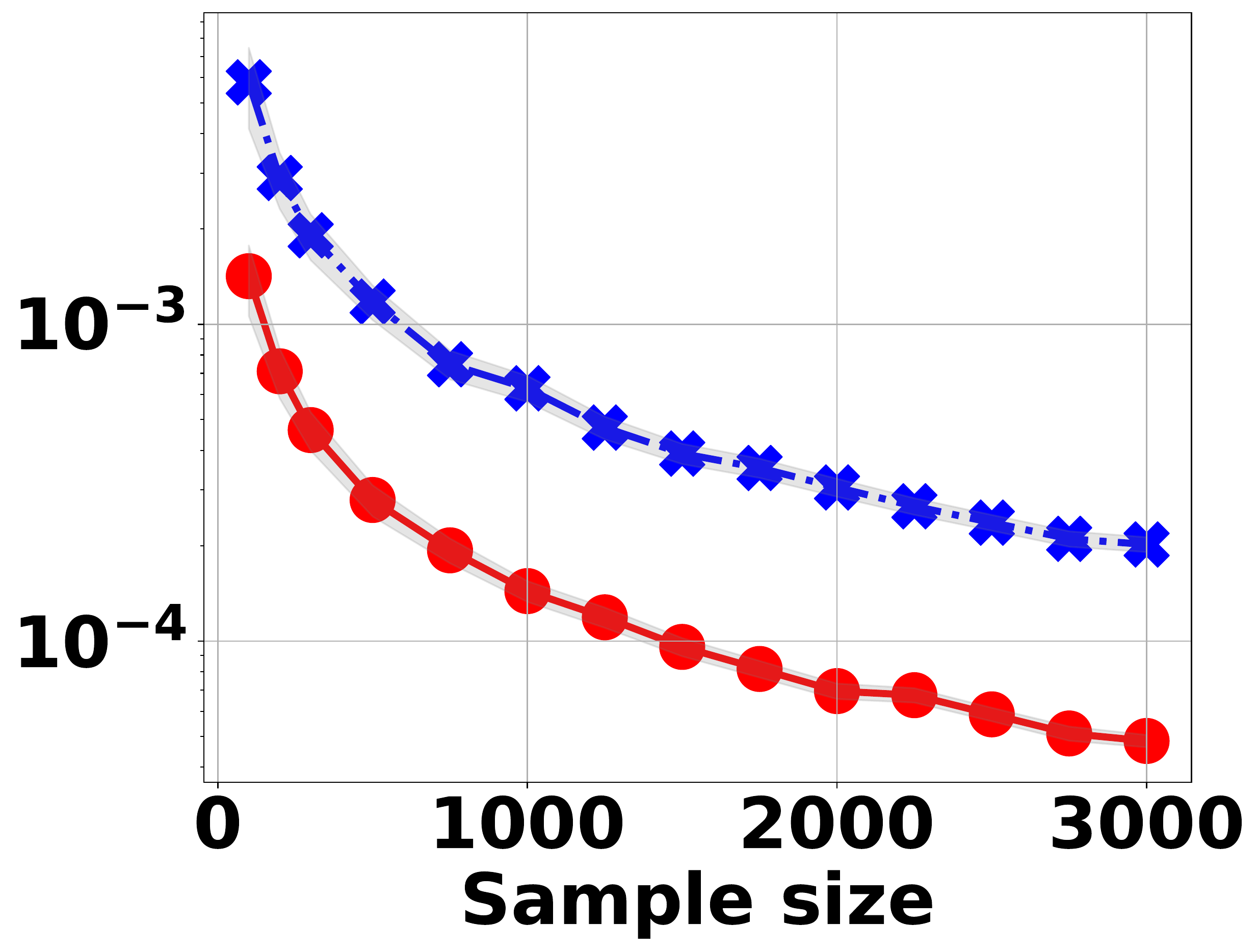}}
    \end{subfigure}
    
	\begin{subfigure}[Microsoft FER+]{\includegraphics[width=0.23\linewidth]{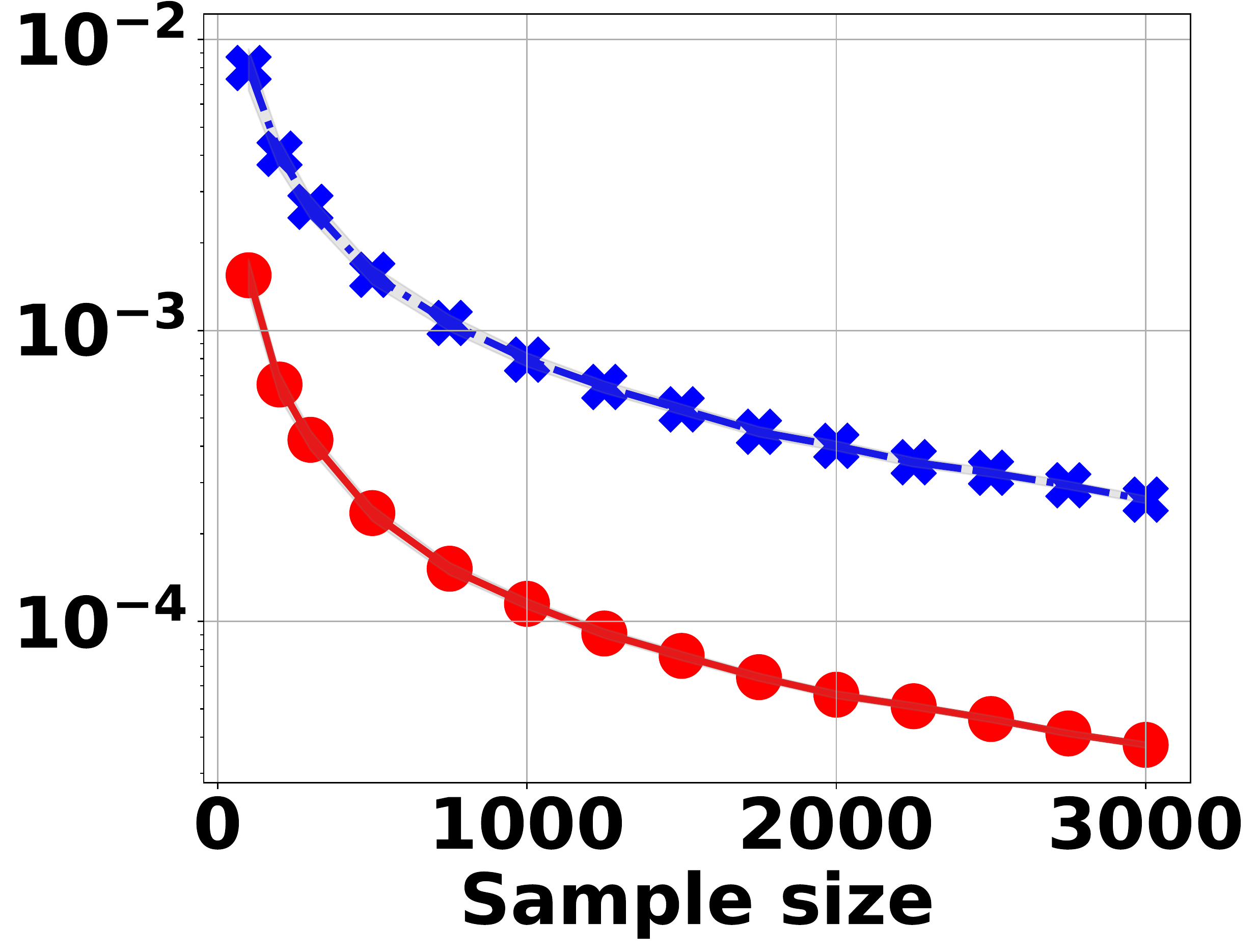}}
    \end{subfigure}
    \begin{subfigure}[Google EXPW]{\includegraphics[width=0.23\linewidth]{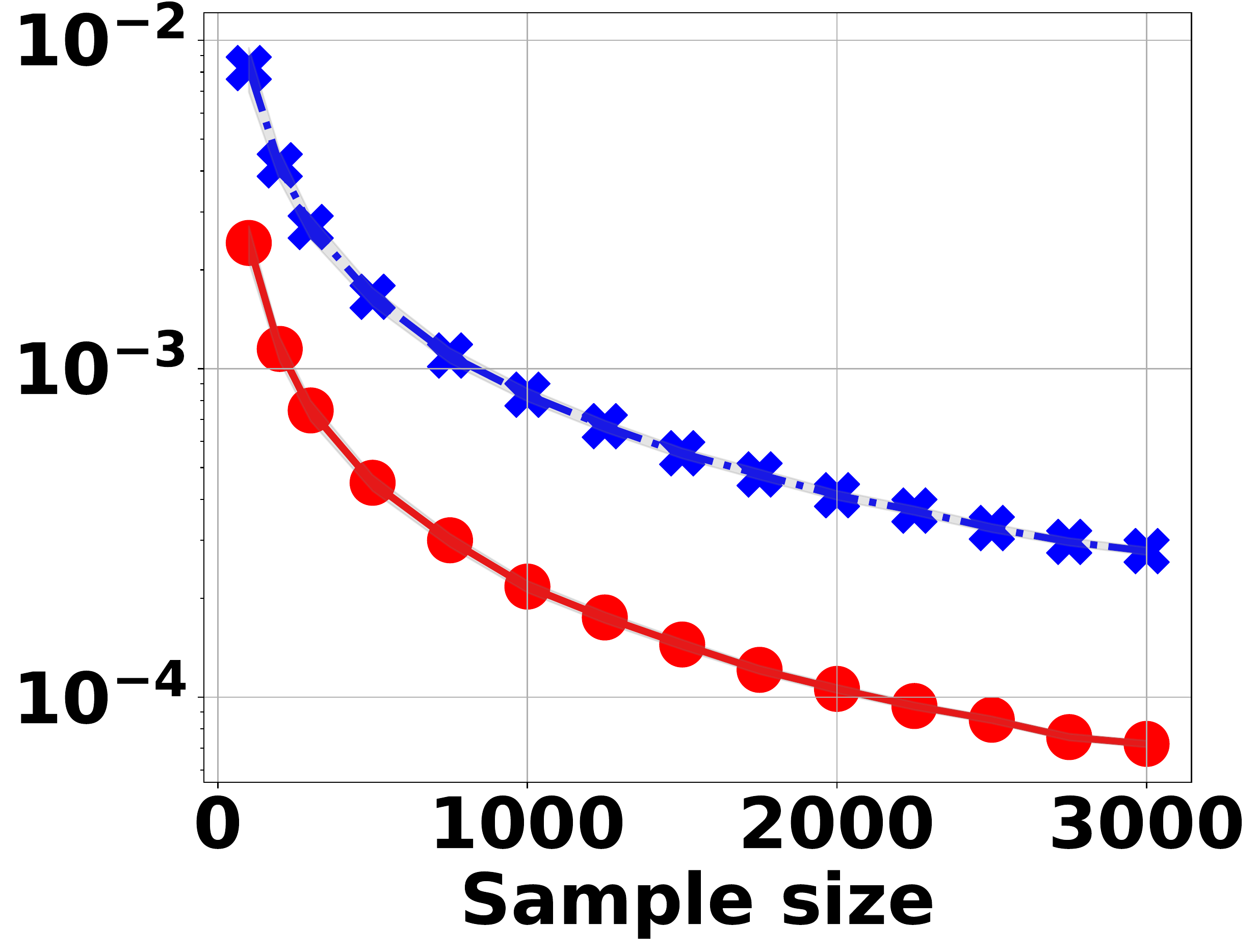}}
    \end{subfigure}
    \begin{subfigure}[Google DIGIT ]{\includegraphics[width=0.23\linewidth]{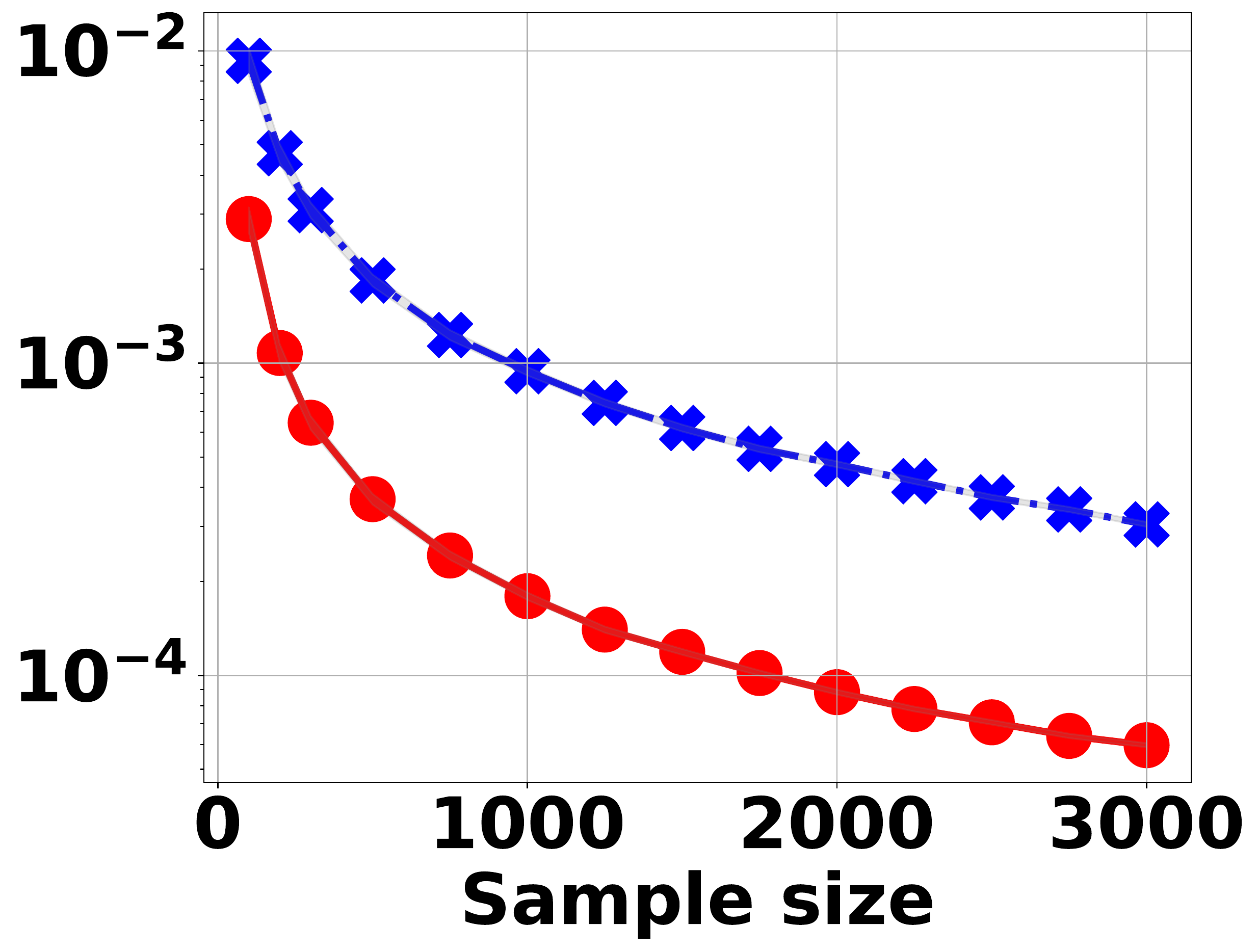}}
    \end{subfigure}
    \begin{subfigure}[Google AMNIST]{\includegraphics[width=0.23\linewidth]{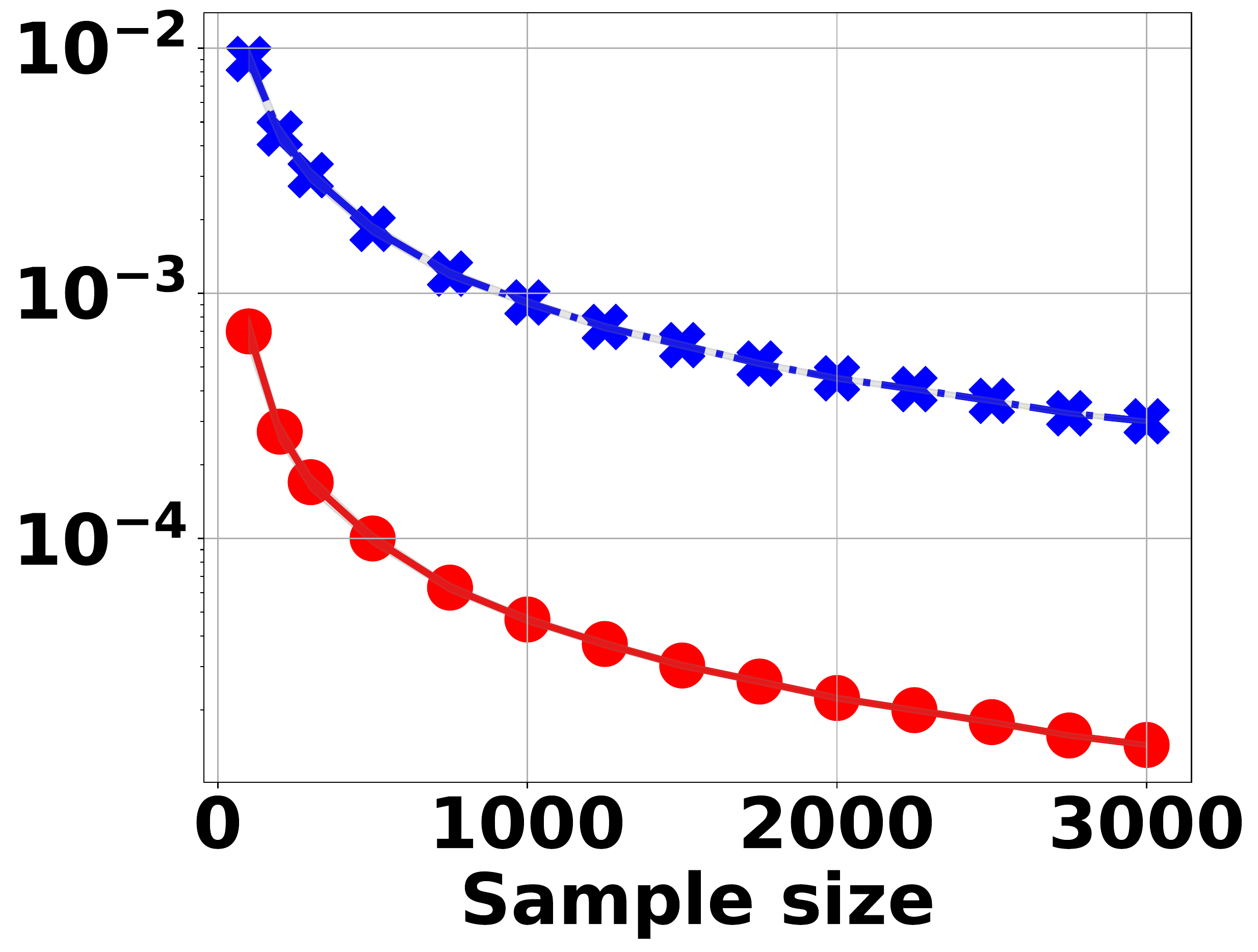}}
    \end{subfigure}
    
    \begin{subfigure}[ Google CMD]{\includegraphics[width=0.23\linewidth]{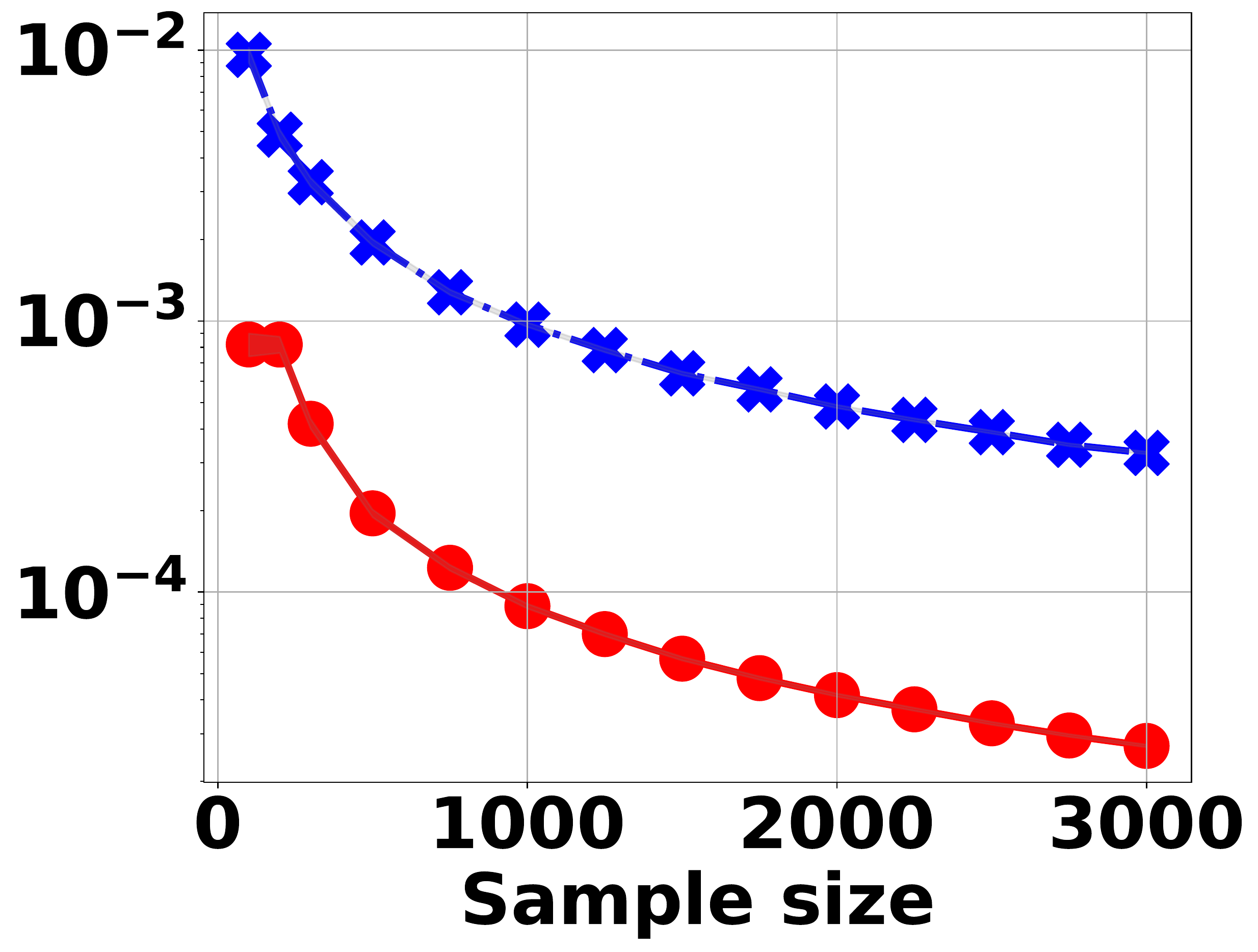}}
    \end{subfigure}
    \begin{subfigure}[IBM DIGIT ]{\includegraphics[width=0.23\linewidth]{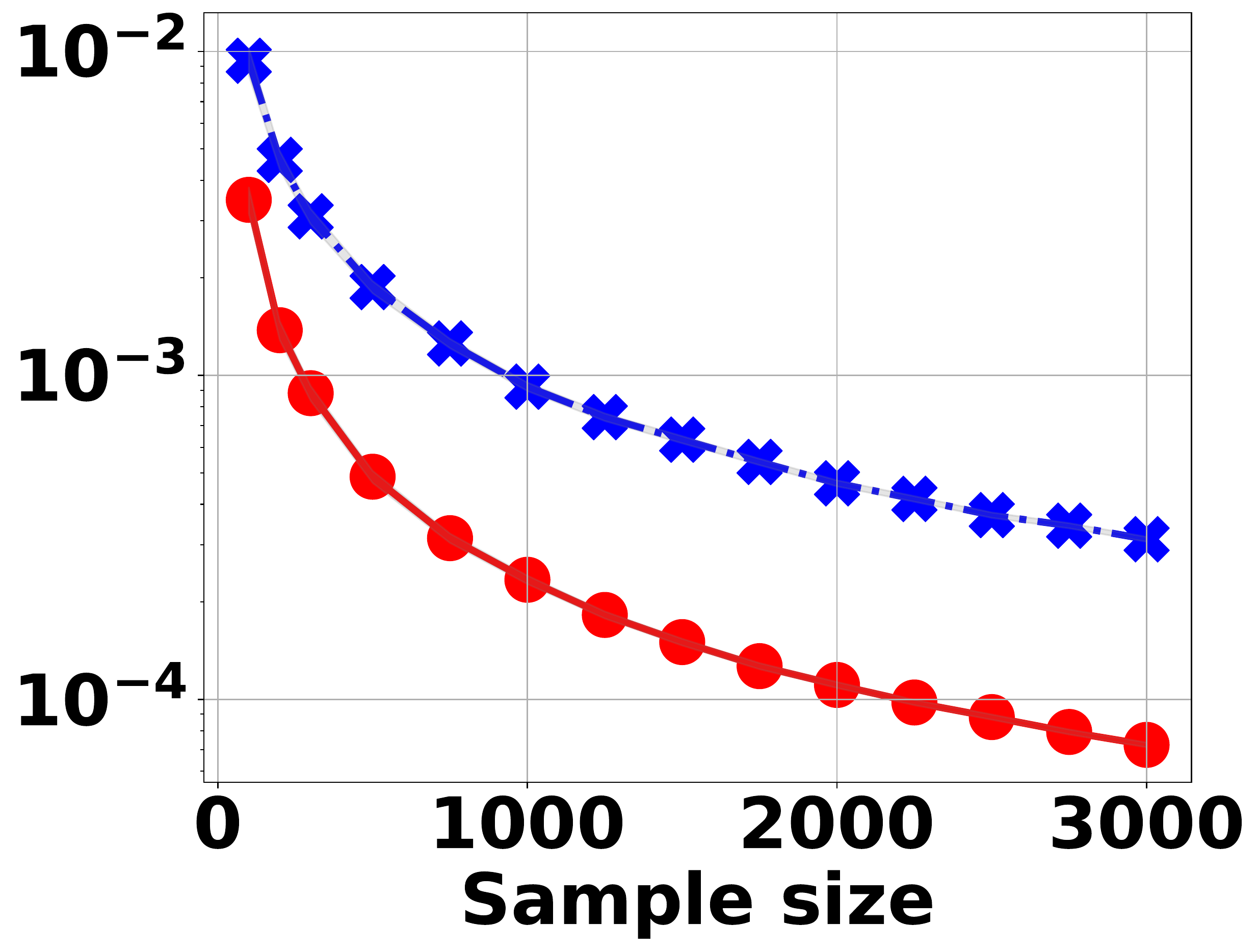}}
    \end{subfigure}
    \begin{subfigure}[IBM AMNIST]{\includegraphics[width=0.23\linewidth]{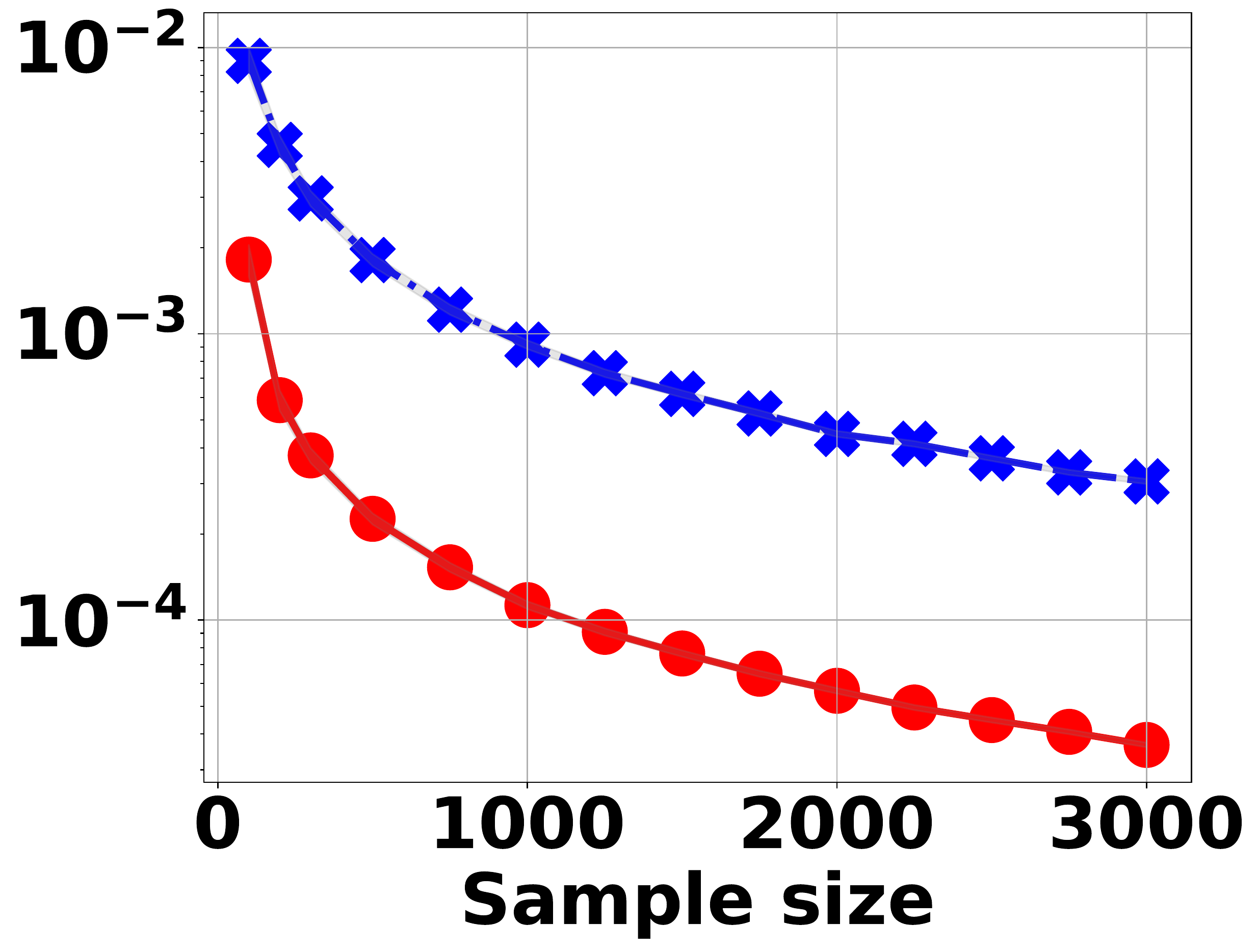}}
    \end{subfigure}
    \begin{subfigure}[Microsoft DIGIT]{\includegraphics[width=0.23\linewidth]{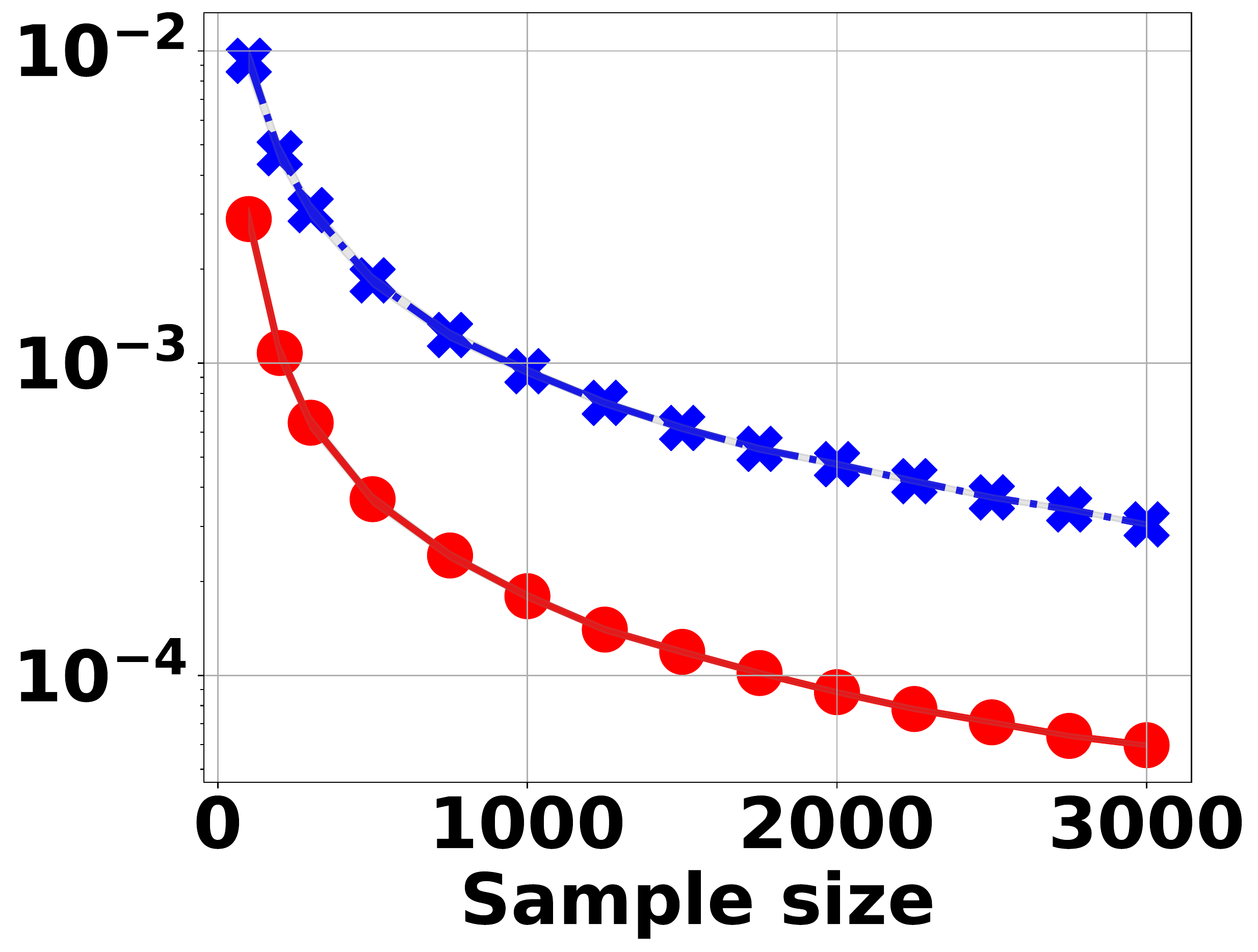}}
    \end{subfigure}
    \caption{API shift estimation performance and sample size trade-offs. We compare the expected squared Frobenius norm error of   \systemnameAPIShift{} with $K=3$ partitions versus standard uniform sampling. 
    For any sample size, \systemnameAPIShift{} consistently leads to an estimation error  much smaller than uniform sampling across different API and dataset combinations. }\label{fig:FAMEShift:tradeoffs}
\end{figure}

\paragraph{Tasks,  ML APIs, and datasets.}
As shown in Section \ref{Sec:APIShift:Preli}, we have observed 12 of 36 cases where there is a >1\% overall accuracy performance change of an ML API.
Thus, we focus on \systemnameAPIShift{}'s performance on those 12 cases. Except for case study, all experiments were averaged over 1500 runs. 
In all tasks, we created partitions using difficulty levels induced by a cheap open source model from GitHub.
More details  are in Appendix \ref{sec:FAMEShift:experimentdetails}.

\paragraph{Sentiment analysis: a case study on Amazon API.}
We start by a case study on Amazon API on a sentiment analysis dataset, YELP to understand \systemnameAPIShift{}'s performance. 
We adopt \systemnameAPIShift{} with sample budget 2000. 
The dataset is divided into 4 partitions $D_{+,l}, D_{+,h}, D_{-,l}, D_{-,h}$, depending on whether the true label is positive (+) or negative (-), and quality score produced by the 2020 version is lower (l) or higher (h) than the median.

We first note that the API shift gives an interesting explanation to Amazon API's accuracy change. 
In fact, as shown in Figure \ref{fig:FAMEShift:casestudy} (a-c), the accuracy increase is mostly because more texts (2.7\%) with negative altitudes are correctly classified. 
On the other hand, there is a small drop (0.6\%) in the prediction accuracy for positive texts.  
One possible explanation is that the API has been retrained on a dataset with more negative texts.
Next, we observe that \systemnameAPIShift{} produces accurate estimation of the API shift by comparing Figure \ref{fig:FAMEShift:casestudy} (c) and (d).
As shown in Figure \ref{fig:FAMEShift:casestudy} (d), there are only negligible differences (0.1\%) between the true API shift and that estimated by \systemnameAPIShift{} using 2000 samples. 
This is primarily due to (i) that the data partitioning separates  more uncertain data from less uncertain ones, and (ii) that adaptive sampling learns the uncertainty level effectively.
Figure \ref{fig:FAMEShift:casestudy} (e-f) shows that  while the partitions' size is similar, their uncertainty scores are diffident. For example, higher quality score implies a much smaller uncertainty for positive texts ($D_{+,h}$ and $D_{+,l}$ in Figure \ref{fig:FAMEShift:casestudy}(f)). 
Similarly, negative texts seem to contain more uncertainty than positive ones ($D_{+,h}$ and $D_{-,h}$ in Figure \ref{fig:FAMEShift:casestudy}(f)).
As shown in Figure \ref{fig:FAMEShift:casestudy} (g), \systemnameAPIShift{} indeed learns to utilize such imbalanced uncertainty: its sampling allocation for each partition is quite close to the  optimal allocation (dark star point). %
The sampling ratio among different partitions scales roughly linear as the total number of samples increases. 
Note that the optimal sample allocation also scales linearly with respect to the total number of samples. 
Thus, \systemnameAPIShift{} can approximate the optimal allocation for any sample numbers (unless it is too small).
Finally, it is worth noting that \systemnameAPIShift{} outperforms standard uniform sampling notably, 
as shown in Figure \ref{fig:FAMEShift:casestudy}(h). %
This is because uniform sampling does not exploit the uncertainty of each partition. %

\paragraph{Budget savings achieved by \systemnameAPIShift{}.}
In many applications, it suffices to obtain an estimated API shift close to the true shift, e.g., 
within  a 1\% Frobenius norm error.  
Thus, a natural question arises: \textit{to reach the same tolerable estimation error, how much sampling cost can \systemnameAPIShift{} reduce compared to standard sampling approaches?}

To answer this question, we compare \systemnameAPIShift{} with two natural approaches:  (i) uniform sampling and  (ii) stratified sampling by drawing same  number of samples for each true label. 
For each approach, we measure the number of samples needed to reach 1\% Frobenius norm error with probability 95\%, via an upper bound on the estimated Frobenius error. 
The details are left to Appendix \ref{sec:FAMEShift:experimentdetails}.  
As shown in Table \ref{tab:FAMEShift:costsaving}, \systemnameAPIShift{} usually requires more than 70\% fewer samples to reach such tolerable Frobenius norm error than the uniform and stratified sampling.  
In fact, the sample size reduced by \systemnameAPIShift{} can be as high as almost 94\%. 
This is primarily because \systemnameAPIShift{}'s shift estimation is more accurate.
Uniform and stratified sampling required the similar number of samples because the upper bounds on their estimated Frobenius error are similar.   
This demonstrates that  \systemnameAPIShift{} can significantly reduce the required sample size and thus the query cost to assess the ML API shifts in practice.

\paragraph{Trade-offs between estimation error and query budget .}
Next we examine  the trade-offs between API shift estimation error and sample size achieved by \systemnameAPIShift{}, shown in Figure \ref{fig:FAMEShift:tradeoffs}. 
We first note that, across all 12 observed API shifts, \systemnameAPIShift{} consistently outperforms standard uniform sampling for any fixed sample size. In fact, the achieved estimation error of \systemnameAPIShift{} is usually an order of magnitude smaller than that of uniform sampling.
This verifies  that \systemnameAPIShift{} can provide more accurate assessments of API shifts  in diverse applications.
Second, some API shifts are easier to estimate than others.
For example, for Google API shift on AMNIST, using 1000 samples already gives an expected squared Frobenius norm error lower than $10^{-4}$, while it usually requires 2000 samples for other shifts. 
This is probably because the skew in its uncertainties among different partitions is more severe than other shifts. 
Another observation is that, the relative error gap between \systemnameAPIShift{} and uniform sampling remains stable.
This is because \systemnameAPIShift{} becomes closer to the optimal sampling, whose estimation error is indeed a constant fraction of that of uniform sampling.

\section{Conclusion}\label{Sec:FAMEShift:Conclusion}
In this paper, we identify and formulate the problem of characterizing ML API shifts. Our systematic empirical study shows that API model updates are frequent, and that some updates can reduce performance substantially. 
Quantifying such shifts is an important but understudied problem that can greatly affect the reliability of applications using ML-as-a-service. 
To assess API shifts, we propose an algorithmic framework,  \systemnameAPIShift{}, which provides significant estimation error and sample size reduction both theoretically and empirically. 
Our work focuses on estimating changes in the confusion matrix because the confusion matrix is often what is used by practitioners to assess API performance. We acknowledge that confusion matrices are most applicable for classification tasks, and other measures need to be used for more complex APIs (e.g. OCR, NLP). While this is a limitation, classification with a small to moderate number of classes of interest is a common use case for ML APIs, and this is an important starting point since it has not been studied before.  

\newpage
{

}
\newpage
\appendix
\paragraph{Outline} The supplement is organized as follows. 
In Section \ref{sec:FAMEShift:Limit}, we provide additional discussions about potential limits and societal impacts of this work. 
Section \ref{sec:FAMEShift:techdetails} provides additional technical details. 
All proofs are presented in Section \ref{sec:FAMEShift:proof}. 
We give experimental setups, details of datasets and ML APIs,  and further empirical results in Section \ref{sec:FAMEShift:experimentdetails}.

\section{Additional Discussion of Limitation and Societal Impacts}\label{sec:FAMEShift:Limit}
In this paper we focus on estimating changes in confusion matrix. Confusion matrices are typically applicable for classification tasks, and other measures may be needed for more complex APIs (for example, OCR or NLP). While this is a limitation and a great direction of future work, we believe that classification tasks are an important starting point for analyzing API shifts since they are the most commonly used type of ML APIs. 
In this paper, we used data distributions that are constant over time in order to clearly isolate the changes in API performance due to model changes. This is also reasonable because in many applications the data distribution changes very slowly. It is another interesting future work to investigate settings where  both the data distribution and ML APIs have change over time.

As the ML as a service industry has attracted more users, it can have significant societal impacts by making ML more accessible.
In this paper, we study how the ML APIs change over time, which can impact the way users understand the performance of their applications empowered by those APIs. 
Furthermore, this also helps users determine which API to use. Overall our approach for monitoring API shifts contribute to making ML-as-a-service more reliable in practice. 
One increasingly important concern is the biases towards minority groups in the ML APIs.
Understanding the API shifts gives one way to examine if and by how much the biases are mitigated or magnified over time.
We will release the code and all of our datasets to stimulate more research to understand API shifts (the website is hidden for now to ensure author anonymity). 

\section{Technical Details}\label{sec:FAMEShift:techdetails}

\paragraph{Computation and space cost of \systemnameAPIShift{}.} One attractive property of \systemnameAPIShift{} is its low computation and space cost. In fact, it can be easily verified from Algorithm \ref{Alg:FAMEShift:MainAlg} that, the computation cost is only linear in the number of samples $N$.
The occupied space is only constant.  
Therefore, \systemnameAPIShift{} can be easily applied for large number of samples. 

\paragraph{Choice of parameter $a$.} The parameter $a$ is used to balance between exploiting and exploration in Algorithm \ref{Alg:FAMEShift:MainAlg}.
Throughout this paper, we set $a=1$ as the default value. 
While theoretically $a$ should depend on the partition size and sample number, in practice we found that $a=1$ works well. 
An in-depth analysis for this remains an interesting open problem.

\paragraph{Stopping rule under loss requirements.} For \systemnameAPIShift{}, we establish the upper bound on the loss by (i) computing the upper bound on the estimated uncertainty score for each partition,  and (ii) summing up all those upper bounds weighted by the partition size to form the upper bound on the loss.
For Uniform sampling or stratified sampling, we directly use the upper bound on the Frobenius loss. 
Here, we adopt the standard upper bound for Bernoulli variables.
That is to say, for any estimator using $n$ samples, we use $\sqrt{\frac{c}{n}}$ as its upper bound, where $c$ is a parameter to control the confidence. For both methods, we choose $c$ to ensure a 1\% error under 95\% confidence level.

\section{Proofs}\label{sec:FAMEShift:proof}
We present all missing proofs here. 
For ease of expositions, let us first introduce a few notations.
We let $x_n$ denote the $n$th sample drawn in Algorithm \ref{Alg:FAMEShift:MainAlg}, and use $I_n$ to indicate from which partition the sample $x_n$ is drawn.
For example, $I_n=(i,k)$ indicates that $x_n$ is drawn from the partition $D_{i,k}$.

Let $\sample_{\ell,k,t} \in [L]$ denote the ML API's predicted label for the $t$th sample drawn from the partition $D_{\ell,k}$.
Abusing the notation a little bit, let $\SN_{\ell, k,n}$ denote the value of $\SN_{\ell, k}$ after the $n-1$th iteration and before the $n$th iteration in Algorithm \ref{Alg:FAMEShift:MainAlg}.
Similarly, let $\hatusc_{\ell,k,n}$ be the value of $\hatusc_{\ell,k}$,  $\Smu_{\ell,k,j,n}$ be the value of $\Smu_{\ell,k,j}$, and $\hash_{\ell,k,,j,n}$ be the value of $\hash_{\ell,k,j}$, all  after the $n-1$th iteration and before the $n$th iteration in Algorithm \ref{Alg:FAMEShift:MainAlg}.
In addition, let $\Delta_{\ell,k} \triangleq \frac{\pmb p_{\ell,k} \usc_{\ell,k}}{\sum_{\ell',k'}^{}\pmb p_{\ell',k'} \usc_{\ell',k'}}$ and $\Delta_{\min} = \min \Delta_{\ell,k}$.
Similarly, let us denote $\usc_{\min} \triangleq \min \usc_{\ell,k}$.
By assumption that $\pmb p_{\ell,k}> 0$ and $\usc_{\ell,k}>0$, we must have $\Delta_{\min}>0$ and $\usc_{\min}>0$.

\subsection{Useful Lemmas}
Let us first give a few useful lemmas. The first gives a high probability bound on our estimated uncertainty score.
\begin{lemma}\label{lemma:FAMEShift:highprob}
Let the event $A$ be  
\begin{equation*}
A \triangleq \mathop{\bigcap_{1\leq k\leq K, 1\leq \ell\leq L}}_{1\leq t\leq N}\left\lbrace  \left| \sqrt{ 1- \frac{1}{t(t-1)} \sum_{i=1}^{t} \sum_{j=1,j\not=i}^{t} \mathbbm{1}_{\sample_{\ell,k,i}=\sample_{\ell,k,j}}}-\usc_{\ell,k} \right| \leq \sqrt[4]{            \frac{\log 2/\delta}{2t}}  \right\rbrace
\end{equation*}
Then for any $\delta>0$, we have $\Pr[A] \geq 1- LKN\delta$.
\end{lemma}
\begin{proof}
For any fixed $t$, let us first denote
\begin{equation*}
\begin{split}
     f(\sample_{\ell,k,1},\sample_{\ell,k,2},\cdots, \sample_{\ell,k,t}) \triangleq 1- \frac{1}{t(t-1)} \sum_{i=1}^{t} \sum_{j=1,j\not=i}^{t} \mathbbm{1}_{\sample_{\ell,k,i}=\sample_{\ell,k,j}}
\end{split}
\end{equation*}
Its expectation is simply 
\begin{equation*}
\begin{split}
     \Exp[f(\sample_{\ell,k,1},\sample_{\ell,k,2},\cdots, \sample_{\ell,k,t})] = & \Exp[1- \frac{1}{t(t-1)} \sum_{i=1}^{t} \sum_{j=1,j\not=i}^{t} \mathbbm{1}_{\sample_{\ell,k,i}=\sample_{\ell,k,j}}]\\
      =  &1 - \frac{1}{t(t-1)}\Exp[\sum_{i=1}^{t} \sum_{j=1,j\not=i}^{t} \mathbbm{1}_{\sample_{\ell,k,i}=\sample_{\ell,k,j}}]\\
      = & 1 - \Exp[ \mathbbm{1}_{\sample_{\ell,k,i}=\sample_{\ell,k,j}}]\\
\end{split}
\end{equation*}
where the second equation applies the linearity of expectation, and the third equation uses the fact that all $\sample_{\ell,k, i}$ are identically independent.
Note that 
\begin{equation*}
\begin{split}
 & \Exp[ \mathbbm{1}_{\sample_{\ell,k,i}=\sample_{\ell,k,j}}]\\
 =& \Pr[\sample_{\ell,k,i}=\sample_{\ell,k,j}] \\
 = & \sum_{r=1}^{L}\Pr[\sample_{\ell,k,i}=r]\Pr[\sample_{\ell,k,j}=r]\\
 = & \sum_{r=1}^{L}\Pr^2[\sample_{\ell,k,i}=r]
\end{split}
\end{equation*}
where the first equation uses the definition of indicator function, the second uses the fact that two sample are independent and there are only $L$ many possible labels, and the last equation uses the fact that those samples' distribution is identical.  
Applying this in the above equation, we get 
\begin{equation*}
\begin{split}
     \Exp[f(\sample_{\ell,k,1},\sample_{\ell,k,2},\cdots, \sample_{\ell,k,t})] 
      = & 1 - \Exp[ \mathbbm{1}_{\sample_{\ell,k,i}=\sample_{\ell,k,j}}]\\
= & 1 - \sum_{r=1}^{L}\Pr^2[\sample_{\ell,k,i}=r] = \usc_{\ell,k}^2
\end{split}
\end{equation*}
That is to say, its expectation is simply the uncertainty score $\usc_{\ell,k}^2$.
On the other hand, we note that, for any $i$, we have 
\begin{equation*}
\begin{split}
     & f(\sample_{\ell,k,1},\cdots, \sample_{\ell,k,i-1},\sample_{\ell,k,i},\sample_{\ell,k,i+1},\cdots, \sample_{\ell,k,t}) -      f(\sample_{\ell,k,1},\cdots, \sample_{\ell,k,i-1}, \sample_{\ell,k,i}',\sample_{\ell,k,i+1},\cdots, \sample_{\ell,k,t})\\
=& \frac{1}{t(t-1)} \sum_{j=1,j\not=i}^{t}\mathbbm{1}_{\sample_{\ell,k,i}'=\sample_{\ell,k,j}} - \mathbbm{1}_{\sample_{\ell,k,i}=\sample_{\ell,k,j}} \leq \frac{1}{t(t-1)}\cdot (t-1) = \frac{1}{t}
\end{split}
\end{equation*}
where the inequality is due to the fact that the indicator function can only take values in $\{0,1\}$.
Similarly, we have 

\begin{equation*}
\begin{split}
     & f(\sample_{\ell,k,1},\cdots, \sample_{\ell,k,i-1},\sample_{\ell,k,i},\sample_{\ell,k,i+1},\cdots, \sample_{\ell,k,t}) -      f(\sample_{\ell,k,1},\cdots, \sample_{\ell,k,i-1}, \sample_{\ell,k,i}',\sample_{\ell,k,i+1},\cdots, \sample_{\ell,k,t})\\
=& \frac{1}{t(t-1)} \sum_{j=1,j\not=i}^{t}\mathbbm{1}_{\sample_{\ell,k,i}'=\sample_{\ell,k,j}} - \mathbbm{1}_{\sample_{\ell,k,i}=\sample_{\ell,k,j}} \geq \frac{1}{t(t-1)}\cdot -(t-1) = -\frac{1}{t}
\end{split}
\end{equation*}
By Mcdiarmid inequality, we have
\begin{equation*}
\begin{split}
    & \Pr[|f(\sample_{\ell,k,1},\sample_{\ell,k,2},\cdots, \sample_{\ell,k,t})- \Exp[f(\sample_{\ell,k,1},\sample_{\ell,k,2},\cdots, \sample_{\ell,k,t})] |\geq \epsilon]
\leq  2 e^{-\frac{2\epsilon^2}{\sum_{i=1}^{t} t^{-2}}} = 2 e^{-2t\epsilon^2} 
\end{split}
\end{equation*}
Set $\delta = 2 e^{-2t\epsilon^2}$. This simply becomes, with probability at most $\delta$, 
\begin{equation*}
\begin{split}
    & |f(\sample_{\ell,k,1},\sample_{\ell,k,2},\cdots, \sample_{\ell,k,t})- \usc^2_{\ell,k} | \\
    = &|f(\sample_{\ell,k,1},\sample_{\ell,k,2},\cdots, \sample_{\ell,k,t})- \Exp[f(\sample_{\ell,k,1},\sample_{\ell,k,2},\cdots, \sample_{\ell,k,t})] | \geq \sqrt{\frac{\log 2/\delta}{2t}}
\end{split}
\end{equation*}
Note that $f$ is positive, we can take square root of both side, and obtain with probability at most $\delta$, 
\begin{equation*}
\begin{split}
    & |\sqrt{f(\sample_{\ell,k,1},\sample_{\ell,k,2},\cdots, \sample_{\ell,k,t})}- \usc_{\ell,k} |  \geq \sqrt[4]{\frac{\log 2/\delta}{2t}}
\end{split}
\end{equation*}
Or alternatively, with probability at least $1-\delta$, \begin{equation*}
\begin{split}
    & |\sqrt{f(\sample_{\ell,k,1},\sample_{\ell,k,2},\cdots, \sample_{\ell,k,t})}- \usc_{\ell,k} |  \leq \sqrt[4]{\frac{\log 2/\delta}{2t}}
\end{split}
\end{equation*}
which holds for fixed $t,\ell,k$.
Taking union bound, we know that with probability $1-KLN\delta$,
\begin{equation*}
\begin{split}
    & |\sqrt{f(\sample_{\ell,k,1},\sample_{\ell,k,2},\cdots, \sample_{\ell,k,t})}- \usc_{\ell,k} |  \leq \sqrt[4]{\frac{\log 2/\delta}{2t}}
\end{split}
\end{equation*}
which holds for all $t,\ell,k$.
Plugging in the form of $f$ completes the proof.
\end{proof}
The next one is more technical: it gives a connection between stopping time and adaptive sampling.
We omit the proof and refer the interested readers to \cite{measureandprob}.
\begin{lemma}[Wald's second inequality]\label{lemma:FAMEShift:wald-inequality}
Let $\{\mathcal F_t\}_{t=1,\ldots,n}$ be a filtration and $\{X_t\}_{t=1,\ldots,n}$ be an $\mathcal F_t$ adapted sequence of i.i.d.~random variables with finite expectation $\mu$ and variance $Var$. Assume that $\mathcal F_t$ and $\sigma(\{X_s:s \geq t+1\})$ are independent for any $t \leq n$, and let $T(\leq n)$ be a stopping time with respect to $\mathcal F_t$. Then
%
\begin{equation*}\label{eq:FAMEShift:wald-inequality}
\E\Bigg[\Big(\sum_{i=1}^T X_i - T\;\mu\Big)^2\Bigg] = \E[T]\;\var.
\end{equation*}
\end{lemma}

\subsection{Proof of Lemma \ref{lemma:FAMEShift:optimalsampling}}
\begin{proof}
Recall that the loss, defined as the expected squared Frobenius norm error, is 
\begin{equation*}
\begin{split}
\Exp\left[\| \deltaCM - \hatdeltaCM \|_F^2\right] = &\sum_{i,j} \Exp\left(\deltaCM_{i,j} - \hatdeltaCM_{i,j} \right)^2 = \sum_{i,j} \Exp\left(\sum_{k}^{} \pmb p_{i,k}[\truemu_{i,k,j}-\Smu_{i,k,j}] \right)^2  \\= &  \sum_{i,j,k} \pmb p_{i,k}^2 \Exp\left([\truemu_{i,k,j}-\Smu_{i,k,j}] \right)^2    
\end{split}
\end{equation*}
Here we basically apply the definition of each entry. 
Suppose $\SN_{i,k}$ samples are allocated to estimate $\truemu_{i,k,j}$. 
Then we have
\begin{equation*}
\begin{split}
\Exp\left([\truemu_{i,k,j}-\Smu_{i,k,j}] \right)^2  = \frac{1}{\SN_{i,k}} \Pr[\hat{y}(x)=j|x\in D_{i,k}](1-\Pr[\hat{y}(x)=j|x\in D_{i,k}])
\end{split}
\end{equation*}
since $\truemu_{i,k,j}$ is effectively a Bernoulli variable. Then the loss becomes

\begin{equation*}
\begin{split}
\Exp\left[\| \deltaCM - \hatdeltaCM \|_F^2\right] = &  \sum_{i,j,k} \pmb p_{i,k}^2 \Exp\left([\truemu_{i,k,j}-\Smu_{i,k,j}] \right)^2 \\
= & \sum_{i,j,k} \pmb p_{i,k}^2 \frac{1}{\SN_{i,k}} \Pr[\hat{y}(x)=j|x\in D_{i,k}](1-\Pr[\hat{y}(x)=j|x\in D_{i,k}])\\
= & \sum_{i,k} \pmb p_{i,k}^2 \frac{1}{\SN_{i,k}} \sum_{j} \Pr[\hat{y}(x)=j|x\in D_{i,k}](1-\Pr[\hat{y}(x)=j|x\in D_{i,k}])\\
\end{split}
\end{equation*}
where the last equation is simply by rearranging the summation. 
Note that 
\begin{equation*}
\sum_{j} \Pr[\hat{y}(x)=j|x\in D_{i,k}]=1    
\end{equation*}
The last summation is simply 
\begin{equation*}
\begin{split}
\sum_{j} \Pr[\hat{y}(x)=j|x\in D_{i,k}](1-\Pr[\hat{y}(x)=& j|x\in D_{i,k}])= 1 - \sum_{j} \Pr^2[\hat{y}(x)=j|x\in D_{i,k}])\\
= & \usc_{i,k}^2
\end{split}
\end{equation*}
Thus, the loss becomes 
\begin{equation*}
\begin{split}
\Exp\left[\| \deltaCM - \hatdeltaCM \|_F^2\right] 
= & \sum_{i,k} \pmb p_{i,k}^2 \usc_{i,k}^2 \frac{1}{\SN_{i,k}}
\end{split}
\end{equation*}
By Cauchy Schwarz inequality, we have 
\begin{equation*}
\begin{split}
& \left(\sum_{i,k} \frac{\pmb p_{i,k}^2 \usc_{i,k}^2}{\SN_{i,k}} \right) \left(\sum_{i,k}^{} \SN_{i,k}\right) \geq \left(\sum_{i,k} \pmb p_{i,k} \usc_{i,k} \right)^2
\end{split}
\end{equation*}
where the equality holds if and only if 
\begin{equation*}
    \frac{\pmb p_{i,k}^2 \usc_{i,k}^2}{\SN_{i,k}^2} =     \frac{\pmb p_{i',k'}^2 \usc_{i',k'}^2}{\SN_{i',k'}^2}
\end{equation*}
for any $i,i',k,k'$.
That is to say, there exists some constant $c$, such that
\begin{equation*}
    \frac{\pmb p_{i,k} \usc_{i,k}}{\SN_{i,k}} =     \frac{\pmb p_{i',k'} \usc_{i',k'}}{\SN_{i',k'}} = \frac{1}{c}
\end{equation*}
And thus, $\SN_{i,k} = \pmb p_{i,k} \usc_{i,k} c$. Summing over $i,k$ gives 
\begin{equation*}
    N = \sum_{i,k}\SN_{i,k} = \sum_{i,k}^{}\pmb p_{i,k} \usc_{i,k} c
\end{equation*}
Thus, \begin{equation*}
    c= \frac{N}{\sum_{i,k}^{}\pmb p_{i,k} \usc_{i,k}}
\end{equation*}
and 
\begin{equation*}
    \SN_{i,k} = \pmb p_{i,k} \usc_{i,k} c = \pmb p_{i,k} \usc_{i,k} \cdot \frac{1}{ \sum_{i,k}^{} \pmb p_{i,k} \usc_{i,k}}  = \frac{\pmb p_{i,k} \usc_{i,k} }{ \sum_{i,k}^{} \pmb p_{i,k} \usc_{i,k}}
\end{equation*}
which completes the proof.
\end{proof}

\subsection{Proof of Theorem \ref{thm:FAMEAPIShift:mainbound}}
\begin{proof}
To prove this theorem, we need a few more lemmas.
\begin{lemma}\label{lemma:FAMEShift:complexityandrule}
Algorithm \ref{Alg:FAMEShift:MainAlg}'s computational cost is $O(LKN)$ and space cost is $O(L^2K)$.
Furthermore, for any $n>2LK$, after the $n-1$th iteration and before the $n$th iteration, we have 
\begin{equation*}
    \SN_{\ell,k} = \SN_{\ell,k,n} = \sum_{i=1}^{n-1} \mathbbm{1}_{I_i=(\ell,k)}
\end{equation*}

\begin{equation*}
    \Smu_{\ell,k,j} = \Smu_{\ell,k,j,n} = \frac{1}{\SN_{\ell,k,n}} \sum_{i=1}^{n-1} \mathbbm{1}_{I_i=(\ell,k)} \mathbbm{1}_{\hat{y}(x_i)=j}
\end{equation*}

\begin{equation*}
    \hatusc_{\ell,k} = \hatusc_{\ell,k,n}=  1 - \frac{1}{\SN_{\ell,k,n}(\SN_{\ell,k,n}-1)} \sum_{i=1}^{n-1}\sum_{j=1,j\not=i}^{n-1} \mathbbm{1}_{I_i=I_j=(\ell,k)} \mathbbm{1}_{\hat{y}({x_i})=\hat{y}({x_j})}
\end{equation*}

\begin{equation*}
    \hash_{\ell,k,j} =  \hash_{\ell,k,j,n} =  \sum_{i=1}^{n-1} \mathbbm{1}_{I_i=(\ell,k)} \mathbbm{1}_{\hat{y}(x_i)=j}
\end{equation*}
\end{lemma}
\begin{proof}
The computational and space cost can be easily verified: 
as shown in Algorithm \ref{Alg:FAMEShift:MainAlg}, the variables $\hatusc, \Smu, \hash, \SN$ take space $LK, L2K, L^2K, LK$.
Therefore, the space is bounded by $O(L^2K)$. For the first $2LK$ iterations (line 3-8) in Algorithm \ref{Alg:FAMEShift:MainAlg}, the computation cost is clearly $O(LK)$.
For the rest iterations (line 10- 16), the most expensive cost is computing $I_n$, which requires $LK$ computations per iteration.
Therefore, the total computational cost is $O(LKN)$.

Next we show that the above four equations hold for every $n>2LK$.
We prove this by induction.

1) $n=2LK+1$: One can easily verify this by plugging the initial values established in line 3-8 in Algorithm \ref{Alg:FAMEShift:MainAlg}.

2) Suppose the four equations hold for the case when $n=m$. Now consider $n=m+1$.
Now let us consider two cases.

\begin{itemize}
    \item Any $\ell,k$ such that  $I_{m+1}\not=(\ell,k)$: There is nothing update, 
\begin{equation*}
\begin{split}
    \SN_{\ell,k,m+1} & = \SN_{\ell,k,m} =  \sum_{i=1}^{m-1} \mathbbm{1}_{I_i=(\ell,k)} = \sum_{i=1}^{m-1} \mathbbm{1}_{I_i=(\ell,k)}  + 0 = \sum_{i=1}^{m-1} \mathbbm{1}_{I_i=(\ell,k)}+\mathbbm{1}_{I_{m}=(\ell,k)} \\
    & = \sum_{i=1}^{m} \mathbbm{1}_{I_i=(\ell,k)}
\end{split}    
\end{equation*}
Similarly, one can show that

\begin{equation*}
  \Smu_{\ell,k,j,m+1} = \Smu_{\ell,k,j,m} =  \frac{1}{\SN_{\ell,k,m}} \sum_{i=1}^{m} \mathbbm{1}_{I_i=(\ell,k)} \mathbbm{1}_{\hat{y}(x_i)=j}
\end{equation*}

\begin{equation*}
    \hatusc_{\ell,k,m+1} = \hatusc_{\ell,k,m}=  1 - \frac{1}{\SN_{\ell,k,m}(\SN_{\ell,k,m}-1)} \sum_{i=1}^{m}\sum_{j=1,j\not=i}^{m} \mathbbm{1}_{I_i=I_j=(\ell,k)} \mathbbm{1}_{\hat{y}(x_i)=\hat{y}(x_j)}
\end{equation*}

\begin{equation*}
    \hash_{\ell,k,j,m+1} =  \hash_{\ell,k,j,m} =  \sum_{i=1}^{m} \mathbbm{1}_{I_i=(\ell,k)} \mathbbm{1}_{\hat{y}(x_i)=j}
\end{equation*}

    \item  For some $\ell^*,k^*$ such that  $I_{m+1}=(\ell^*,k^*)$.
    
    Let us first consider $\SN_{\ell^*,k^*}$. 
    We increment $\SN_{\ell^*,k^*}$ by one, and thus
\begin{equation*}
\begin{split}
    \SN_{\ell^*,k^*,m+1} & = \SN_{\ell^*,k^*,m} + 1 =  \sum_{i=1}^{m-1} \mathbbm{1}_{I_i=(\ell^*,^*k)} + 1 = \sum_{i=1}^{m-1} \mathbbm{1}_{I_i=(\ell^*,k^*)}+\mathbbm{1}_{I_{m}=(\ell^*,k^*)} \\
    & = \sum_{i=1}^{m} \mathbbm{1}_{I_i=(\ell^*,k^*)}
\end{split}    
\end{equation*}

Next we consider $\Smu_{\ell^*,k^*,j}$. Using a similar argument as above, we have  

\begin{equation*}
\begin{split}
\Smu_{\ell^*,k^*,j,m+1} = & \Smu_{\ell^*,k^*,j,m} + \frac{\mathbbm{1}_{\hat{y}(x_{m})=j}-\Smu_{\ell^*,k^*,j,m}}{\SN_{\ell^*,k^*,m+1}}, \\
=& \frac{\SN_{\ell^*,k^*,m+1}-1}{\SN_{\ell^*,k^*,m+1}}\Smu_{\ell^*,k^*,j,m} + \frac{\mathbbm{1}_{\hat{y}(x_m)=j}}{\SN_{\ell^*,k^*,m+1}}, \\
=& \frac{\SN_{\ell^*,k^*,m}}{\SN_{\ell^*,k^*,m+1}}\Smu_{\ell^*,k^*,j,m} + \frac{\mathbbm{1}_{\hat{y}(x_m)=j}}{\SN_{\ell^*,k^*,m+1}}, \\
=& \frac{\SN_{\ell^*,k^*,m}}{\SN_{\ell^*,k^*,m+1}} \frac{1}{\SN_{\ell^*,k^*,m}} \sum_{i=1}^{m-1} \mathbbm{1}_{I_i=(\ell^*,k^*)} \mathbbm{1}_{\hat{y}(x_i)=j} +\frac{\mathbbm{1}_{\hat{y}(x_m)=j}}{\SN_{\ell^*,k^*,m+1}} \\
=& \frac{1}{\SN_{\ell^*,k^*,m+1}} \sum_{i=1}^{m} \mathbbm{1}_{I_i=(\ell^*,k^*)} \mathbbm{1}_{\hat{y}(x_i)=j} \end{split}    
\end{equation*}
where the first equation is due to the update rule of 
$\Smu_{\ell^*,k^*,j}$, the second equation is simply grouping by $\Smu_{\ell^*,k^*,j,m}$, the third equation is due to the fact that $ \SN_{\ell^*,k^*,m+1}  = \SN_{\ell^*,k^*,m} + 1 $, the forth equation is due to the induction assumption, and the forth equation is simply algebraic rewriting. 

Now let us consider $\hatusc_{\ell^*,k^*,m+1}^2$. We can write 
\begin{equation*}
    \begin{split}
        & \hatusc_{\ell^*,k^*,m+1}^2\\
        =         &\hatusc_{\ell^*,k^*,m}^2 + \frac{2}{\SN_{\ell^*,k^*,m+1}} (1- \frac{\hash_{\ell^*,k^*,\hat{y}(x_m),m}}{\SN_{\ell^*,k^*,m+1}-1} - \hatusc^2_{\ell^*,k^*,m})\\
        = & \frac{\SN_{\ell^*,k^*,m+1}-2}{\SN_{\ell^*,k^*,m+1}} \hatusc_{\ell^*,k^*,m}^2 + \frac{2}{\SN_{\ell^*,k^*,m+1}} (1- \frac{\hash_{\ell^*,k^*,\hat{y}(x_m),m}}{\SN_{\ell^*,k^*,m+1}-1})\\
        = &  \frac{\SN_{\ell^*,k^*,m+1}-2}{\SN_{\ell^*,k^*,m+1}} (1 - \frac{1}{\SN_{\ell^*,k^*,m}(\SN_{\ell^*,k^*,m}-1)} \sum_{i=1}^{m-1}\sum_{j=1,j\not=i}^{m-1} \mathbbm{1}_{I_i=I_j=(\ell^*,k^*)} \mathbbm{1}_{\hat{y}(x_i)=\hat{y}(x_j)})\\
       + &  \frac{2}{\SN_{\ell^*,k^*,m+1}} (1- \frac{\hash_{\ell^*,k^*,\hat{y}(x_m),m}}{\SN_{\ell^*,k^*,m+1}-1}) \\
    = &  \frac{\SN_{\ell^*,k^*,m+1}-2}{\SN_{\ell^*,k^*,m+1}} (1 - \frac{1}{(\SN_{\ell^*,k^*,m+1}-2)(\SN_{\ell^*,k^*,m+1}-1)} \sum_{i=1}^{m-1}\sum_{j=1,j\not=i}^{m-1} \mathbbm{1}_{I_i=I_j=(\ell^*,k^*)} \mathbbm{1}_{\hat{y}(x_i)=\hat{y}(x_j)})\\
       + &  \frac{2}{\SN_{\ell^*,k^*,m+1}} (1- \frac{\hash_{\ell^*,k^*,\hat{y}(x_m),m}}{\SN_{\ell^*,k^*,m+1}-1}) \\
       = & 1 -  \frac{1}{\SN_{\ell^*,k^*,m+1}(\SN_{\ell^*,k^*,m+1}-1)} \sum_{i=1}^{m-1}\sum_{j=1,j\not=i}^{m-1} \mathbbm{1}_{I_i=I_j=(\ell^*,k^*)} \mathbbm{1}_{\hat{y}(x_i)=\hat{y}(x_j)} \\
       -& \frac{2\hash_{\ell^*,k^*,\hat{y}(x_m),m}}{\SN_{\ell^*,k^*,m+1}(\SN_{\ell^*,k^*,m+1}-1)}
    \end{split}
\end{equation*}
where the first equation is by the update rule in Algorithm \ref{Alg:FAMEShift:MainAlg}, the second equation is simply rearranging the terms, the third equation uses the induction assumption, the forth one uses the update rule on $\SN_{\ell^*,k^*}$ and thus $\SN_{\ell^*,k^*,m} = \SN_{\ell^*,k^*,m+1}-1$, and the fifth equation is also rearranging the terms. 

On the other hand, by induction assumption, we have 
\begin{equation*}
     \hash_{\ell^*,k^*,\hat{y}(x_m),m} =  \sum_{i=1}^{m-1} \mathbbm{1}_{I_i=(\ell^*,k^*)} \mathbbm{1}_{\hat{y}(x_i)=\hat{y}(x_m)}
\end{equation*}
And thus 
\begin{equation*}
\begin{split}
     &\sum_{i=1}^{m-1}\sum_{j=1,j\not=i}^{m-1} \mathbbm{1}_{I_i=I_j=(\ell^*,k^*)} \mathbbm{1}_{\hat{y}(x_i)=\hat{y}(x_j)}  + 2 \hash_{\ell^*,k^*,\hat{y}(x_m),m}\\
     = & \sum_{i=1}^{m} \sum_{j=1,j\not=i}^{m} \mathbbm{1}_{I_i=(\ell^*,k^*)} \mathbbm{1}_{\hat{y}(x_i)=\hat{y}(x_m)}    
\end{split}
\end{equation*}
Hence, the above equation becomes 

\begin{equation*}
    \begin{split}
        & \hatusc_{\ell^*,k^*,m+1}^2
       = 1 -  \frac{1}{\SN_{\ell^*,k^*,m+1}(\SN_{\ell^*,k^*,m+1}-1)} \sum_{i=1}^{m}\sum_{j=1,j\not=i}^{m} \mathbbm{1}_{I_i=I_j=(\ell^*,k^*)} \mathbbm{1}_{\hat{y}(x_i)=\hat{y}(x_j)} \\
    \end{split}
\end{equation*}
Finally, let us consider $\hash_{\ell^*,k^*,j}$. 
If $j\not=\hat{y}(x_m)$, it is clear that 
\begin{equation*}
    \begin{split}
        \hash_{\ell^*,k^*,j,m+1} =         &\hash_{\ell^*,k^*,j,m} = \sum_{i=1}^{m-1} \mathbbm{1}_{I_i=(\ell^*,k^*)} \mathbbm{1}_{\hat{y}(x_i)=j} + 0 \\
        = &  \sum_{i=1}^{m-1} \mathbbm{1}_{I_i=(\ell^*,k^*)} \mathbbm{1}_{\hat{y}(x_i)=j} + \mathbbm{1}_{I_i=(\ell^*,k^*)}\mathbbm{1}_{\hat{y}(x_m)=j}\\
        = &\sum_{i=1}^{m} \mathbbm{1}_{I_i=(\ell^*,k^*)} \mathbbm{1}_{\hat{y}(x_i)=j} 
    \end{split}
\end{equation*}
where the first is due to that there is no update for this $j$, the third equation is due to the fact that $\hat{y}(x_m)\not=j$, and all the other equations are algebraic rewriting.

If $j=\hat{y}(x_m)$, it is clear that 
\begin{equation*}
    \begin{split}
        \hash_{\ell^*,k^*,j,m+1} =         &\hash_{\ell^*,k^*,j,m} + 1 = \sum_{i=1}^{m-1} \mathbbm{1}_{I_i=(\ell^*,k^*)} \mathbbm{1}_{\hat{y}(x_i)=j} + 1 \\
        = &  \sum_{i=1}^{m-1} \mathbbm{1}_{I_i=(\ell^*,k^*)} \mathbbm{1}_{\hat{y}(x_i)=j} + \mathbbm{1}_{I_i=(\ell^*,k^*)}\mathbbm{1}_{\hat{y}(x_m)=j}\\
        = &\sum_{i=1}^{m} \mathbbm{1}_{I_i=(\ell^*,k^*)} \mathbbm{1}_{\hat{y}(x_i)=j} 
    \end{split}
\end{equation*}
where the first is due to that there is no update for this $j$, the third equation is due to the fact that $\hat{y}(x_m)=j$, and all the other equations are algebraic rewriting.

\end{itemize}

That is to say, we have shown that, 
\begin{equation*}
\begin{split}
    \SN_{\ell,k,m+1} & =  \sum_{i=1}^{m} \mathbbm{1}_{I_i=(\ell,k)}
\end{split}    
\end{equation*}

\begin{equation*}
   \Smu_{\ell,k,j,m+1} = \frac{1}{\SN_{\ell,k,m+1}} \sum_{i=1}^{m} \mathbbm{1}_{I_i=(\ell,k)} \mathbbm{1}_{\hat{y}(x_i)=j}
\end{equation*}

\begin{equation*}
\hatusc_{\ell,k,m+1}=  1 - \frac{1}{\SN_{\ell,k,m+1}(\SN_{\ell,k,m+1}-1)} \sum_{i=1}^{m}\sum_{j=1,j\not=i}^{m} \mathbbm{1}_{I_i=I_j=(\ell,k)} \mathbbm{1}_{\hat{y}({x_i})=\hat{y}({x_j})}
\end{equation*}

\begin{equation*} \hash_{\ell,k,j,m+1} =  \sum_{i=1}^{m} \mathbbm{1}_{I_i=(\ell,k)} \mathbbm{1}_{\hat{y}(x_i)=j}
\end{equation*}
always hold.
By induction, we can say that for any $n>2LK$, the original equations hold, which completes the proof.
\end{proof}
\begin{lemma}\label{lemma:FAMEShift:samplenumberbound}
Suppose that the event $A$ holds. Set $\delta = 2 e^{-a}$. Then for each $\ell,k$, we have 
\begin{equation*}
\begin{split}
\frac{1}{\SN_{\ell, k}} 
&  \leq  \frac{1}{\SN_{\ell,k}^*} \left[ 1 + 4LK N^{-1} +  \frac{4}{\usc_{\min}}\sqrt[4]{             \frac{\log 2/\delta}{ \Delta_{\min}}} N^{-\frac{1}{4}}\right]\\
\end{split}
\end{equation*}
for any $1\leq \ell \leq L, 1\leq k\leq K$.
\end{lemma}
\begin{proof}
To show this, let us first establish the following useful lemma.
\begin{lemma}\label{lemma:FAMEShift:algproperty}
Suppose that the event $A$ holds. If Algorithm \ref{Alg:FAMEShift:MainAlg} draws at least one sample from $D_{\ell_0,k_0}$ after the first $2LK$ iterations, we must have, for every $\ell,k$,
\begin{equation*}
\SN_{\ell, k} \geq \left(\SN_{\ell_0, k_0}-1 \right) \usc_{\ell,k} \frac{\pmb p_{\ell, k}}{\pmb p_{\ell_0, k_0}}
\left(\usc_{\ell_0,k_0} + 2  \sqrt[4]{            \frac{\log 2/\delta}{2( \SN_{\ell_0,k_0}-`1)}} \right)^{-1}
\end{equation*}

\end{lemma}
\begin{proof}

Since the event $A$ holds, we have
\begin{equation*}
 \left\lbrace  \left| \sqrt{ 1- \frac{1}{t(t-1)} \sum_{i=1}^{t} \sum_{j=1,j\not=i}^{t} \mathbbm{1}_{\sample_{\ell,k,i}=\sample_{\ell,k,j}}}-\usc_{\ell,k} \right| \leq \sqrt[4]{            \frac{\log 2/\delta}{2t}}  \right\rbrace
\end{equation*}
for every $\ell,k,t$.
Since this holds for every fixed $t$, it should also holds for any random variable $t$. 
Specifically, we must have 
\begin{equation*}
 \left| \sqrt{ 1- \frac{1}{\SN_{\ell,k,n}(\SN_{\ell,k,n}-1)} \sum_{i=1}^{\SN_{\ell,k,n}} \sum_{j=1,j\not=i}^{\SN_{\ell,k,n}} \mathbbm{1}_{\sample_{\ell,k,i}=\sample_{\ell,k,j}}}-\usc_{\ell,k} \right| \leq \sqrt[4]{            \frac{\log 2/\delta}{2 \SN_{\ell,k,n}}}  
\end{equation*}
Note that, by definition,
\begin{equation*}
 \hatusc_{\ell,k,n} =  \sqrt{ 1- \frac{1}{\SN_{\ell,k,n}(\SN_{\ell,k,n}-1)} \sum_{i=1}^{\SN_{\ell,k,n}} \sum_{j=1,j\not=i}^{\SN_{\ell,k,n}} \mathbbm{1}_{\sample_{\ell,k,i}=\sample_{\ell,k,j}}}
\end{equation*}
We can then rewrite the above inequality as 
\begin{equation*}
  \left| \hatusc_{\ell,k,n} -\usc_{\ell,k} \right| \leq \sqrt[4]{            \frac{\log 2/\delta}{2 \SN_{\ell,k,n}}} 
\end{equation*}
That is to say,
\begin{equation*}
 \usc_{\ell,k} - \sqrt[4]{            \frac{\log 2/\delta}{2 \SN_{\ell,k,n}}} \leq  \hatusc_{\ell,k,n}  \leq \usc_{\ell,k} + \sqrt[4]{            \frac{\log 2/\delta}{2 \SN_{\ell,k,n}}} 
\end{equation*}
Adding  $\sqrt[4]{            \frac{\log 2/\delta}{2 \SN_{\ell,k,n}}}$ to both sides, this becomes
\begin{equation*}
 \usc_{\ell,k}  \leq  \hatusc_{\ell,k,n} + \sqrt[4]{            \frac{\log 2/\delta}{2 \SN_{\ell,k,n}}}  \leq \usc_{\ell,k} + 2 \sqrt[4]{            \frac{\log 2/\delta}{2 \SN_{\ell,k,n}}} 
\end{equation*}

Multiplying both sides by $\frac{\pmb p_{\ell, k}}{\SN_{\ell, k,n}}$, we have
\begin{equation}\label{eq:FAMEAPIShift:uscorebound}
\frac{\pmb p_{\ell, k}}{\SN_{\ell, k,n}} \usc_{\ell,k} \leq \frac{\pmb p_{\ell, k}}{\SN_{\ell, k,n}} \left(\hatusc_{\ell,k,n} +  \sqrt[4]{            \frac{\log 2/\delta}{2 \SN_{\ell,k,n}}} \right) \leq  \frac{\pmb p_{\ell, k}}{\SN_{\ell, k,n}} \left(\usc_{\ell,k} + 2 \sqrt[4]{            \frac{\log 2/\delta}{2 \SN_{\ell,k,n}}} \right)
\end{equation}

which holds for any $\ell,k,n$.
Note that  $N>2LK$,  there must exist some $\ell_0,k_0$, such that  Algorithm \ref{Alg:FAMEShift:MainAlg} draws a sample from the data partition $D_{\ell_0,k_0}$ after the first $2LK$ iterations.
Suppose the last time a sample is drawn from  $D_{\ell_0,k_0}$ is  $n_0>2LK$.
That is to say, $\SN_{\ell_0,k_0,n_0} = \SN_{\ell_0,k_0,n}-1, \forall n=n_0+1,\cdots, N$.
Since Algorithm \ref{Alg:FAMEShift:MainAlg} chooses $\ell_0,k_0$ at iteration $n_0$, by line 11 in Algorithm \ref{Alg:FAMEShift:MainAlg}, we have 
\begin{equation*}
\ell_0, k_0 = \arg \max  \frac{\pmb p_{\ell, k}}{\SN_{\ell, k,n_0}} (\hatusc_{\ell,k,n_0} +  \sqrt[4]{            \frac{\log 2/\delta}{2 \SN_{\ell,k,n_0}}} )    
\end{equation*}
By definition of $\arg\max$, we have 
\begin{equation*}
\frac{\pmb p_{\ell_0, k_0}}{\SN_{\ell_0, k_0,n_0}} (\hatusc_{\ell_0,k_0,n_0} +  \sqrt[4]{            \frac{\log 2/\delta}{2 \SN_{\ell_0,k_0,n_0}}} ) \geq  \frac{\pmb p_{\ell, k}}{\SN_{\ell, k,n_0}} (\hatusc_{\ell,k,n_0} +  \sqrt[4]{            \frac{\log 2/\delta}{2 \SN_{\ell,k,n_0}}} )    
\end{equation*}
Setting $n=n_0$ in the first half of inequality \ref{eq:FAMEAPIShift:uscorebound}, we have 
\begin{equation*}
\frac{\pmb p_{\ell, k}}{\SN_{\ell, k,n_0}} \left(\hatusc_{\ell,k,n_0} +  \sqrt[4]{            \frac{\log 2/\delta}{2 \SN_{\ell,k,n_0}}} \right) \geq \frac{\pmb p_{\ell, k}}{\SN_{\ell, k,n_0}} \usc_{\ell,k,n_0}
\end{equation*}
Combining the above two inequalities gives
\begin{equation*}
\frac{\pmb p_{\ell_0, k_0}}{\SN_{\ell_0, k_0,n_0}} (\hatusc_{\ell_0,k_0,n_0} +  \sqrt[4]{            \frac{\log 2/\delta}{2 \SN_{\ell_0,k_0,n_0}}} ) \geq \frac{\pmb p_{\ell, k}}{\SN_{\ell, k,n_0}} \usc_{\ell,k}
\end{equation*}
Noting that by definition, $\SN_{\ell,k,n_0} \leq \SN_{\ell,k,N} = \SN_{\ell,k}$, we can lower bound $1/\SN_{\ell,k,n_0}$ by $1/\SN_{\ell,k}$, and the above inequality becomes
\begin{equation*}
\frac{\pmb p_{\ell_0, k_0}}{\SN_{\ell_0, k_0,n_0}} (\hatusc_{\ell_0,k_0,n_0} +  \sqrt[4]{            \frac{\log 2/\delta}{2 \SN_{\ell_0,k_0,n_0}}} ) \geq \frac{\pmb p_{\ell, k}}{\SN_{\ell, k}} \usc_{\ell,k}
\end{equation*}
Now setting $n=n_0,\ell=\ell_0,k=k_0$ in the second half of inequality \ref{eq:FAMEAPIShift:uscorebound}, we have 
\begin{equation*}
\frac{\pmb p_{\ell_0, k_0}}{\SN_{\ell_0, k_0,n}} \left(\hatusc_{\ell_0,k_0,n} +  \sqrt[4]{            \frac{\log 2/\delta}{2 \SN_{\ell_0,k_0,n_0}}} \right) \leq  \frac{\pmb p_{\ell_0, k_0}}{\SN_{\ell_0, k_0,n_0}} \left(\usc_{\ell_0,k_0} + 2 \sqrt[4]{            \frac{\log 2/\delta}{2 \SN_{\ell_0,k_0,n_0}}} \right)
\end{equation*}
Combining the above two inequalities, we have 
\begin{equation*}
\frac{\pmb p_{\ell_0, k_0}}{\SN_{\ell_0, k_0,n_0}} (\usc_{\ell_0,k_0} + 2  \sqrt[4]{            \frac{\log 2/\delta}{2 \SN_{\ell_0,k_0,n_0}}} ) \geq \frac{\pmb p_{\ell, k}}{\SN_{\ell, k}} \usc_{\ell,k}
\end{equation*}
Observe that $n_0$ is the last time a sample is drawn from partition $D_{\ell_0,k_0}$, we have $\SN_{\ell_0,k_0,n_0} = \SN_{\ell_0,k_0,n}-1, \forall n=n_0+1,\cdots, N$.
Specifically,  $\SN_{\ell_0,k_0,n_0} = \SN_{\ell_0,k_0,N}-1=\SN_{\ell_0,k_0}-1$. 
Replacing $\SN_{\ell_0,k_0,n_0}$ by 
$\SN_{\ell_0,k_0}-1$ in the above inequality, we get
\begin{equation*}
\frac{\pmb p_{\ell_0, k_0}}{\SN_{\ell_0, k_0}-1} (\usc_{\ell_0,k_0} + 2  \sqrt[4]{            \frac{\log 2/\delta}{2( \SN_{\ell_0,k_0}-`1)}} ) \geq \frac{\pmb p_{\ell, k}}{\SN_{\ell, k}} \usc_{\ell,k}
\end{equation*}
which holds for every $\ell,k$. Rearranging the terms completes the proof.
\end{proof}

Now we are ready to prove the bound on $\SN_{\ell,k}-\SN_{\ell,k}^*$.

Let us first consider the lower bound. 
By definition, we have 
\begin{equation*}
    \sum_{\ell=1}^{L}\sum_{k=1}^{K} \SN_{\ell,k} = N
\end{equation*}
Subtracting 2 from each element, we have
\begin{equation*}
    \sum_{\ell=1}^{L}\sum_{k=1}^{K} (\SN_{\ell,k}-2) = N-2LK = \frac{N-2LK }{N} N
\end{equation*}
Note that by definition, $N = \sum_{\ell=1}^{L}\sum_{k=1}^{K} \SN_{\ell,k}^*$. We can now replace the second $N$ in the above equality, and obtain 
\begin{equation*}
    \sum_{\ell=1}^{L}\sum_{k=1}^{K} (\SN_{\ell,k}-2) = \frac{N-2LK }{N} N= \frac{N-2LK }{N}  \sum_{\ell=1}^{L}\sum_{k=1}^{K} \SN_{\ell,k}^* = \sum_{\ell=1}^{L}\sum_{k=1}^{K}  \frac{\SN_{\ell,k}^*(N-2LK )}{N} 
\end{equation*}

Now let us consider two cases.

(i) Assume $\SN_{\ell,k} -2 \geq \frac{\SN_{\ell,k}^*(N-2LK )}{N} $.
That is to say, $\SN_{\ell,k} \geq \frac{\SN_{\ell,k}^*(N-2LK )}{N}+2$. Then we have 
\begin{equation*}
    \frac{1}{\SN_{\ell,k}} \leq \frac{1}{\frac{\SN_{\ell,k}^*(N-2LK )}{N}+2}
\end{equation*}
subtracting $\frac{1}{\SN_{\ell,k}^*}$ from both sides, we get 
\begin{equation*}
\begin{split}
    \frac{1}{\SN_{\ell,k}}-    \frac{1}{\SN_{\ell,k}^*} & \leq \frac{1}{\frac{\SN_{\ell,k}^*(N-2LK )}{N}+2} -     \frac{1}{\SN_{\ell,k}^*}\\ 
    & = \frac{\SN_{\ell,k}^*-\frac{\SN_{\ell,k}^*(N-2LK )}{N}-2}{\SN_{\ell,k}^* \cdot(\frac{\SN_{\ell,k}^*(N-2LK )}{N}+2)}\\
    &\leq   \frac{\SN_{\ell,k}^*-\frac{\SN_{\ell,k}^*(N-2LK )}{N}}{\SN_{\ell,k}^* \cdot(\frac{\SN_{\ell,k}^*(N-2LK )}{N})}\\
    &= \frac{\frac{2LK\SN_{\ell,k}^* }{N}}{\SN_{\ell,k}^* \cdot(\frac{\SN_{\ell,k}^*(N-2LK )}{N})}\\
    &= \frac{2LK }{ \SN_{\ell,k}^*(N-2LK )}\\
\end{split}
\end{equation*}
where the last inequality is simply by removing the constant $2$.
Now by assumption, $N>4LK$, we have $N-2LK<\frac{1}{2}N$.
The above inequality can be further simplified as

\begin{equation*}
\begin{split}
    \frac{1}{\SN_{\ell,k}}-    \frac{1}{\SN_{\ell,k}^*}  
    &\leq \frac{2LK }{ \SN_{\ell,k}^*(N-2LK )} \leq  \frac{4LK }{ \SN_{\ell,k}^*N}\\
\end{split}
\end{equation*}
By definition, we have $\SN_{\ell,k}^*=N \Delta_{\ell,k} \leq N \Delta_{\min}$.
Therefore, we have 
\begin{equation*}
\begin{split}
    \frac{1}{\SN_{\ell,k}}-    \frac{1}{\SN_{\ell,k}^*}  
    &\leq \frac{2LK }{ \SN_{\ell,k}^*(N-2LK )} \leq  \frac{4LK }{ \SN_{\ell,k}^*N}\\
\end{split}
\end{equation*}
That is to say, 

\begin{equation*}
\begin{split}
    \frac{1}{\SN_{\ell,k}} \leq     \frac{1}{\SN_{\ell,k}^*}\left[  1
    + \frac{4LK}{N}\right] 
\end{split}
\end{equation*}
And thus, apparently, 
\begin{equation*}
\begin{split}
\frac{1}{\SN_{\ell, k}} 
&  \leq  \frac{1}{\SN_{\ell,k}^*} \left[ 1 + 4LK N^{-1} +  \frac{4}{\usc_{\min}}\sqrt[4]{             \frac{\log 2/\delta}{ \Delta_{\min}}} N^{-\frac{1}{4}}\right]\\
\end{split}
\end{equation*}

(ii) Assume $\SN_{\ell,k} -2 < \frac{\SN_{\ell,k}^*(N-2LK )}{N} $.
Then there must exists some $\ell_0,k_0$ such that $\SN_{\ell_0,k_0} -2 > \frac{\SN_{\ell_0,k_0}^*(N-2LK )}{N}>0$.
That is to say, Algorithm \ref{Alg:FAMEShift:MainAlg} draws at least one sample from $D_{\ell_0,k_0}$ after the first $2LK$ iterations.
By Lemma \ref{lemma:FAMEShift:algproperty}, we must have 
\begin{equation*}
\SN_{\ell, k} \geq \left(\SN_{\ell_0, k_0}-1 \right) \usc_{\ell,k} \frac{\pmb p_{\ell, k}}{\pmb p_{\ell_0, k_0}}
\left(\usc_{\ell_0,k_0} + 2  \sqrt[4]{            \frac{\log 2/\delta}{2( \SN_{\ell_0,k_0}-`1)}} \right)^{-1}
\end{equation*}

$\SN_{\ell_0,k_0} -2 > \frac{\SN_{\ell_0,k_0}^*(N-2LK )}{N}$ implies 

\begin{equation*}
    \SN_{\ell_0,k_0} -1 > \SN_{\ell_0,k_0 -2 } > \frac{\SN_{\ell_0,k_0}^*(N-2LK )}{N}
\end{equation*}
Therefore, we can use this lower bound on ${\SN_{\ell_0,k_0} -1 }$ in the above inequality and obtain 

\begin{equation*}
\begin{split}
\SN_{\ell, k} & \geq  \frac{\SN_{\ell_0,k_0}^*(N-2LK )}{N} \usc_{\ell,k} \frac{\pmb p_{\ell, k}}{\pmb p_{\ell_0, k_0}}
\left(\usc_{\ell_0,k_0} + 2  \sqrt[4]{            \frac{\log 2/\delta}{2\frac{\SN_{\ell_0,k_0}^*(N-2LK )}{N}}} \right)^{-1}\\
=& \frac{\SN_{\ell_0,k_0}^*(N-2LK )}{N}  \frac{\usc_{\ell,k} \pmb p_{\ell, k}}{\usc_{\ell_0,k_0} \pmb p_{\ell_0, k_0}}
\left(1 + \frac{2}{\usc_{\ell_0,k_0}}\sqrt[4]{            \frac{\log 2/\delta}{2\frac{\SN_{\ell_0,k_0}^*(N-2LK )}{N}}} \right)^{-1}\\
=& \frac{\SN_{\ell,k}^*(N-2LK )}{N} 
\left(1 + \frac{2}{\usc_{\ell_0,k_0}}\sqrt[4]{            \frac{\log 2/\delta}{2\frac{\SN_{\ell_0,k_0}^*(N-2LK )}{N}}} \right)^{-1}\\
\end{split}
\end{equation*}
where the first equality is by dividing $\usc_{\ell_0,k_0}$ at both denominator and numerator, and the second equality uses the fact that $\SN_{\ell,k}^*$ is proportional to  
$p_{\ell, k}{\usc_{\ell,k}}$.
Taking inverse of the above inequality gives 

\begin{equation*}
\begin{split}
\frac{1}{\SN_{\ell, k}} & \leq   \frac{N}{\SN_{\ell,k}^*(N-2LK )} 
\left(1 + \frac{2}{\usc_{\ell_0,k_0}}\sqrt[4]{            \frac{\log 2/\delta}{2\frac{\SN_{\ell_0,k_0}^*(N-2LK )}{N}}} \right)\\
\end{split}
\end{equation*}
Now let us simplify this inequality. Let us first expand all terms and obtain
\begin{equation*}
\begin{split}
\frac{1}{\SN_{\ell, k}} & \leq   \frac{N}{\SN_{\ell,k}^*(N-2LK )} 
\left(1 + \frac{2}{\usc_{\ell_0,k_0}}\sqrt[4]{            \frac{\log 2/\delta}{2\frac{\SN_{\ell_0,k_0}^*(N-2LK )}{N}}} \right)\\
& = \frac{1}{\SN_{\ell,k}^*}  + \frac{2LK}{\SN_{\ell,k}^* (N-2LK)} +  \frac{N}{\SN_{\ell,k}^*(N-2LK )} 
 \cdot \frac{2}{\usc_{\ell_0,k_0}}\sqrt[4]{            \frac{\log 2/\delta}{2\frac{\SN_{\ell_0,k_0}^*(N-2LK )}{N}}} \\
 & = \frac{1}{\SN_{\ell,k}^*}  + \frac{2LK}{\SN_{\ell,k}^* (N-2LK)} +  \frac{2}{\usc_{\ell_0,k_0}}\sqrt[4]{        \left(\frac{N}{N-2LK}\right)^5     \frac{\log 2/\delta}{2\SN_{\ell,k}^{*4} \SN_{\ell_0,k_0}^*}} \\
\end{split}
\end{equation*}
For the second term, by assumption, $N>4LK$ and thus $N-2LK> 1/2N $, we have
\begin{equation*}
    \frac{2LK}{N-2LK} \leq \frac{4LK}{N}
\end{equation*}
Thus the above equation becomes
\begin{equation*}
\begin{split}
\frac{1}{\SN_{\ell, k}} & \leq    \frac{1}{\SN_{\ell,k}^*}  + \frac{4LK}{\SN_{\ell,k}^* N} +  \frac{2}{\usc_{\ell_0,k_0}}\sqrt[4]{        \left(\frac{N}{N-2LK}\right)^5     \frac{\log 2/\delta}{2\SN_{\ell,k}^{*4} \SN_{\ell_0,k_0}^*}} \\
\end{split}
\end{equation*}
For the third term, $N>4LK$ also implies 
\begin{equation*}
    \frac{N}{N-2LK} =  1 +     \frac{2LK}{N-2LK} < 1 +\frac{2LK}{4LK-2LK} = 2 
\end{equation*}
Thus the above inequality can be further simplified as 
\begin{equation*}
\begin{split}
\frac{1}{\SN_{\ell, k}} & \leq    \frac{1}{\SN_{\ell,k}^*}  + \frac{4LK}{\SN_{\ell,k}^* N} +  \frac{4}{\usc_{\ell_0,k_0}}\sqrt[4]{             \frac{\log 2/\delta}{\SN_{\ell,k}^{*4} \SN_{\ell_0,k_0}^*}} \\
& \leq  \frac{1}{\SN_{\ell,k}^*} \left[ 1 + \frac{4LK}{ N} +  \frac{4}{\usc_{\ell_0,k_0}}\sqrt[4]{             \frac{\log 2/\delta}{ \SN_{\ell_0,k_0}^*}}\right] \\
\end{split}
\end{equation*}
Now by definition, $\usc_{\ell_0,k_0}\geq \usc_{\min}$, and $\SN_{\ell_0,k_0}^* = N \Delta_{\ell_0,k_0} \geq N \Delta_{\min}$, we can further simplify the above inequality 
\begin{equation*}
\begin{split}
\frac{1}{\SN_{\ell, k}}& \leq  \frac{1}{\SN_{\ell,k}^*} \left[ 1 + \frac{4LK}{ N} +  \frac{4}{\usc_{\ell_0,k_0}}\sqrt[4]{             \frac{\log 2/\delta}{ \SN_{\ell_0,k_0}^*}}\right] \\
 & \leq  \frac{1}{\SN_{\ell,k}^*} \left[ 1 + \frac{4LK}{ N} +  \frac{4}{\usc_{\ell_0,k_0}}\sqrt[4]{             \frac{\log 2/\delta}{ N \Delta_{\min}}} \right] \\
 & \leq  \frac{1}{\SN_{\ell,k}^*} \left[ 1 + \frac{4LK}{ N} +  \frac{4}{\usc_{\min}}\sqrt[4]{             \frac{\log 2/\delta}{ N \Delta_{\min}}} \right]\\
\end{split}
\end{equation*}
That is to say,  
\begin{equation*}
\begin{split}
\frac{1}{\SN_{\ell, k}} 
&  \leq  \frac{1}{\SN_{\ell,k}^*} \left[ 1 + 4LK N^{-1} +  \frac{4}{\usc_{\min}}\sqrt[4]{             \frac{\log 2/\delta}{ \Delta_{\min}}} N^{-\frac{1}{4}}\right]\\
\end{split}
\end{equation*}
That is to say, no matter $\SN_{\ell,k} -2 < \frac{\SN_{\ell,k}^*(N-2LK )}{N} $ or not, this inequality always holds, which completes the proof.
\end{proof}

Now we are ready to prove Theorem \ref{thm:FAMEAPIShift:mainbound}.
Let us first note that the loss can be written as 
\begin{equation}\label{eq:FAMEShift:temp00}
\begin{split}
    \loss_N &= \sum_{\ell=1}^{L} \sum_{k=1}^{K} \sum_{j=1}^{L} \pmb p_{\ell,k}^2 \Exp[\truemu_{\ell,k,j} -\Smu_{\ell,k,j}  ]^2\\
& = \sum_{\ell=1}^{L} \sum_{k=1}^{K} \sum_{j=1}^{L} \pmb p_{\ell,k}^2 \Exp[(\truemu_{\ell,k,j} - \frac{1}{\SN_{\ell,k}}\sum_{t=1}^{\SN_{\ell,k}} \mathbbm{1}_{\sample_{\ell,k,t}=j} )^2 \mathbbm{1}_{ A}] \\
&+  \sum_{\ell=1}^{L} \sum_{k=1}^{K} \sum_{j=1}^{L} \pmb p_{\ell,k}^2 \Exp[(\truemu_{\ell,k,j} - \frac{1}{\SN_{\ell,k}}\sum_{t=1}^{\SN_{\ell,k}} \mathbbm{1}_{\sample_{\ell,k,t}=j} )^2\mathbbm{1}_{A^C}]\\   
\end{split}
\end{equation}
Let us first consider the first term. 
\begin{equation}\label{eq:FAMEShift:temp2}
\begin{split}
     & \sum_{\ell=1}^{L} \sum_{k=1}^{K} \sum_{j=1}^{L} \pmb p_{\ell,k}^2 \Exp[(\truemu_{\ell,k,j} -\Smu_{\ell,k,j})^2  \mathbbm{1}_{A}]\\
= & \sum_{\ell=1}^{L} \sum_{k=1}^{K} \sum_{j=1}^{L} \pmb p_{\ell,k}^2 \Exp[(\truemu_{\ell,k,j} - \frac{1}{\SN_{\ell,k}}\sum_{t=1}^{\SN_{\ell,k}} \mathbbm{1}_{\sample_{\ell,k,t}=j})^2 \mathbbm{1}_{A}]\\   
= & \sum_{\ell=1}^{L} \sum_{k=1}^{K} \sum_{j=1}^{L} \pmb p_{\ell,k}^2 \Exp\left[ \frac{1}{\SN_{\ell,k}^2} \left( \SN_{\ell,k} \truemu_{\ell,k,j} - \sum_{t=1}^{\SN_{\ell,k}} \mathbbm{1}_{\sample_{\ell,k,t}=j} \right)^2 \mathbbm{1}_{A}\right]\\
\end{split}
\end{equation}
where we plug in the definition of $\Smu$.
By Lemma \ref{lemma:FAMEShift:samplenumberbound}, we have the upper bound on $1/\SN_{\ell,k}$

\begin{equation*}
\begin{split}
\frac{1}{\SN_{\ell, k}} 
&  \leq  \frac{1}{\SN_{\ell,k}^*} \left[ 1 + 4LK N^{-1} +  \frac{4}{\usc_{\min}}\sqrt[4]{             \frac{\log 2/\delta}{ \Delta_{\min}}} N^{-\frac{1}{4}}\right]\\
\end{split}
\end{equation*}
Therefore, we can use this inequality to obtain 
\begin{equation}\label{eq:FAMEShift:temp1}
\begin{split}
& \Exp\left[ \frac{1}{\SN_{\ell,k}^2} \left( \SN_{\ell,k} \truemu_{\ell,k,j} - \sum_{t=1}^{\SN_{\ell,k}} \mathbbm{1}_{\sample_{\ell,k,t}=j} \right)^2 \mathbbm{1}_{A}\right]\\
\leq  & [\frac{1}{\SN_{\ell,k}^*}  + \frac{4LK}{ \Delta_{\min}}  N^{-2}+  \frac{4}{\usc_{\min}}\sqrt[4]{             \frac{\log 2/\delta}{ \Delta_{\min}^5}} N^{-\frac{5}{4}}]^{2} \Exp\left[  \left( \SN_{\ell,k} \truemu_{\ell,k,j} - \sum_{t=1}^{\SN_{\ell,k}} \mathbbm{1}_{\sample_{\ell,k,t}=j} \right)^2 \mathbbm{1}_{A}\right]\\
\end{split}
\end{equation}

It is not hard to see that $\SN_{\ell,k}$ is a stopping time.
In fact, for any $\ell,k$, and any time $n$, a new sample is drawn purely based on estimated uncertainty score $\hatusc$ and observed sample number $\SN_{\ell,k,n-1}$ up to the current iteration, which is part of the history.
As $\SN_{\ell,k}<N$ is bounded, 
 $\SN_{\ell,k}$ is a stopping time.
 Hence, we can apply Lemma \ref{lemma:FAMEShift:wald-inequality}, and obtain

\begin{equation*}
\begin{split}
 & \Exp\left[  \left( \SN_{\ell,k} \truemu_{\ell,k,j} - \sum_{t=1}^{\SN_{\ell,k}} \mathbbm{1}_{\sample_{\ell,k,t}=j} \right)^2 \mathbbm{1}_{A}\right] \leq \Exp\left[  \left( \SN_{\ell,k} \truemu_{\ell,k,j} - \sum_{t=1}^{\SN_{\ell,k}} \mathbbm{1}_{\sample_{\ell,k,t}=j} \right)^2 \right] \\
 \leq &\Exp[\SN_{\ell,k}]  \Pr[\sample_{\ell,k,1}=j](1-\Pr[\sample_{\ell,k,1}=j]) \\
\end{split}
\end{equation*}
where the first inequality uses the fact that square term must be non-negative, and the second inequality uses the fact that, for Bernoulli distribution with mean $a$, its variance is $a(1-a)$.
Applying this in inequality \ref{eq:FAMEShift:temp1}, we have 
\begin{equation*}
\begin{split}
& \Exp\left[ \frac{1}{\SN_{\ell,k}^2} \left( \SN_{\ell,k} \truemu_{\ell,k,j} - \sum_{t=1}^{\SN_{\ell,k}} \mathbbm{1}_{\sample_{\ell,k,t}=j} \right)^2 \mathbbm{1}_{A}\right]\\
\leq  & [\frac{1}{\SN_{\ell,k}^*}  + \frac{4LK}{ \Delta_{\min}}  N^{-2}+  \frac{4}{\usc_{\min}}\sqrt[4]{             \frac{\log 2/\delta}{ \Delta_{\min}^5}} N^{-\frac{5}{4}}]^{2} \Exp\left[  \left( \SN_{\ell,k} \truemu_{\ell,k,j} - \sum_{t=1}^{\SN_{\ell,k}} \mathbbm{1}_{\sample_{\ell,k,t}=j} \right)^2 \mathbbm{1}_{A}\right]\\
\leq & [\frac{1}{\SN_{\ell,k}^*}  + \frac{4LK}{ \Delta_{\min}}  N^{-2}+  \frac{4}{\usc_{\min}}\sqrt[4]{             \frac{\log 2/\delta}{ \Delta_{\min}^5}} N^{-\frac{5}{4}}]^{2} \Exp[\SN_{\ell,k}]  \Pr[\sample_{\ell,k,1}=j] (1-\Pr[\sample_{\ell,k,1}=j] )
\end{split}
\end{equation*}
Now applying this in equality \ref{eq:FAMEShift:temp2}, we get  

\begin{equation}\label{eq:FAMEShift:temp4}
\begin{split}
     & \sum_{\ell=1}^{L} \sum_{k=1}^{K} \sum_{j=1}^{L} \pmb p_{\ell,k}^2 \Exp[(\truemu_{\ell,k,j} -\Smu_{\ell,k,j})^2  \mathbbm{1}_{A} ]\\
= & \sum_{\ell=1}^{L} \sum_{k=1}^{K} \sum_{j=1}^{L} \pmb p_{\ell,k}^2 \Exp\left[ \frac{1}{\SN_{\ell,k}^2} \left( \SN_{\ell,k} \truemu_{\ell,k,j} - \sum_{t=1}^{\SN_{\ell,k}} \mathbbm{1}_{\sample_{\ell,k,t}=j} \right)^2 \mathbbm{1}_{A}\right]\\
\leq  & \sum_{\ell=1}^{L} \sum_{k=1}^{K} \sum_{j=1}^{L} \pmb p_{\ell,k}^2 [\frac{1}{\SN_{\ell,k}^*}  + \frac{4LK}{ \Delta_{\min}}  N^{-2}+  \frac{4}{\usc_{\min}}\sqrt[4]{             \frac{\log 2/\delta}{ \Delta_{\min}^5}} N^{-\frac{5}{4}}]^{2} \Exp[\SN_{\ell,k}]  \Pr[\sample_{\ell,k,1}=j](1-\Pr[\sample_{\ell,k,1}=j] ) \\
=  & \sum_{\ell=1}^{L} \sum_{k=1}^{K}  \pmb p_{\ell,k}^2 \usc^2_{\ell,k} [\frac{1}{\SN_{\ell,k}^*}  + \frac{4LK}{ \Delta_{\min}}  N^{-2}+  \frac{4}{\usc_{\min}}\sqrt[4]{             \frac{\log 2/\delta}{ \Delta_{\min}^5}} N^{-\frac{5}{4}}]^{2} \Exp[\SN_{\ell,k}] \\
\end{split}
\end{equation}
where the last equation uses the fact that $\usc_{\ell,k} = 1-\sum_{j=1}^{L}\Pr^2[\sample_{\ell,k,1}=j] = \sum_{j=1}^{L}\Pr[\sample_{\ell,k,1}=j] (1-\Pr[\sample_{\ell,k,1}=j] ) $.
Applying the inequality $1/(1+x)\leq 1-x$

\begin{equation*}
\begin{split}
\frac{1}{\SN_{\ell, k}} 
&  \leq  \frac{1}{\SN_{\ell,k}^*} \left[ 1 + 4LK N^{-1} +  \frac{4}{\usc_{\min}}\sqrt[4]{             \frac{\log 2/\delta}{ \Delta_{\min}}} N^{\frac{1}{4}}\right]\\
\end{split}
\end{equation*}

Note that 
\begin{equation*}\label{eq:FAMEShift:temp3}
\begin{split}
 &  \pmb p_{\ell,k}^2 \usc^2_{\ell,k} [\frac{1}{\SN_{\ell,k}^*} \left[ 1 + 4LK N^{-1} +  \frac{4}{\usc_{\min}}\sqrt[4]{             \frac{\log 2/\delta}{ \Delta_{\min}}} N^{-\frac{1}{4}}\right]]^{2} \Exp[\SN_{\ell,k}] \\
 =& (\frac{ \pmb p_{\ell,k}\usc_{\ell,k}}{\SN_{\ell,k}^*})^2  \left[ 1 + 4LK N^{-1} +  \frac{4}{\usc_{\min}}\sqrt[4]{             \frac{\log 2/\delta}{ \Delta_{\min}}} N^{-\frac{1}{4}}\right]^2 \Exp[\SN_{\ell,k}] \\
 =& N^{-2} (\sum_{\ell',k'} \pmb p_{\ell',k'} \usc_{\ell',k'})^2 \left[ 1 + 4LK N^{-1} +  \frac{4}{\usc_{\min}}\sqrt[4]{             \frac{\log 2/\delta}{ \Delta_{\min}}} N^{-\frac{1}{4}}\right]^2 \Exp[\SN_{\ell,k}]  \end{split}
\end{equation*}
where the last equation is by definition of $\SN_{\ell,k}$.
Now applying this in inequality \ref{eq:FAMEShift:temp4}, we have 
\begin{equation}\label{eq:FAMEShift:temp5}
\begin{split}
     & \sum_{\ell=1}^{L} \sum_{k=1}^{K} \sum_{j=1}^{L} \pmb p_{\ell,k}^2 \Exp[(\truemu_{\ell,k,j} -\Smu_{\ell,k,j})^2  \mathbbm{1}_{A} ]\\
\leq  & \sum_{\ell=1}^{L} \sum_{k=1}^{K}  \pmb p_{\ell,k}^2 \usc^2_{\ell,k} [\frac{1}{\SN_{\ell,k}^*}  + \frac{4LK}{ \Delta_{\min}}  N^{-2}+  \frac{4}{\usc_{\min}}\sqrt[4]{             \frac{\log 2/\delta}{ \Delta_{\min}^5}} N^{-\frac{5}{4}}]^{2} \Exp[\SN_{\ell,k}] \\
=& \sum_{\ell=1}^{L} \sum_{k=1}^{K}  N^{-2} (\sum_{\ell',k'} \pmb p_{\ell',k'} \usc_{\ell',k'})^2 \left[ 1 + 4LK N^{-1} +  \frac{4}{\usc_{\min}}\sqrt[4]{             \frac{\log 2/\delta}{ \Delta_{\min}}} N^{-\frac{1}{4}}\right]^2 \Exp[\SN_{\ell,k}] \\
=&N^{-2} (\sum_{\ell',k'} \pmb p_{\ell',k'} \usc_{\ell',k'})^2 \left[ 1 + 4LK N^{-1} +  \frac{4}{\usc_{\min}}\sqrt[4]{             \frac{\log 2/\delta}{ \Delta_{\min}}} N^{-\frac{1}{4}}\right]^2  \sum_{\ell=1}^{L} \sum_{k=1}^{K}  \Exp[\SN_{\ell,k}] \\
=&N^{-2} (\sum_{\ell',k'} \pmb p_{\ell',k'} \usc_{\ell',k'})^2 \left[ 1 + 4LK N^{-1} +  \frac{4}{\usc_{\min}}\sqrt[4]{             \frac{\log 2/\delta}{ \Delta_{\min}}} N^{-\frac{1}{4}}\right]^2  N \\
=&N^{-1} (\sum_{\ell,k} \pmb p_{\ell,k} \usc_{\ell,k})^2 \left[ 1 + 4LK N^{-1} +  \frac{4}{\usc_{\min}}\sqrt[4]{             \frac{\log 2/\delta}{ \Delta_{\min}}} N^{-\frac{1}{4}}\right]^2  \\
\end{split}
\end{equation}
where the second equation uses the fact that only $\Exp[\SN_{\ell,k}]$ depends on $\ell,k$,  the third equation uses the fact that $\sum_{\ell=1}^{L} \sum_{k=1}^{K}  \SN_{\ell,k}=N$ and thus $\sum_{\ell=1}^{L} \sum_{k=1}^{K}  \Exp[\SN_{\ell,k}]=N$.
Note that $\delta = L^{-1} K^{-1} N^{-\frac{5}{4}}$, we have 

\begin{equation*}
\begin{split}
&N^{-1} (\sum_{\ell,k} \pmb p_{\ell,k} \usc_{\ell,k})^2 \left[ 1 + 4LK N^{-1} +  \frac{4}{\usc_{\min}}\sqrt[4]{             \frac{\log 2/\delta}{ \Delta_{\min}}} N^{-\frac{1}{4}}\right]^2  \\
=& N^{-1} (\sum_{\ell,k} \pmb p_{\ell,k} \usc_{\ell,k})^2 \left[ 1 + O(N^{-\frac{1}{4}} \log^{\frac{1}{4}} N ) \right] \\
=& N^{-1} (\sum_{\ell,k} \pmb p_{\ell,k} \usc_{\ell,k})^2 + O(N^{-\frac{5}{4}} \log^{\frac{1}{4}} N )  \\
\end{split}
\end{equation*}
Applying this back to inequality \ref{eq:FAMEShift:temp5}, we have 
\begin{equation}\label{eq:FAMEShift:temp6}
\begin{split}
 & \sum_{\ell=1}^{L} \sum_{k=1}^{K} \sum_{j=1}^{L} \pmb p_{\ell,k}^2 \Exp[(\truemu_{\ell,k,j} -\Smu_{\ell,k,j})^2  \mathbbm{1}_{A} ] 
\leq  N^{-1} (\sum_{\ell,k} \pmb p_{\ell,k} \usc_{\ell,k})^2 + O(N^{-\frac{5}{4}} \log^{\frac{1}{4}} N ) 
\end{split}
\end{equation}
Now consider the second term in equation \ref{eq:FAMEShift:temp00}.
As $\truemu$ and $\Smu$ are within $\{0,1\}$, we have \begin{equation*}
    (\truemu_{\ell,k,j} - \frac{1}{\SN_{\ell,k}}\sum_{t=1}^{\SN_{\ell,k}} \mathbbm{1}_{\sample_{\ell,k,t}=j} )^2\in[0,1]
\end{equation*}
Therefore,
\begin{equation*}
\begin{split}
 & \sum_{\ell=1}^{L} \sum_{k=1}^{K} \sum_{j=1}^{L} \pmb p_{\ell,k}^2 \Exp[(\truemu_{\ell,k,j} - \frac{1}{\SN_{\ell,k}}\sum_{t=1}^{\SN_{\ell,k}} \mathbbm{1}_{\sample_{\ell,k,t}=j} )^2\mathbbm{1}_{A^C}]\leq \sum_{\ell=1}^{L} \sum_{k=1}^{K} \sum_{j=1}^{L} \pmb p_{\ell,k}^2 \Pr[{A^C}]\\
\end{split}
\end{equation*}
By Lemma \ref{lemma:FAMEShift:highprob}, the probability of $A$ is at least $1-KLN\delta$.
Hence, the probability of $A^C$ is at most $KLN\delta$.
Hence, 
\begin{equation*}
\begin{split}
 & \sum_{\ell=1}^{L} \sum_{k=1}^{K} \sum_{j=1}^{L} \pmb p_{\ell,k}^2 \Exp[(\truemu_{\ell,k,j} - \frac{1}{\SN_{\ell,k}}\sum_{t=1}^{\SN_{\ell,k}} \mathbbm{1}_{\sample_{\ell,k,t}=j} )^2\mathbbm{1}_{A^C}]\\  
 \leq & \sum_{\ell=1}^{L} \sum_{k=1}^{K} \sum_{j=1}^{L} \pmb p_{\ell,k}^2 \Pr[{A^C}]
 \leq  \sum_{\ell=1}^{L} \sum_{k=1}^{K} \sum_{j=1}^{L} \pmb p_{\ell,k}^2 LKN\delta\\ \leq & \sum_{j=1}^{L} \pmb  LKN\delta = L^2 KN\delta
\end{split}
\end{equation*}
where the last inequality uses the fact that $\sum_{\ell=1}^{L} \sum_{k=1}^{K}  \pmb p_{\ell,k}^2\leq 1$ since $\sum_{\ell=1}^{L} \sum_{k=1}^{K}  \pmb p_{\ell,k} =  1$ and $\pmb p_{\ell,k}\geq 0$. Since $\delta = L^{-2} K^{-1} N^{-\frac{9}{4}}$, we have 
\begin{equation*}
\begin{split}
 & \sum_{\ell=1}^{L} \sum_{k=1}^{K} \sum_{j=1}^{L} \pmb p_{\ell,k}^2 \Exp[(\truemu_{\ell,k,j} - \frac{1}{\SN_{\ell,k}}\sum_{t=1}^{\SN_{\ell,k}} \mathbbm{1}_{\sample_{\ell,k,t}=j} )^2\mathbbm{1}_{A^C}]\\  
 \leq &  L^2 KN\delta \leq N^{-\frac{5}{4}}
\end{split}
\end{equation*}
Applying this as well as inequality \ref{eq:FAMEShift:temp6} to the equation \ref{eq:FAMEShift:temp00}, we have 

\begin{equation*}
\begin{split}
    \loss_N &= \sum_{\ell=1}^{L} \sum_{k=1}^{K} \sum_{j=1}^{L} \pmb p_{\ell,k}^2 \Exp[\truemu_{\ell,k,j} -\Smu_{\ell,k,j}  ]^2\\
& = \sum_{\ell=1}^{L} \sum_{k=1}^{K} \sum_{j=1}^{L} \pmb p_{\ell,k}^2 \Exp[(\truemu_{\ell,k,j} - \frac{1}{\SN_{\ell,k}}\sum_{t=1}^{\SN_{\ell,k}} \mathbbm{1}_{\sample_{\ell,k,t}=j} )^2 \mathbbm{1}_{ A}] \\
&+  \sum_{\ell=1}^{L} \sum_{k=1}^{K} \sum_{j=1}^{L} \pmb p_{\ell,k}^2 \Exp[(\truemu_{\ell,k,j} - \frac{1}{\SN_{\ell,k}}\sum_{t=1}^{\SN_{\ell,k}} \mathbbm{1}_{\sample_{\ell,k,t}=j} )^2\mathbbm{1}_{A^C}]\\  
\leq & N^{-1} (\sum_{\ell,k} \pmb p_{\ell,k} \usc_{\ell,k})^2 + O(N^{-\frac{5}{4}} \log^{\frac{1}{4}} N )  + N^{-\frac{5}{4}}
\end{split}
\end{equation*}
Note that the loss of the optimal allocation is simply $\loss^*_N = N^{-1} (\sum_{\ell,k} \pmb p_{\ell,k} \usc_{\ell,k})^2$.
The above inequality is simply 
\begin{equation*}
\begin{split}
    \loss_N - \loss_N^* \leq  
 O(N^{-\frac{5}{4}} \log^{\frac{1}{4}} N )
\end{split}
\end{equation*}
which completes the proof.
\end{proof}

\section{Experimental Details}\label{sec:FAMEShift:experimentdetails}

\paragraph{Experimental Setups.} 
All experiments were run on a machine with 2 E5-2690 v4 CPUs, 160 GB RAM and 500 GB disk with Ubuntu 18.04 LTS as the OS. Our code is implemented and tested in python 3.7. 
All experimental results were averaged over 1500 runs, except the case study.
Overall the experiments took about two month, including debugging and evaluation on all datasets.
Running \systemnameAPIShift{} once to draw a few thousand samples typically only takes a few seconds. 
Our implementation is purely in Python for demonstration purposes, and more code optimization (e.g., using cython or multi-thread) can generate a much faster implementation.

\paragraph{ML APIs and Dataset Statistics.}
We focus on three common classification tasks, namely, sentiment analysis, facial emotion recognition, and spoken command recognition. 
For each of the tasks, we evaluated three APIs' performance in spring 2020 and spring 2021, respectively, for four datasets. 
The details of datasets and ML APIs are summarized in Table \ref{tab:APIShift:DatasetStats} and Table \ref{tab:MLservice} respectively. 
Now we give more context of the datasets.

\begin{table}[t]
  \centering
  \small
  \caption{\small Dataset statistics.}
    \begin{tabular}{|c||c|c|c||c|c||c|}
    \hline
    Dataset & Size & \# Classes & Dataset & Size & \# Classes & Tasks \bigstrut\\
    \hline
     \hline
    FER+   & 6358  & 7     & RAFDB \cite{Dataset_FAFDB_li2017reliable} & 15339 & 7     & \multirow{2}[4]{*}{\textit{FER}} \bigstrut\\
\cline{1-6}    EXPW  & 31510 & 7     & AFFECTNET & 87401 & 7     & \bigstrut \\
    \hline
    YELP  & 20000  & 2     & SHOP  & 62774 & 2     & \multirow{2}[4]{*}{\textit{SA}} \bigstrut \\
\cline{1-6}    IMDB & 25000 & 2     & WAIMAI & 11987 & 2     &  \bigstrut\\
    \hline
    DIGIT  & 2000  & 10    & AUDIOMNIST  & 30000 & 10      & \multirow{2}[4]{*}{\textit{STT}} \bigstrut \\
\cline{1-6}  FLUENT  & 30043  & 31   &   COMMAND  & 64727 & 31   & \bigstrut\\
    \hline
    \end{tabular}%
  \label{tab:APIShift:DatasetStats}%
\end{table}%

\begin{table}[t]
	\centering
	\small
	\caption{\small ML services used for each task. Price unit: USD/10,000 queries. We consider three tasks, sentiment analysis (SA), facial emotion recognition (FER), } an spoken command recognition (SCR).
	\begin{tabular}{|c||c|c|c|c|c|c|}
		\hline
		Tasks  & ML service & Price & ML service & Price & ML service & Price \bigstrut \\
		\hline
		\hline
		\textit{SA} & Google NLP \cite{GoNLPAPI}& 2.5     & AMZN Comp \cite{AmazonAPI} & 0.75& Baidu NLP \cite{BaiduAPI} & 3.5  \bigstrut\\
		\hline
		\textit{FER} & Google Vision \cite{GoogleAPI} & 15    & MS  Face \cite{MicrosoftAPI}& 10    & Face++ \cite{FacePPAPI}& 5 \bigstrut\\
		\hline
		\textit{SCR} & Google Speech \cite{GoogleSpeechAPI}& 60    & MS Speech  \cite{MicrosoftSpeechAPI}& 41    & IBM Speech \cite{IBMAPI}& 25 \bigstrut\\
		\hline
	\end{tabular}%
	\label{tab:MLservice}%
\end{table}%

For sentiment analysis, we use four datasets, YELP, IMDB, SHOP, and WAIMAI. YELP and IMDB are both English text datasets. YELP \cite{Dataset_SEntiment_YELP} is generated by drawn twenty thousand samples from the large YELP review challenge dataset. Each original review is labeled by rating
in  \{1,2,3,4,5\}. We generate the binary label by transforming rating 1 and 2 into negative, and rating 4 and 5 into positive. Ten thousand positive reviews and ten thousand negative reviews are then randomly drawn, respectively. 
IMDB \cite{Dataset_SEntiment_IMDB_ACL_HLT2011} is a polarized sentiment analysis dataset with provided training and testing partitions.
We use its testing partition which has twenty-five thousand text paragraphs. SHOP~\cite{Dataset_SENTIMENT_SHOP} and WAIMAI~\cite{Dataset_SENTIMENT_WAIMAI} are two Chinese text datasets. SHOP contains polarized labels for reviews for various purchases including fruits,
hotels, computers. 
WAIMAI is a dataset for polarized delivery reviews. 
Both SHOP and WAIMAI are publicly available without licence requirements.
There is a dataset user agreement for YELP dataset, which disallows commercial usage of the datasets but encourages academic study. 
Same thing applies to the IMDB dataset.

For facial emotion recognition, we use four datasets: FER+, RAFDB, EXPW, and AFNET. 
All the datasets are annotated by the standard seven basic emotions, i.e., \{anger, disgust, fear, happy, sad, surprise, neutral\}. The images in FER+ \cite{Dataset_FER2013} are from the ICML 2013 Workshop on Challenges in Representation.
We use the provided testing portion in FER+. RAFDB \cite{Dataset_FAFDB_li2017reliable} and AFFECTNET \cite{Dataset_AFFECTNET_MollahosseiniHM19} were annotated with both basic emotions and fine-grained labels. In this paper, we only use  basic emotions since commercial APIs cannot
work for compound emotions. EXPW \cite{Dataset_EXPW_SOCIALRELATION_ICCV2015} contains raw images and bound boxes pointing out the face locations. Here we use the true bounding box associated with the dataset to create
aligned faces first, and only pick the images that are faces with confidence larger than 0.6.
We cotnacted the creators of RAFDB and AFNET to obtain the data access for academic purposes. 
FER+ and EXPW are both publicly available online without consent or licence requirements.

For spoken command recognition, we use DIGIT, AMNIST, CMD, 
and FLUENT. DIGIT \cite{Dataset_Speech_DIGIT} and AMNIST~\cite{Dataset_Speech_AudioMNIST_becker2018interpreting} are spoken digit datasets, where the label is is a spoken digit (i.e., 0-9). The sampling rate is 8 kHz for DIGIT and 48 kHz for AMNIST. Each sample in CMD~\cite{Dataset_Speech_GoogleCommand} is a spoken command such as “go”, “left”, “right”, “up”, and “down”, with a sampling rate of 16 kHz. 
In total, there are 30 commands and a few white noise utterances. FLUENT~\cite{Dataset_Speech_Fluent_LugoschRITB19} is another recently developed dataset for speech command. 
The commands in FLUENT are typically a phrase (e.g., “turn on the light” or “turn down the music”). There are in total 248 possible phrases, which are mapped to 31 unique labels. 
The sampling rate is also 16 kHz.
All those datasets are freely available online for academic purposes. 

Some of the datasets may contain personal information. For example, the human faces contained in the facial emotion recognition dataset may be deemed as personal information.
On the other hand, our study focuses on whether there is a performance change on the dataset, and does not use or disclose any personal information. 
\begin{figure}[t]
	\centering
	\begin{subfigure}[Amazon IMDB 2020]{\includegraphics[width=0.32\linewidth]{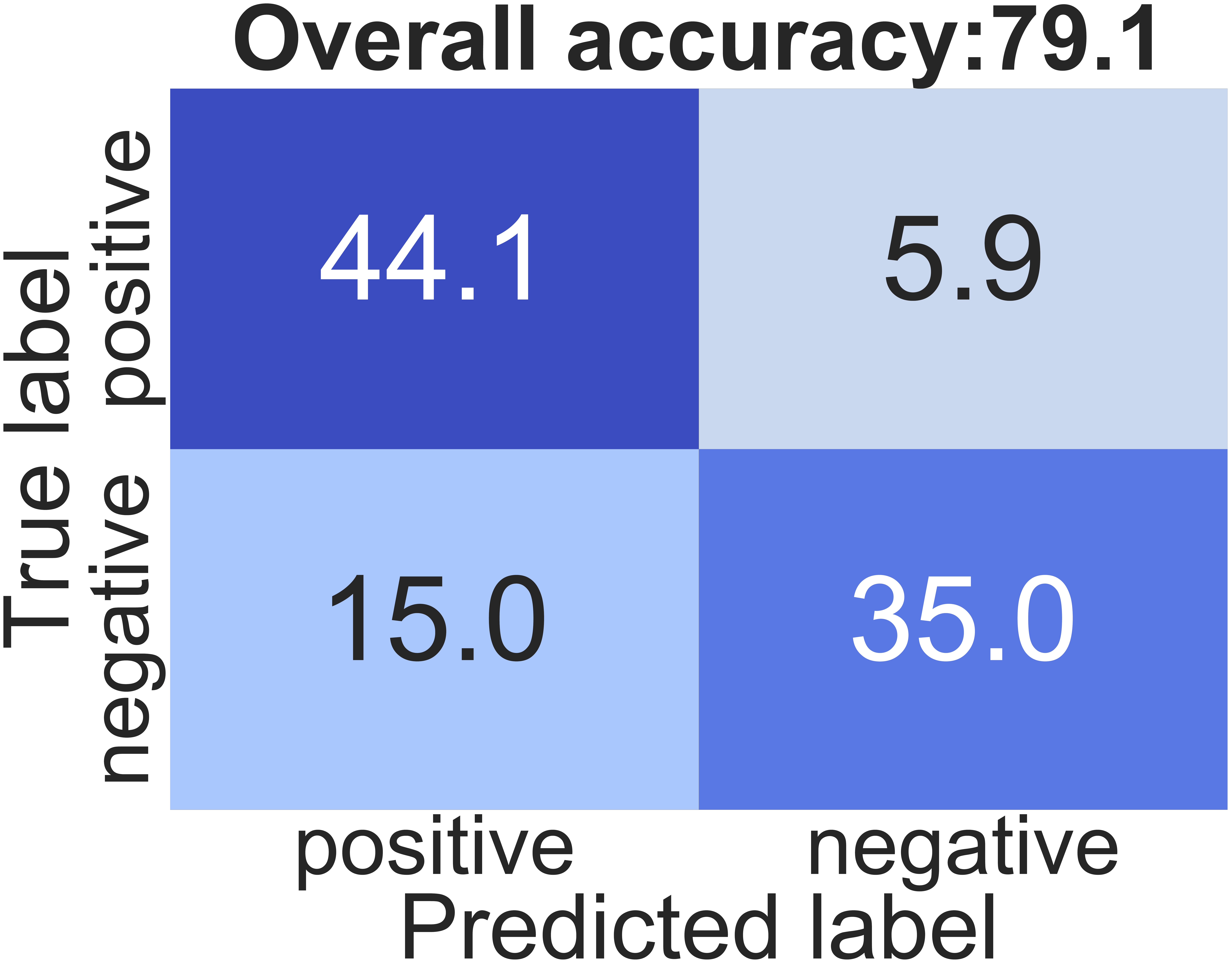}}
    \end{subfigure}
	\begin{subfigure}[Amazon IMDB 2021]{\includegraphics[width=0.32\linewidth]{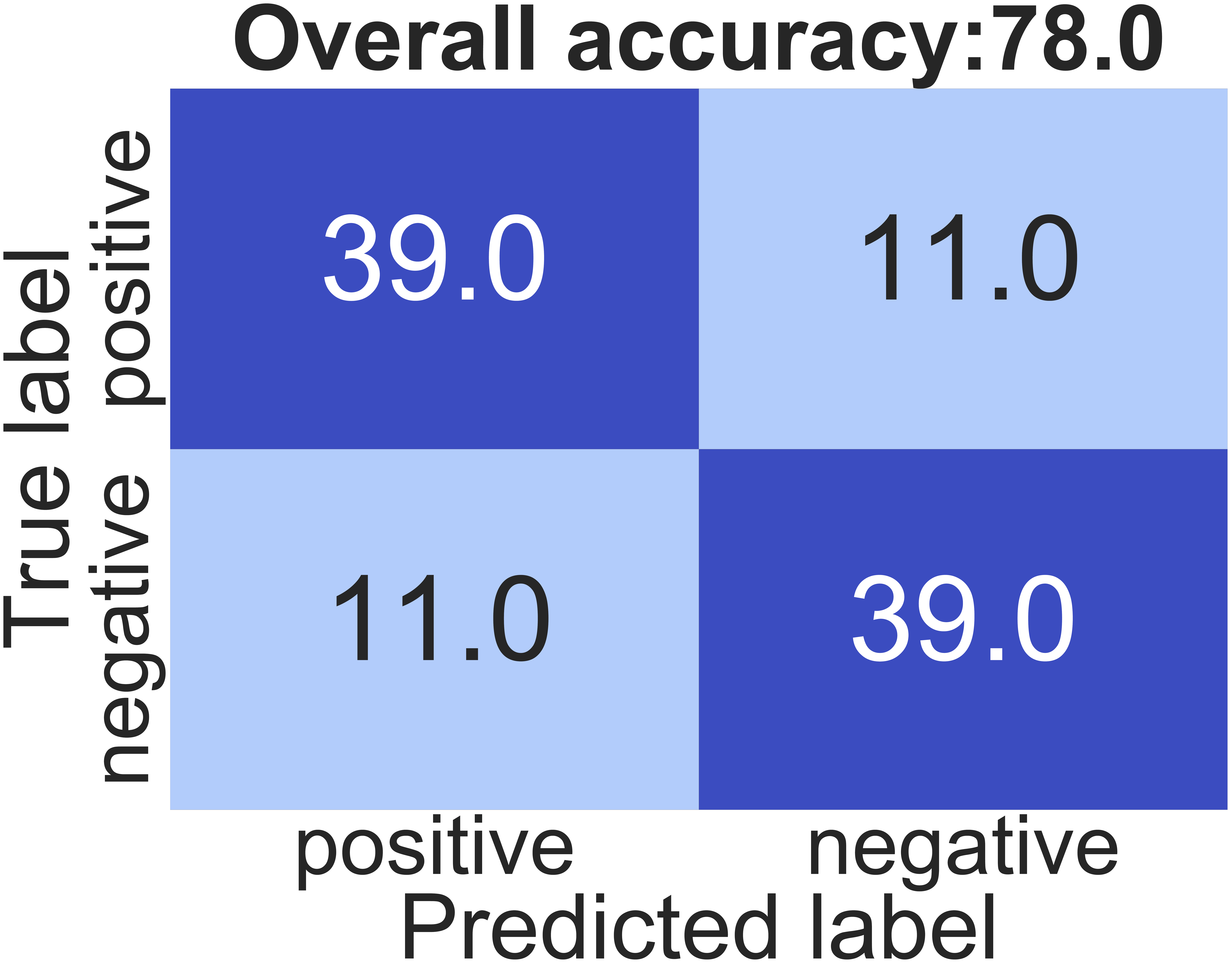}}
    \end{subfigure}
	\begin{subfigure}[Amazon IMDB 2021]{\includegraphics[width=0.32\linewidth]{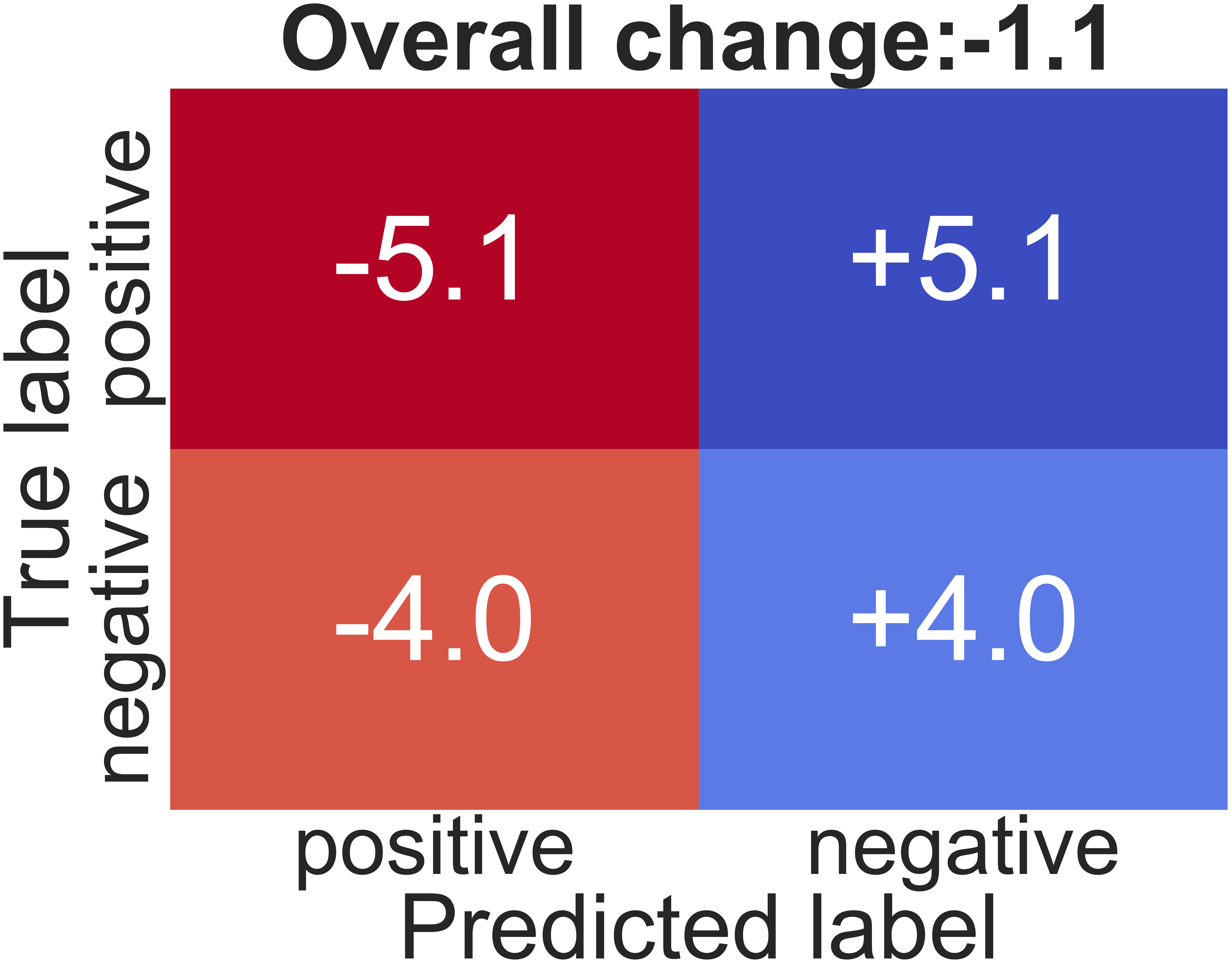}}
    \end{subfigure}
    
	\begin{subfigure}[Microsoft FER+ 2020]{\includegraphics[width=0.32\linewidth]{Figure/figures/Shifts/FERPLUSMicrosoft2020_TrueCM.pdf}}
    \end{subfigure}
	\begin{subfigure}[Microsoft FER+ 2021]{\includegraphics[width=0.32\linewidth]{Figure/figures/Shifts/FERPLUSMicrosoft2021_TrueCM.pdf}}
    \end{subfigure}
	\begin{subfigure}[Microsoft FER+ API shift]{\includegraphics[width=0.32\linewidth]{Figure/figures/Shifts/FERPLUSMicrosoft2020_APIShift.pdf}}
    \end{subfigure}
    
	\begin{subfigure}[Google DIGIT 2020]{\includegraphics[width=0.32\linewidth]{Figure/figures/Shifts/DIGITGoogle2020_TrueCM.pdf}}
    \end{subfigure}
	\begin{subfigure}[Google DIGIT 2021]{\includegraphics[width=0.32\linewidth]{Figure/figures/Shifts/DIGITGoogle2021_TrueCM.pdf}}
    \end{subfigure}
	\begin{subfigure}[Google DIGIT 2021]{\includegraphics[width=0.32\linewidth]{Figure/figures/Shifts/DIGITGoogle2020_APIShift.pdf}}
    \end{subfigure}

  \caption{Confusion matrices of a few APIs in spring 2020/2021, along with their API shifts.}\label{fig:FAMEShift:APIShiftDetail}
\end{figure}

For sentiment analysis, we use the Google NLP API~\cite{GoNLPAPI}, Amazon Comprehend API~\cite{AmazonAPI}, and the Baidu NLP API~\cite{BaiduAPI}.
For facial emotion recognition, we use Google Vision API~\cite{GoogleAPI}, Microsoft Face API~\cite{MicrosoftAPI}, and the Face++ API~\cite{FacePPAPI}.
For spoken command recognition, we adopt Google speech API~\cite{GoogleSpeechAPI}, Microsoft Speech API~\cite{MicrosoftSpeechAPI}, and IBM speech API~\cite{IBMAPI}.

\paragraph{Details of observed ML API Shifts. }

Now we present a few more observed ML API shifts, as shown in Figure \ref{fig:FAMEShift:APIShiftDetail}. 
One observation is that individual entry's change in the API shift can be larger than the overall accuracy's. For example, as shown in Figure \ref{fig:FAMEShift:APIShiftDetail} (c), the overall accuracy change is about -1.1\% for Amazon on IDMB, but the performance drop for positive texts is as large as 5\%. 
This indicates the importance of using fine-grained confusion matrix difference to measure API shifts.
In addition, when the overall accuracy increases, it is possible that the accuracy for each label has been improved. This can be easily verified by Figure \ref{fig:FAMEShift:APIShiftDetail} (d-f).
On the other hand, as shown in Figure \ref{fig:FAMEShift:APIShiftDetail} (g-i), Google API's large accuracy improvement (24\%) is mostly because it is able to correctly predict many samples that were previously deemed as empty. One possible explanation is that Google API internally uses a higher threshold to generate a recognition.  
When the number of label increases, it might become hard to manually check the API shifts. 
For those cases, an anomaly detector can be applied to quickly identify the most surprising components in the API shifts.
\paragraph{Partition size's effects on \systemnameAPIShift{}.}

\begin{figure}[t]
	\centering
	\begin{subfigure}[Amazon YELP]{\includegraphics[width=0.24\linewidth]{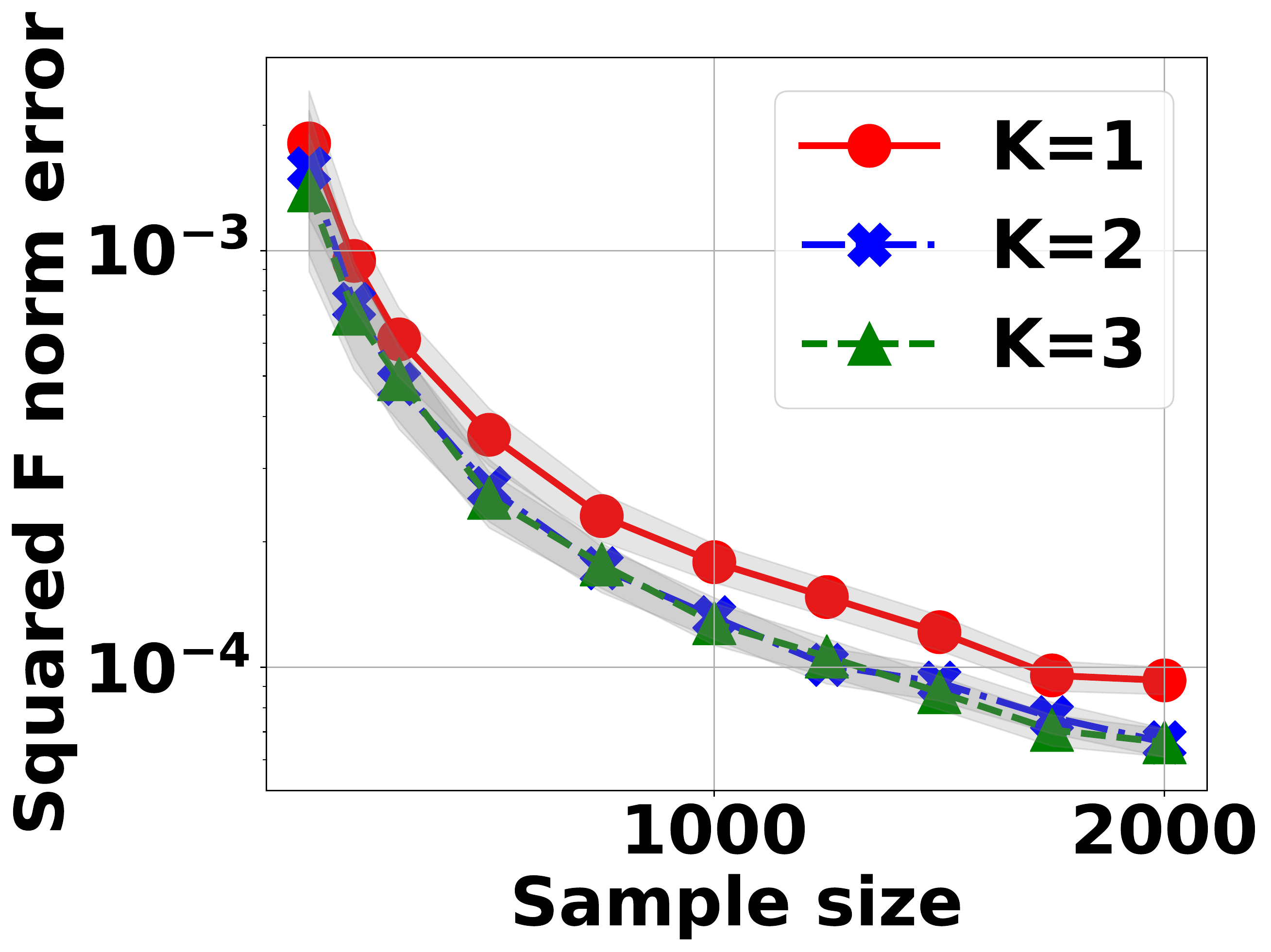}}
    \end{subfigure}
	\begin{subfigure}[Amazon SHOP]{\includegraphics[width=0.23\linewidth]{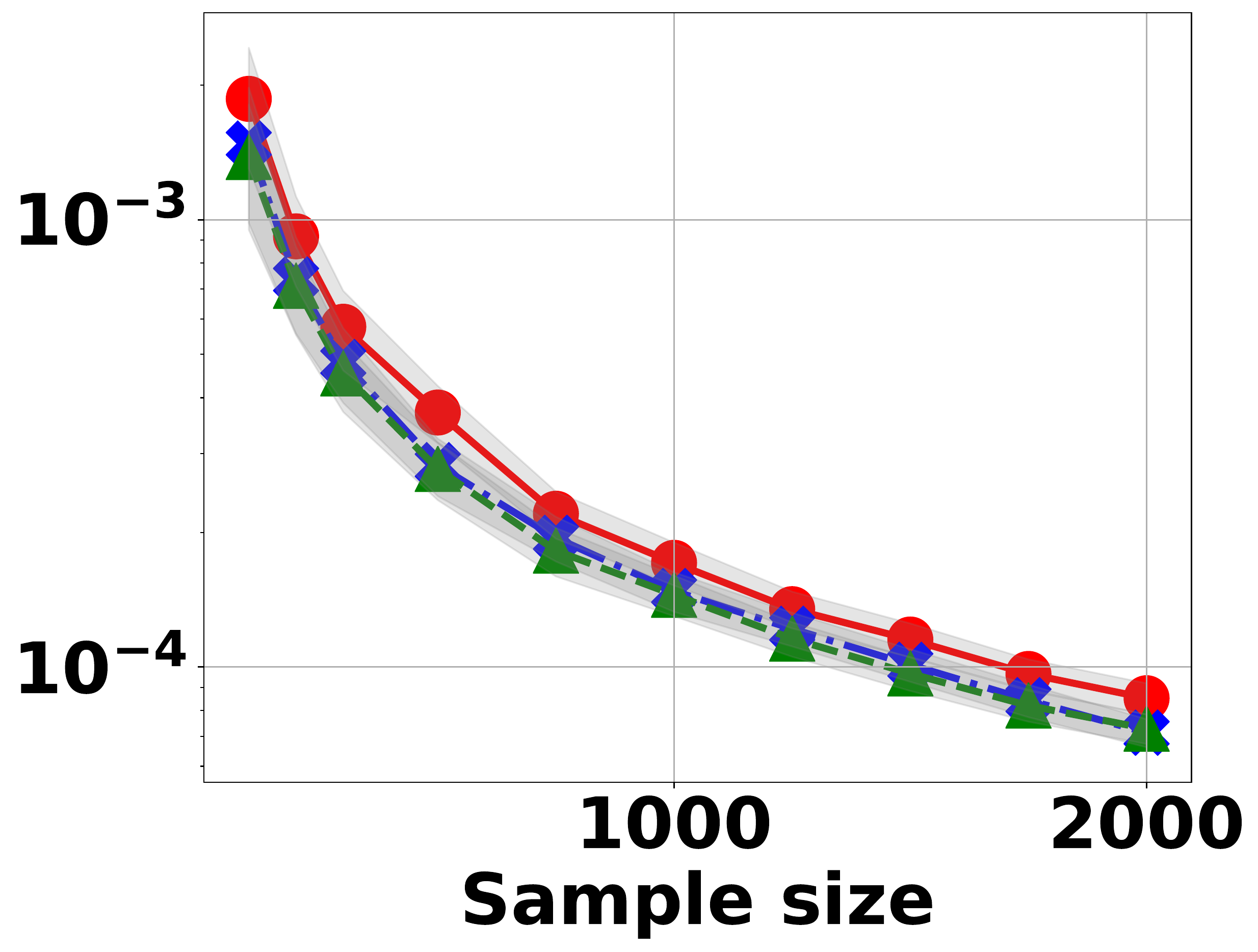}}
    \end{subfigure}
	\begin{subfigure}[Amazon IMDB]{\includegraphics[width=0.23\linewidth]{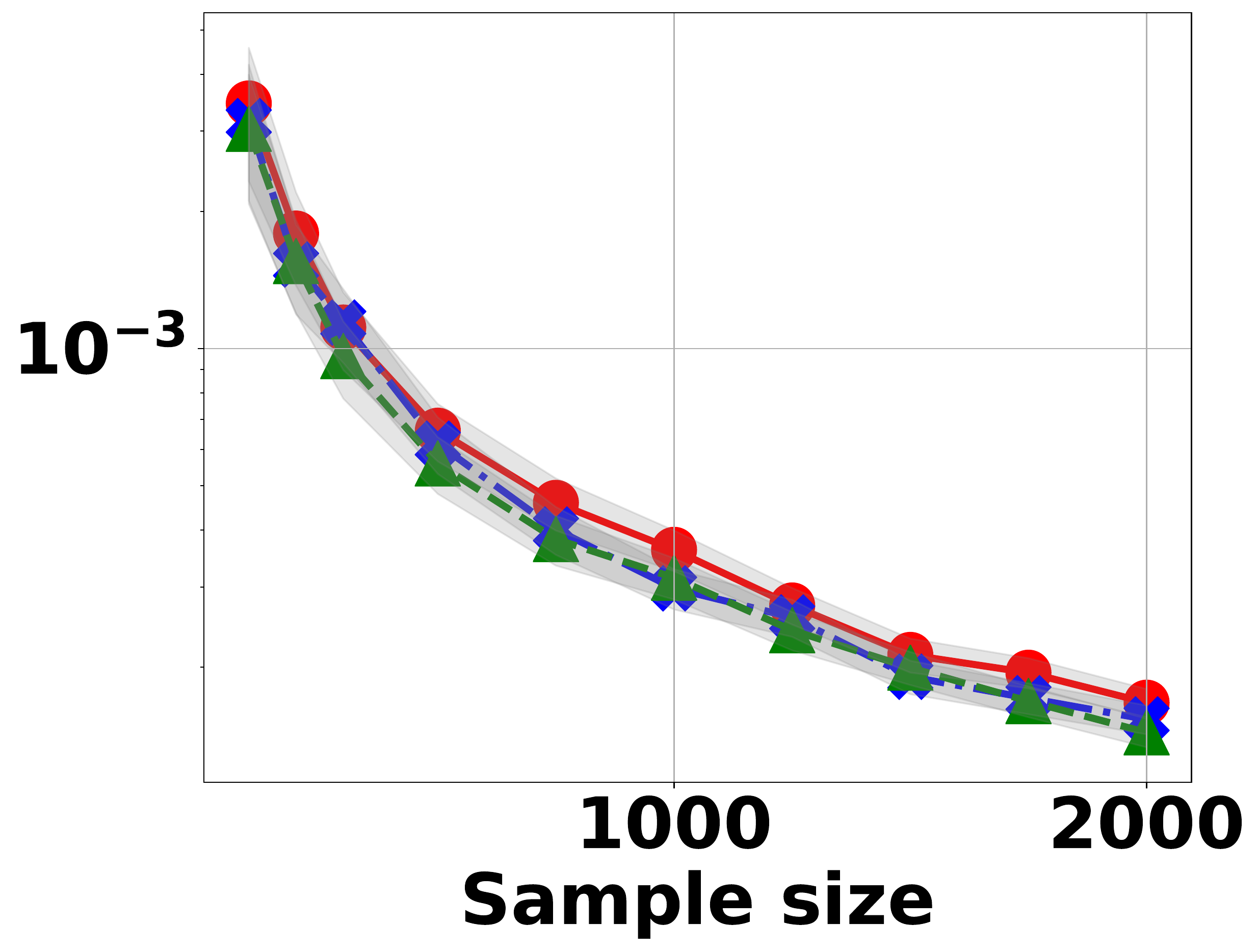}}
    \end{subfigure}
	\begin{subfigure}[Amazon WAIMAI]{\includegraphics[width=0.23\linewidth]{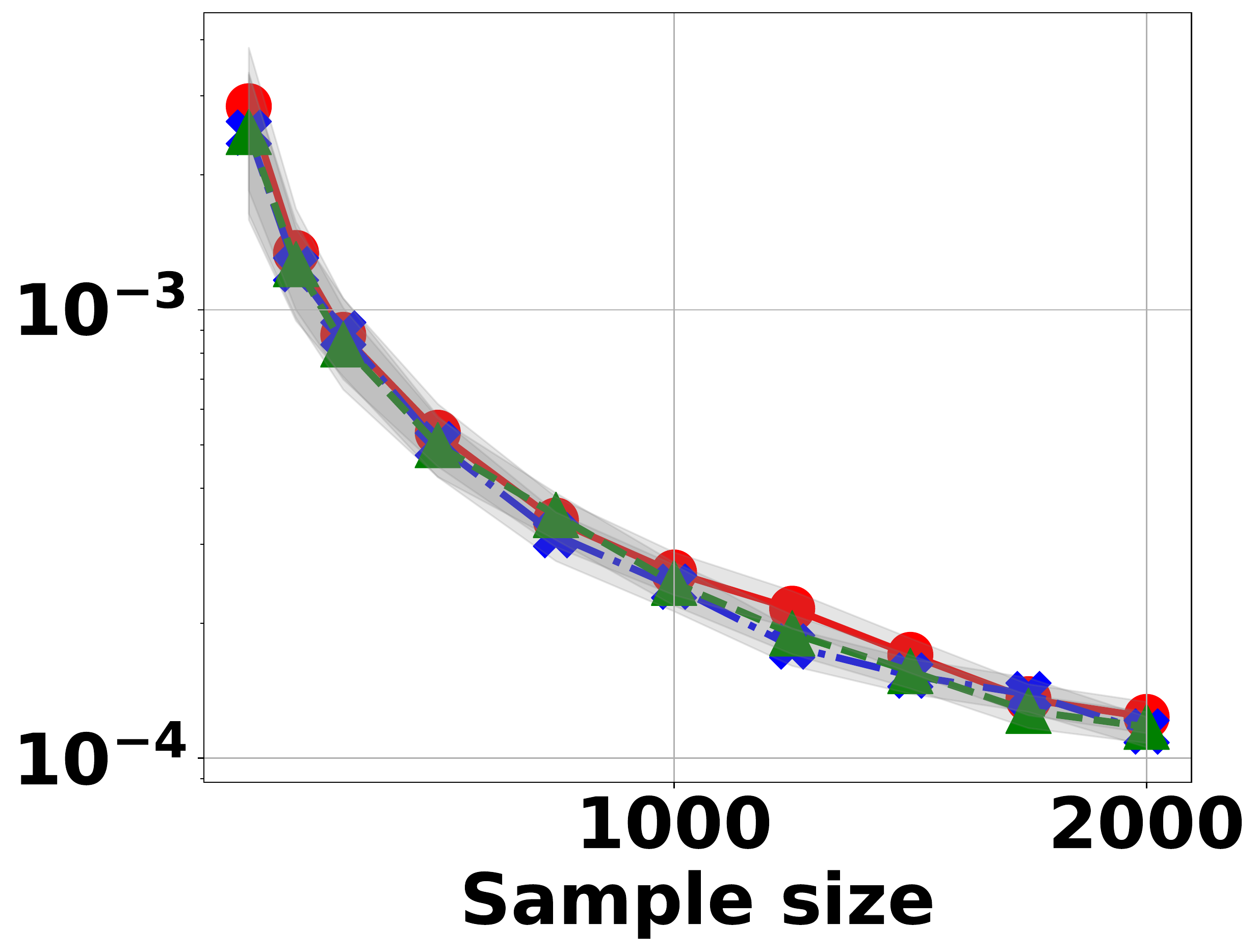}}
    \end{subfigure}

	\begin{subfigure}[Microsoft FER+]{\includegraphics[width=0.23\linewidth]{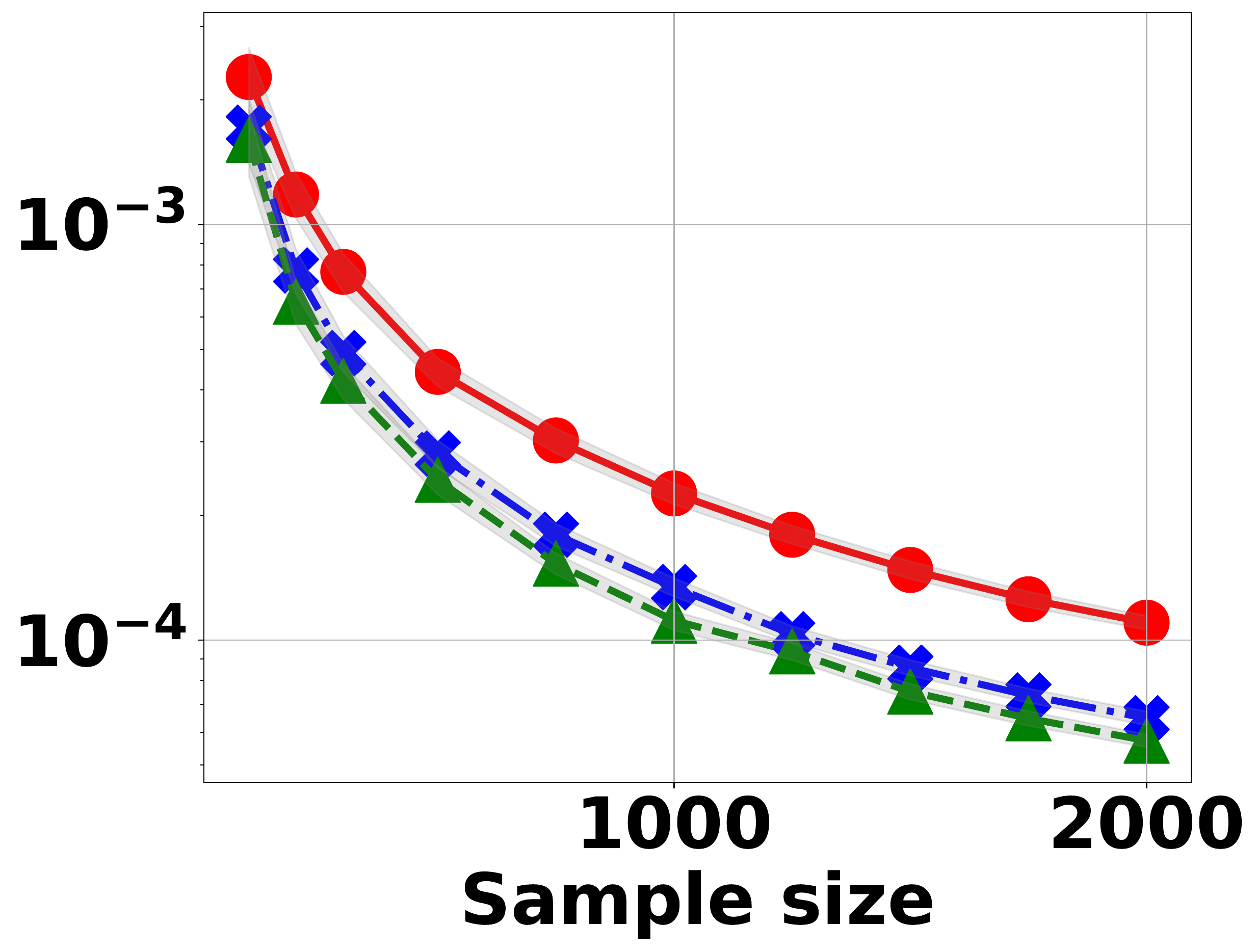}}
    \end{subfigure}
	\begin{subfigure}[Google EXPW+]{\includegraphics[width=0.23\linewidth]{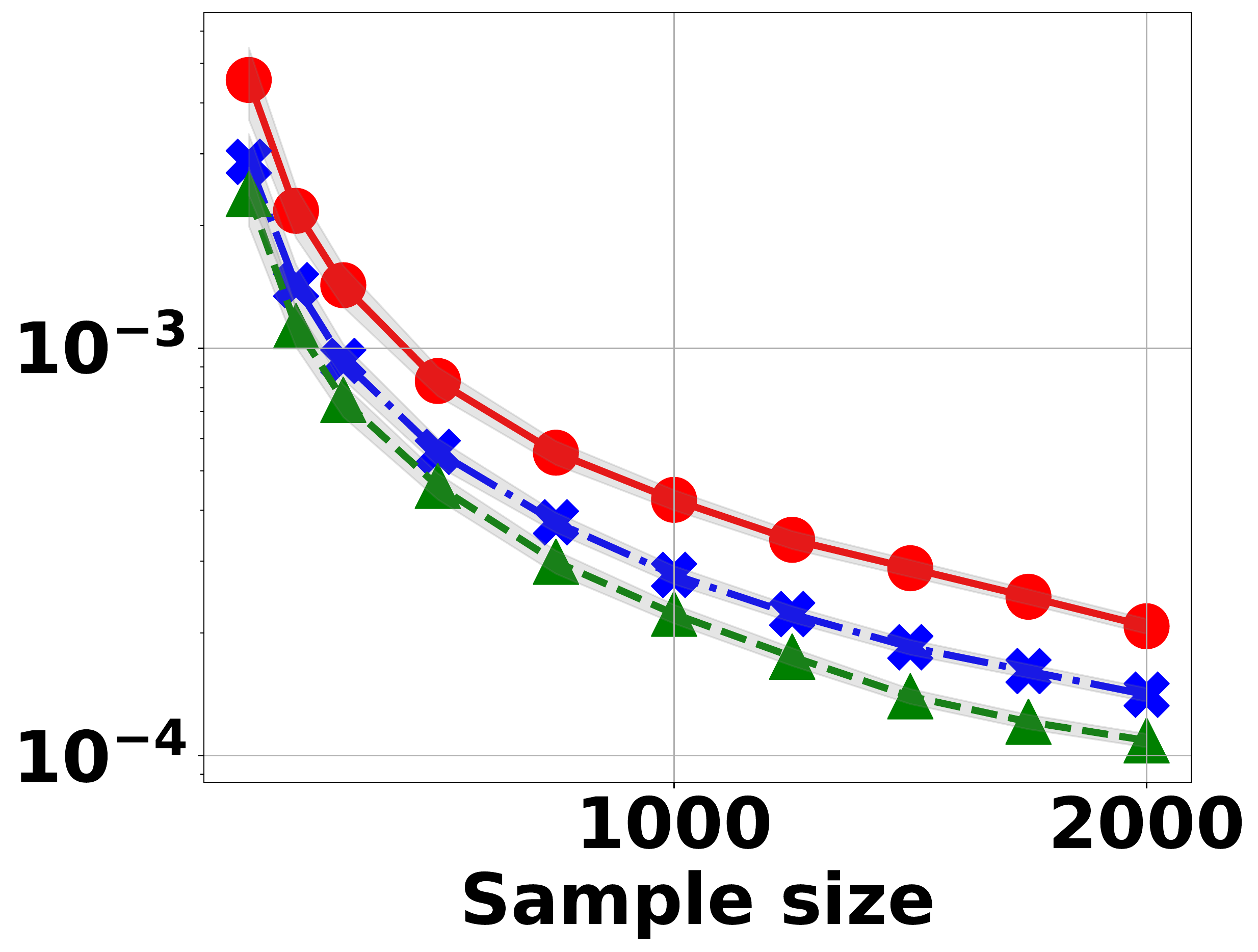}}
    \end{subfigure}
	\begin{subfigure}[IBM DIGIT ]{\includegraphics[width=0.23\linewidth]{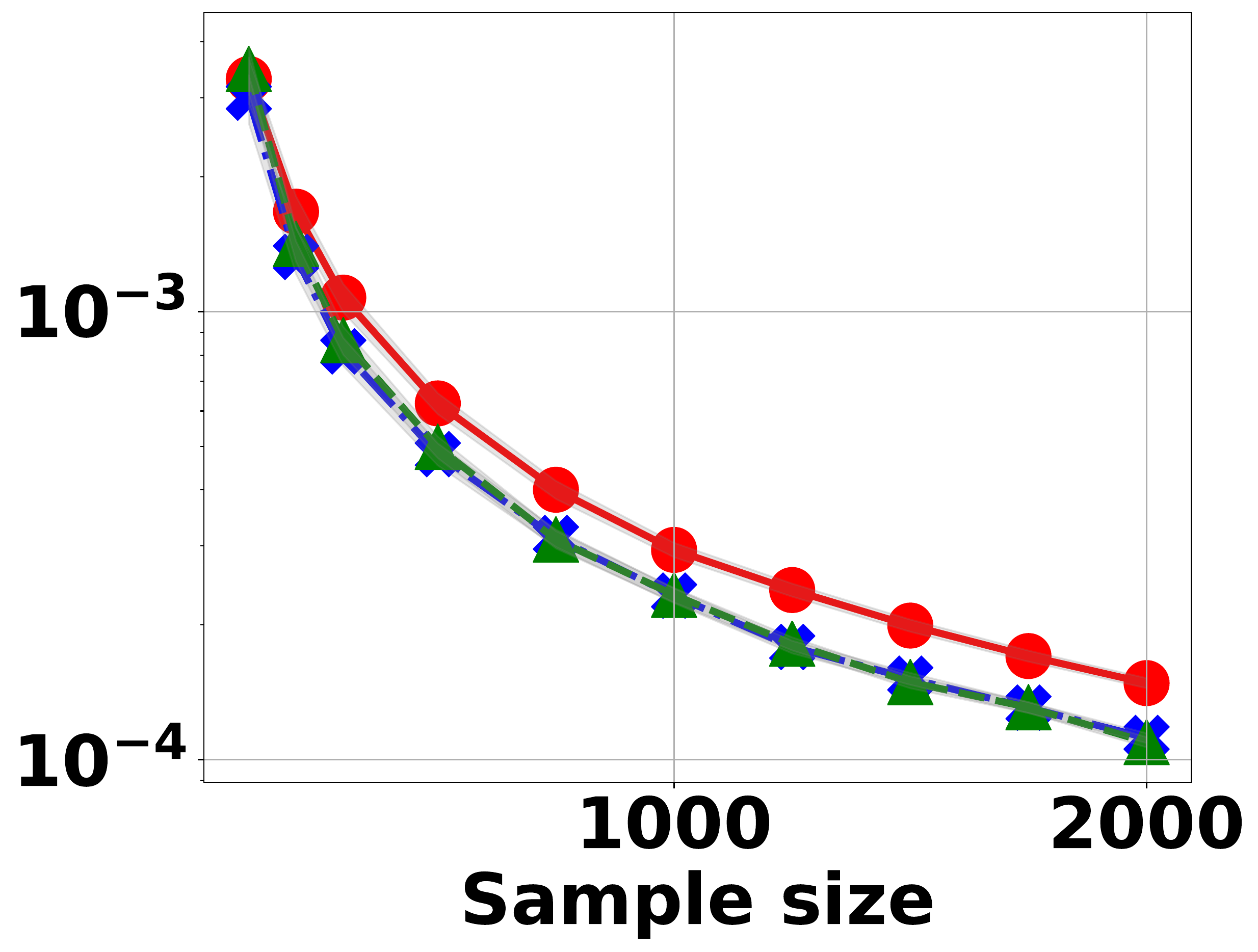}}
    \end{subfigure}
	\begin{subfigure}[IBM AMNIST]{\includegraphics[width=0.23\linewidth]{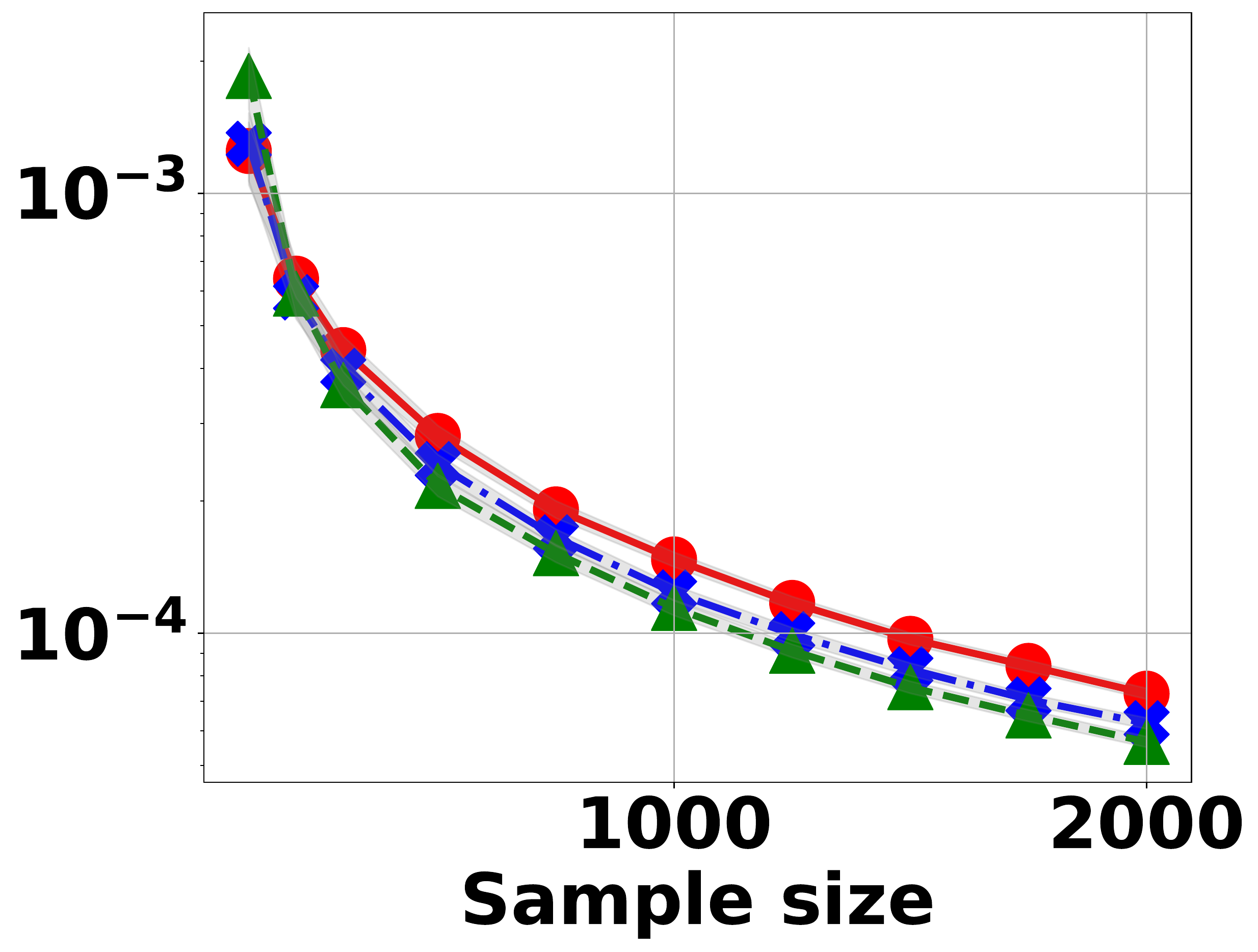}}
    \end{subfigure}    
  
	\begin{subfigure}[Google DIGIT ]{\includegraphics[width=0.23\linewidth]{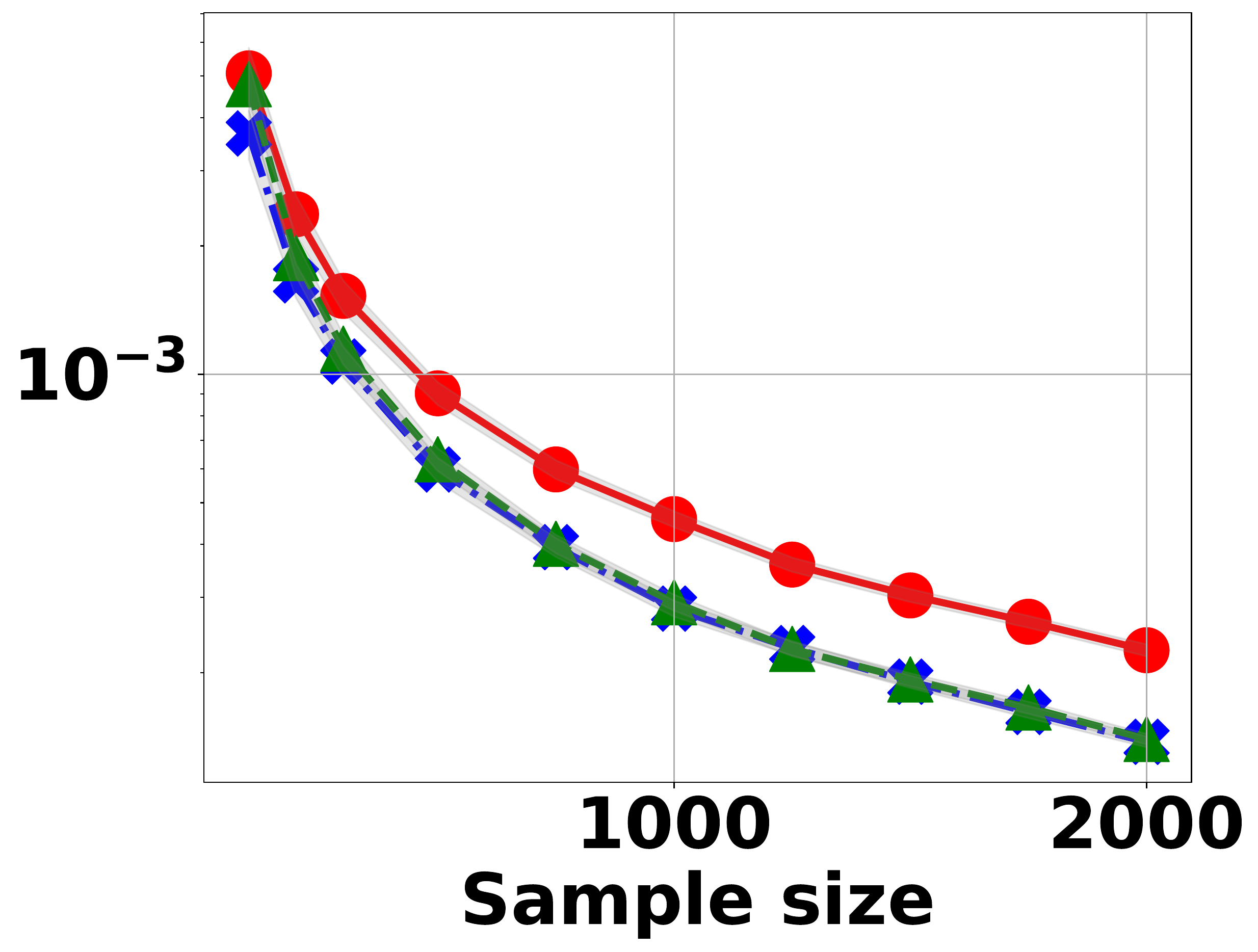}}
    \end{subfigure}
	\begin{subfigure}[Google AMINST ]{\includegraphics[width=0.23\linewidth]{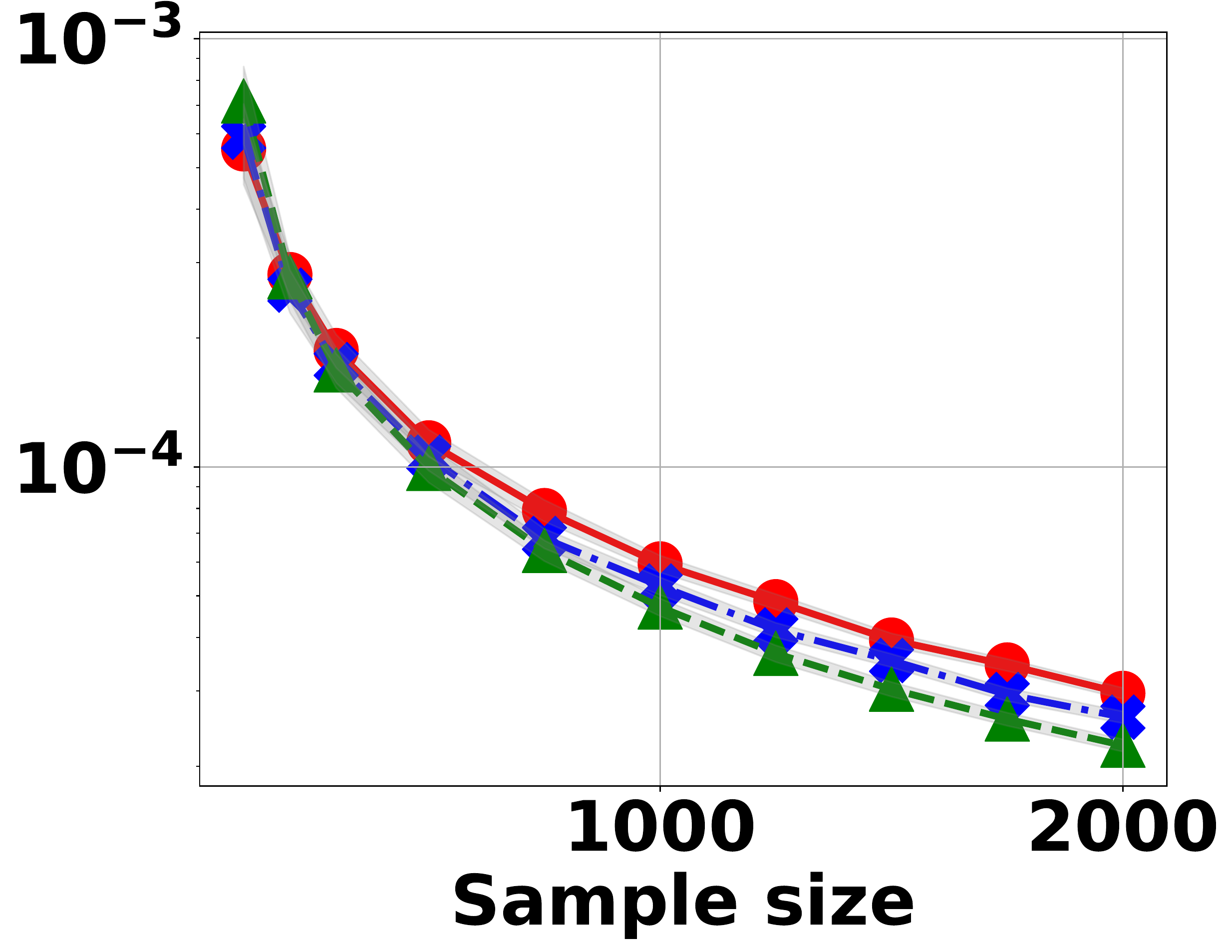}}
    \end{subfigure}
	\begin{subfigure}[Google CMD ]{\includegraphics[width=0.23\linewidth]{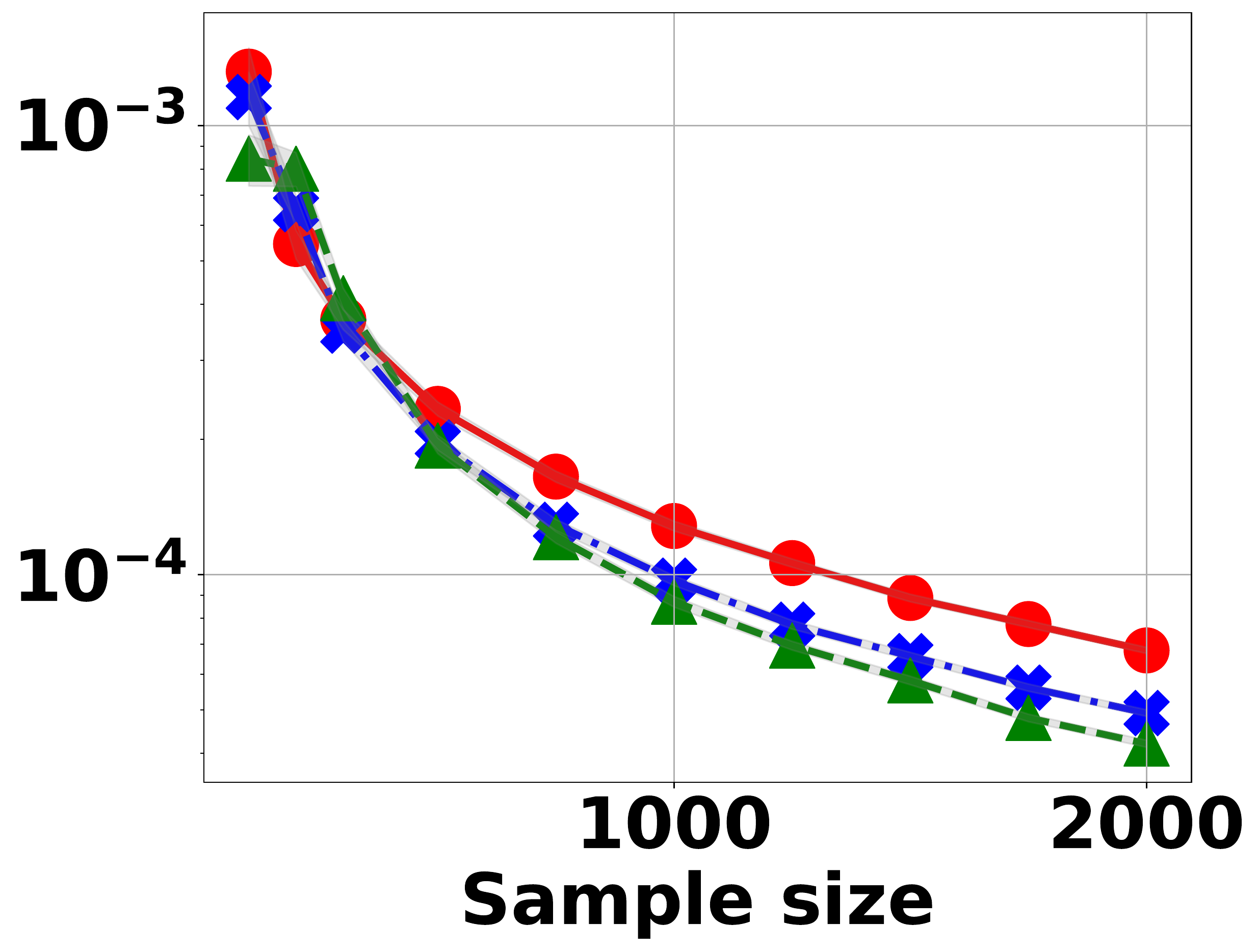}}
    \end{subfigure}
	\begin{subfigure}[Microsoft DIGIT ]{\includegraphics[width=0.23\linewidth]{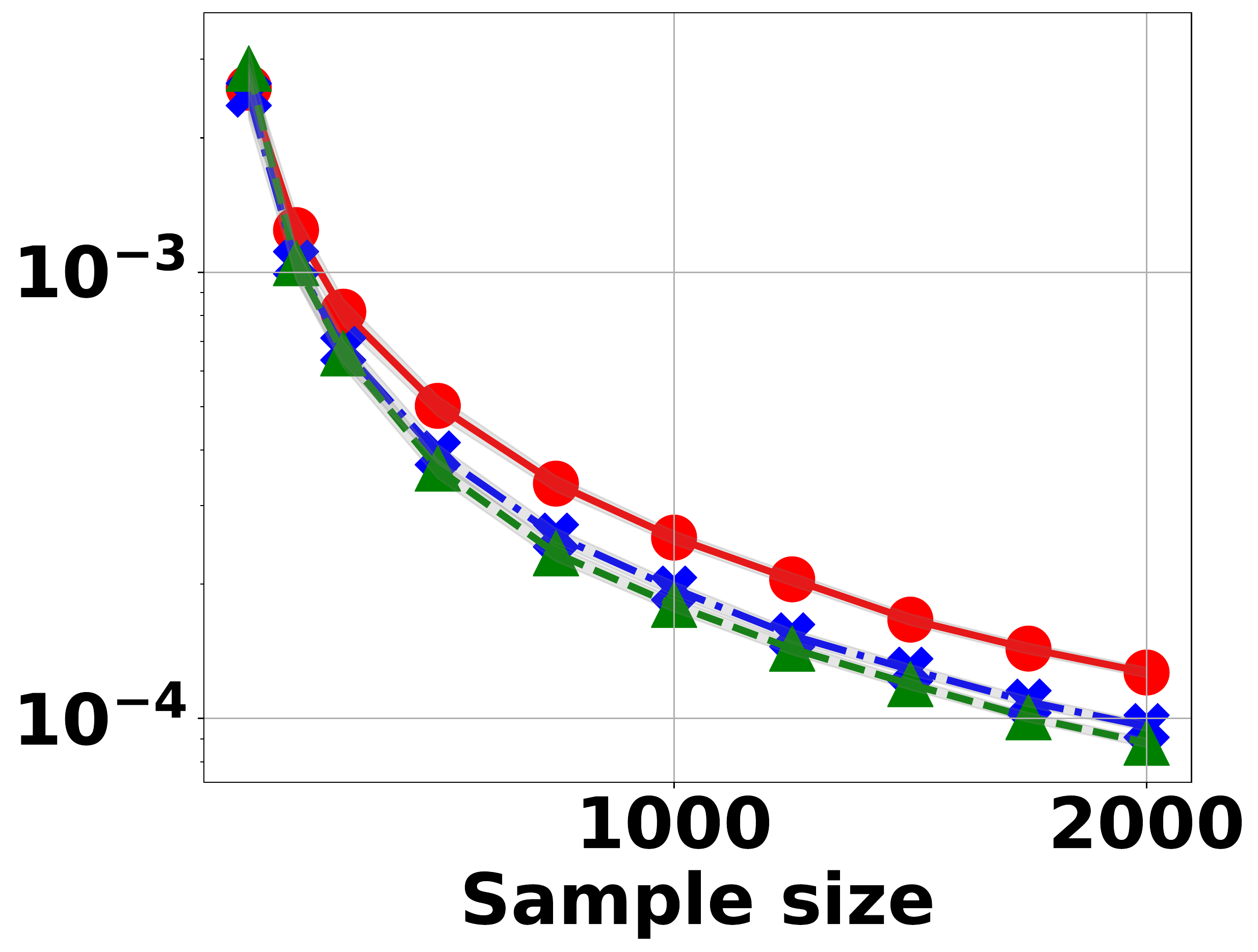}}
    \end{subfigure}
    
  \caption{Effects of partition parameter $K$. The total number of partitions is $LK$, and thus Larger $K$ implies more partitions. Generally, across 12 cases where API shifts are identified, larger number of partitions usually leads to smaller estimation error for large samples. In practice, we observe that $K=3$ is enough to reach good error rate.}\label{fig:FAMEShift:keffects}
\end{figure}

Finally we study how the partition number affects the performance of \systemnameAPIShift{}, as shown in Figure \ref{fig:FAMEShift:keffects}.
Across all API shifts we estimated, we note that larger number of partitions leads to a smaller overall Frobenious norm in general. This is expected, as larger $K$ effectively introduces more parameters to estimate and thus is more powerful. 
The trade-off is that the computational cost increases, and more samples are needed for initial estimation. 
Interestingly, as $K$ becomes large, the relative error reduction improvement becomes small. 
This is probably because there is no strong uncertainty difference within small partitions.
In practice, we found that $K=3$ already gives a small enough error reduction.


\begin{thebibliography}{10}

\bibitem{MLasS_MarketInfo}
{ Machine Learning as a Service Market Report }.
\newblock
  \url{https://www.mordorintelligence.com/industry-reports/global-machine-learning-as-a-service-mlaas-market}.

\bibitem{AmazonAPI}
{Amazon Comprehend {API}}.
\newblock \url{https://aws.amazon.com/comprehend}.
\newblock [Accessed March-2020 and March-2021].

\bibitem{BaiduAPI}
{Baidu {API}}.
\newblock \url{https://ai.baidu.com/}.
\newblock [Accessed March-2020 and March-2021].

\bibitem{Dataset_Speech_DIGIT}
{DIGIT dataset}, \url{https://github.com/Jakobovski/free-spoken-digit-dataset}.

\bibitem{FacePPAPI}
{Face++ Emotion API}.
\newblock \url{https://www.faceplusplus.com/emotion-recognition/}.
\newblock [Accessed March-2020 and March-2021].

\bibitem{GoNLPAPI}
{Google NLP {API}}.
\newblock \url{https://cloud.google.com/natural-language}.
\newblock [Accessed March-2020 and March-2021].

\bibitem{GoogleSpeechAPI}
{Google Speech {API}}.
\newblock \url{https://cloud.google.com/speech-to-text}.
\newblock [Accessed March-2020 and March-2021].

\bibitem{GoogleAPI}
{Google Vision {API}}.
\newblock \url{https://cloud.google.com/vision}.
\newblock [Accessed March-2020 and March-2021].

\bibitem{IBMAPI}
{{IBM} Speech {API}}.
\newblock \url{https://cloud.ibm.com/apidocs/speech-to-text}.
\newblock [Accessed March-2020 and March-2021].

\bibitem{MicrosoftAPI}
{Microsoft computer vision {API}}.
\newblock
  \url{https://azure.microsoft.com/en-us/services/cognitive-services/computer-vision}.
\newblock [Accessed March-2020 and March-2021].

\bibitem{MicrosoftSpeechAPI}
{Microsoft speech {API}}.
\newblock
  \url{https://azure.microsoft.com/en-us/services/cognitive-services/speech-to-text}.
\newblock [Accessed March-2020 and March-2021].

\bibitem{Dataset_SENTIMENT_SHOP}
{SHOP} dataset,
  \url{https://github.com/SophonPlus/ChineseNlpCorpus/tree/master/datasets/online_shopping_10_cats}.

\bibitem{Dataset_SENTIMENT_WAIMAI}
{WAIMAI} dataset,
  \url{https://github.com/SophonPlus/ChineseNlpCorpus/tree/master/datasets/waimai_10k}.

\bibitem{Dataset_SEntiment_YELP}
{YELP} dataset, \url{https://www.kaggle.com/yelp-dataset/yelp-dataset}.

\bibitem{measureandprob}
K.B. Athreya and S.N. Lahiri.
\newblock {\em Measure Theory and Probability Theory}.
\newblock Springer, 2006.

\bibitem{Learningunderlabelshift2019}
Kamyar Azizzadenesheli, Anqi Liu, Fanny Yang, and Animashree Anandkumar.
\newblock Regularized learning for domain adaptation under label shifts.
\newblock In {\em 7th International Conference on Learning Representations,
  {ICLR} 2019, New Orleans, LA, USA, May 6-9, 2019}. OpenReview.net, 2019.

\bibitem{Dataset_Speech_AudioMNIST_becker2018interpreting}
S{\"{o}}ren Becker, Marcel Ackermann, Sebastian Lapuschkin, Klaus{-}Robert
  M{\"{u}}ller, and Wojciech Samek.
\newblock Interpreting and explaining deep neural networks for classification
  of audio signals.
\newblock {\em CoRR}, abs/1807.03418, 2018.

\bibitem{pmlr-v81-buolamwini18a}
Joy Buolamwini and Timnit Gebru.
\newblock Gender shades: Intersectional accuracy disparities in commercial
  gender classification.
\newblock In Sorelle~A. Friedler and Christo Wilson, editors, {\em Conference
  on Fairness, Accountability and Transparency, {FAT} 2018, 23-24 February
  2018, New York, NY, {USA}}, volume~81 of {\em Proceedings of Machine Learning
  Research}, pages 77--91. {PMLR}, 2018.

\bibitem{SamplingMAB2012}
Alexandra Carpentier and R{\'{e}}mi Munos.
\newblock Finite time analysis of stratified sampling for monte carlo.
\newblock In John Shawe{-}Taylor, Richard~S. Zemel, Peter~L. Bartlett, Fernando
  C.~N. Pereira, and Kilian~Q. Weinberger, editors, {\em Advances in Neural
  Information Processing Systems 24: 25th Annual Conference on Neural
  Information Processing Systems 2011. Proceedings of a meeting held 12-14
  December 2011, Granada, Spain}, pages 1278--1286, 2011.

\bibitem{SamplingMAB2015}
Alexandra Carpentier, R{\'{e}}mi Munos, and Andr{\'{a}}s Antos.
\newblock Adaptive strategy for stratified monte carlo sampling.
\newblock {\em J. Mach. Learn. Res.}, 16:2231--2271, 2015.

\bibitem{AQPsampling2007}
Surajit Chaudhuri, Gautam Das, and Vivek~R. Narasayya.
\newblock Optimized stratified sampling for approximate query processing.
\newblock {\em {ACM} Trans. Database Syst.}, 32(2):9, 2007.

\bibitem{FrugalML2020}
Lingjiao Chen, Matei Zaharia, and James~Y. Zou.
\newblock Frugalml: How to use {ML} prediction apis more accurately and
  cheaply.
\newblock In Hugo Larochelle, Marc'Aurelio Ranzato, Raia Hadsell,
  Maria{-}Florina Balcan, and Hsuan{-}Tien Lin, editors, {\em Advances in
  Neural Information Processing Systems 33: Annual Conference on Neural
  Information Processing Systems 2020, NeurIPS 2020, December 6-12, 2020,
  virtual}, 2020.

\bibitem{onlinesample}
Ira~W. Cotton.
\newblock Remark on stably updating mean and standard deviation of data.
\newblock {\em Commun. {ACM}}, 18(8):458, 1975.

\bibitem{Dataset_FER2013}
Ian~J. Goodfellow, Dumitru Erhan, Pierre~Luc Carrier, Aaron~C. Courville, Mehdi
  Mirza, Benjamin Hamner, William Cukierski, Yichuan Tang, David Thaler,
  Dong{-}Hyun Lee, Yingbo Zhou, Chetan Ramaiah, Fangxiang Feng, Ruifan Li,
  Xiaojie Wang, Dimitris Athanasakis, John Shawe{-}Taylor, Maxim Milakov, John
  Park, Radu~Tudor Ionescu, Marius Popescu, Cristian Grozea, James Bergstra,
  Jingjing Xie, Lukasz Romaszko, Bing Xu, Zhang Chuang, and Yoshua Bengio.
\newblock Challenges in representation learning: {A} report on three machine
  learning contests.
\newblock {\em Neural Networks}, 64:59--63, 2015.

\bibitem{RebalancingClassifier2019}
Saihui Hou, Xinyu Pan, Chen~Change Loy, Zilei Wang, and Dahua Lin.
\newblock Learning a unified classifier incrementally via rebalancing.
\newblock In {\em {IEEE} Conference on Computer Vision and Pattern Recognition,
  {CVPR} 2019, Long Beach, CA, USA, June 16-20, 2019}, pages 831--839. Computer
  Vision Foundation / {IEEE}, 2019.

\bibitem{Modelassertion2020}
Daniel Kang, Deepti Raghavan, Peter Bailis, and Matei Zaharia.
\newblock Model assertions for monitoring and improving {ML} models.
\newblock In Inderjit~S. Dhillon, Dimitris~S. Papailiopoulos, and Vivienne Sze,
  editors, {\em Proceedings of Machine Learning and Systems 2020, MLSys 2020,
  Austin, TX, USA, March 2-4, 2020}. mlsys.org, 2020.

\bibitem{DistributionShiftBenchmark2021}
Pang~Wei Koh, Shiori Sagawa, Henrik Marklund, Sang~Michael Xie, Marvin Zhang,
  Akshay Balsubramani, Weihua Hu, Michihiro Yasunaga, Richard~Lanas Phillips,
  Sara Beery, Jure Leskovec, Anshul Kundaje, Emma Pierson, Sergey Levine,
  Chelsea Finn, and Percy Liang.
\newblock {WILDS:} {A} benchmark of in-the-wild distribution shifts.
\newblock {\em CoRR}, abs/2012.07421, 2020.

\bibitem{samplingintegration2011}
Florian Lepr{\^{e}}tre, Fabien Teytaud, and Julien Dehos.
\newblock Multi-armed bandit for stratified sampling: Application to numerical
  integration.
\newblock In {\em Conference on Technologies and Applications of Artificial
  Intelligence, {TAAI} 2017, Taipei, Taiwan, December 1-3, 2017}, pages
  190--195. {IEEE} Computer Society, 2017.

\bibitem{Dataset_FAFDB_li2017reliable}
Shan Li, Weihong Deng, and JunPing Du.
\newblock Reliable crowdsourcing and deep locality-preserving learning for
  expression recognition in the wild.
\newblock In {\em {CVPR} 2017}.

\bibitem{labelshift2018}
Zachary~C. Lipton, Yu{-}Xiang Wang, and Alexander~J. Smola.
\newblock Detecting and correcting for label shift with black box predictors.
\newblock In Jennifer~G. Dy and Andreas Krause, editors, {\em Proceedings of
  the 35th International Conference on Machine Learning, {ICML} 2018,
  Stockholmsm{\"{a}}ssan, Stockholm, Sweden, July 10-15, 2018}, volume~80 of
  {\em Proceedings of Machine Learning Research}, pages 3128--3136. {PMLR},
  2018.

\bibitem{Dataset_Speech_Fluent_LugoschRITB19}
Loren Lugosch, Mirco Ravanelli, Patrick Ignoto, Vikrant~Singh Tomar, and Yoshua
  Bengio.
\newblock Speech model pre-training for end-to-end spoken language
  understanding.
\newblock In {\em Interspeech 2019}.

\bibitem{ConfusionMatrixclassimbalance2019}
Amalia Luque, Alejandro Carrasco, Alejandro Mart{\'{\i}}n, and Ana de~las
  Heras.
\newblock The impact of class imbalance in classification performance metrics
  based on the binary confusion matrix.
\newblock {\em Pattern Recognit.}, 91:216--231, 2019.

\bibitem{Dataset_SEntiment_IMDB_ACL_HLT2011}
Andrew~L. Maas, Raymond~E. Daly, Peter~T. Pham, Dan Huang, Andrew~Y. Ng, and
  Christopher Potts.
\newblock Learning word vectors for sentiment analysis.
\newblock In {\em Human Language Technologies, ACL 2011}.

\bibitem{Dataset_AFFECTNET_MollahosseiniHM19}
Ali Mollahosseini, Behzad Hasani, and Mohammad~H. Mahoor.
\newblock Affectnet: {A} database for facial expression, valence, and arousal
  computing in the wild.
\newblock {\em {IEEE} Trans. Affect. Comput.}, 10(1):18--31, 2019.

\bibitem{IBMAPIUpdatePaper2020}
Haode Qi, Lin Pan, Atin Sood, Abhishek Shah, Ladislav Kunc, and Saloni Potdar.
\newblock Benchmarking intent detection for task-oriented dialog systems.
\newblock {\em CoRR}, abs/2012.03929, 2020.

\bibitem{covshiftkernel2007}
Joaquin Quiñonero-Candela, Masashi Sugiyama, Anton Schwaighofer, and Neil~D.
  Lawrence.
\newblock {\em Covariate Shift by Kernel Mean Matching}, pages 131--160.
\newblock 2009.

\bibitem{ACLtest2020}
Marco~T{\'{u}}lio Ribeiro, Tongshuang Wu, Carlos Guestrin, and Sameer Singh.
\newblock Beyond accuracy: Behavioral testing of {NLP} models with checklist.
\newblock In Dan Jurafsky, Joyce Chai, Natalie Schluter, and Joel~R. Tetreault,
  editors, {\em Proceedings of the 58th Annual Meeting of the Association for
  Computational Linguistics, {ACL} 2020, Online, July 5-10, 2020}, pages
  4902--4912. Association for Computational Linguistics, 2020.

\bibitem{priorshift2002}
Marco Saerens, Patrice Latinne, and Christine Decaestecker.
\newblock Adjusting the outputs of a classifier to new a priori probabilities:
  {A} simple procedure.
\newblock {\em Neural Comput.}, 14(1):21--41, 2002.

\bibitem{covariateshift2000}
Hidetoshi Shimodaira.
\newblock Improving predictive inference under covariate shift by weighting the
  log-likelihood function.
\newblock {\em Journal of Statistical Planning and Inference}, 90(2):227--244,
  2000.

\bibitem{covshift2007}
Masashi Sugiyama, Shinichi Nakajima, Hisashi Kashima, Paul von B{\"{u}}nau, and
  Motoaki Kawanabe.
\newblock Direct importance estimation with model selection and its application
  to covariate shift adaptation.
\newblock In John~C. Platt, Daphne Koller, Yoram Singer, and Sam~T. Roweis,
  editors, {\em Advances in Neural Information Processing Systems 20,
  Proceedings of the Twenty-First Annual Conference on Neural Information
  Processing Systems, Vancouver, British Columbia, Canada, December 3-6, 2007},
  pages 1433--1440. Curran Associates, Inc., 2007.

\bibitem{Imagenetbenckmark20}
Dimitris Tsipras, Shibani Santurkar, Logan Engstrom, Andrew Ilyas, and
  Aleksander Madry.
\newblock From imagenet to image classification: Contextualizing progress on
  benchmarks.
\newblock In {\em Proceedings of the 37th International Conference on Machine
  Learning, {ICML} 2020, 13-18 July 2020, Virtual Event}, volume 119 of {\em
  Proceedings of Machine Learning Research}, pages 9625--9635. {PMLR}, 2020.

\bibitem{MaskFaceDataset2020}
Zhongyuan Wang, Guangcheng Wang, Baojin Huang, Zhangyang Xiong, Qi~Hong, Hao
  Wu, Peng Yi, Kui Jiang, Nanxi Wang, Yingjiao Pei, Heling Chen, Yu~Miao,
  Zhibing Huang, and Jinbi Liang.
\newblock Masked face recognition dataset and application.
\newblock {\em CoRR}, abs/2003.09093, 2020.

\bibitem{Dataset_Speech_GoogleCommand}
Pete Warden.
\newblock Speech commands: {A} dataset for limited-vocabulary speech
  recognition.
\newblock {\em CoRR}, abs/1804.03209, 2018.

\bibitem{Dataset_EXPW_SOCIALRELATION_ICCV2015}
Zhanpeng Zhang, Ping Luo, Chen~Change Loy, and Xiaoou Tang.
\newblock Learning social relation traits from face images.
\newblock In {\em {ICCV} 2015}.

\bibitem{activelearninglabelshift2021}
Eric Zhao, Anqi Liu, Animashree Anandkumar, and Yisong Yue.
\newblock Active learning under label shift.
\newblock In Arindam Banerjee and Kenji Fukumizu, editors, {\em The 24th
  International Conference on Artificial Intelligence and Statistics, {AISTATS}
  2021, April 13-15, 2021, Virtual Event}, volume 130 of {\em Proceedings of
  Machine Learning Research}, pages 3412--3420. {PMLR}, 2021.

\end{thebibliography}
\end{document}